\documentclass{article}

\PassOptionsToPackage{numbers, compress, sort}{natbib}

\usepackage[final]{neurips_2022}

\RequirePackage{algorithm}
\RequirePackage{algorithmic}

\usepackage[utf8]{inputenc} 
\usepackage[T1]{fontenc}    
\usepackage{hyperref}       
\usepackage{url}            
\usepackage{booktabs}       
\usepackage{amsfonts}       
\usepackage{nicefrac}       
\usepackage{microtype}      
\usepackage[multiple]{footmisc}
\usepackage{bm}
\usepackage{soul,xcolor}
\usepackage{amsmath}
\usepackage{bm}
\usepackage{graphicx}
\usepackage{subcaption}
\usepackage{placeins}
\usepackage{comment}
\newsavebox{\imagebox}
\usepackage{bbm}
\usepackage{refcount}
\usepackage{wrapfig}
\usepackage{amsthm}
\usepackage{thmtools, thm-restate}

\usepackage{enumitem}
\usepackage{amssymb}
\usepackage{enumitem}

\newtheorem{corollary}{Corollary}

\newtheorem{lemma}{Lemma}

\newcommand{\PR}{\textsc{PR}}

\newcommand{\casmo}{\textsc{Casmopolitan}}
\newcommand{\hybo}{\textsc{HyBO}}
\newcommand{\bocs}{\textsc{BOCS}}
\newcommand{\mercbo}{\textsc{MerCBO}}
\newcommand{\combo}{\textsc{COMBO}}
\newcommand{\mivabo}{\textsc{MiVaBO}}
\newcommand{\cocabo}{\textsc{CoCaBO}}
\newcommand{\smac}{\textsc{SMAC}}
\newcommand{\turbo}{\textsc{TR}}
\newcommand{\gryffin}{\textsc{Gryffin}}
\newcommand{\mvrsm}{\textsc{MVRSM}}
\newcommand{\contbo}{\textsc{Cont. Relax.}}
\newcommand{\exactround}{\textsc{Exact Round}}
\newcommand{\exactroundste}{\textsc{Exact Round + STE}}
\newcommand{\tpe}{\textsc{TPE}}

\DeclareMathOperator*{\argmax}{arg\,max}
\DeclareMathOperator*{\round}{round}
\DeclareMathOperator*{\discretize}{discretize}

\newenvironment{enumerate_compact}{
\begin{enumerate}
  \setlength{\itemsep}{4pt}
  \setlength{\parskip}{0pt}
  \setlength{\parsep}{0pt}
}{\end{enumerate}}

\newcommand{\papertitle}{
Bayesian Optimization over Discrete and Mixed Spaces via Probabilistic Reparameterization
}

\title{\papertitle}

\author{%
  Samuel Daulton\\
  University of Oxford, Meta\\
  \texttt{sdaulton@meta.com}\\
   \And
   Xingchen Wan\\
   University of Oxford\\
   \texttt{xwan@robots.ox.ac.uk}\\
   \And
   David Eriksson\\
   Meta\\
   \texttt{deriksson@meta.com}\\
   \And
   Maximilian Balandat\\
   Meta\\
   \texttt{balandat@meta.com}\\
   \And
   Michael A. Osborne\\
   University of Oxford\\
   \texttt{mosb@robots.ox.ac.uk}\\
   \And
   Eytan Bakshy\\
   Meta\\
   \texttt{ebakshy@meta.com} \\
}

\begin{document}

\maketitle

\begin{abstract}
Optimizing expensive-to-evaluate black-box functions of discrete (and potentially continuous) design parameters is a ubiquitous problem in scientific and engineering applications. 
Bayesian optimization (BO) is a popular, sample-efficient method that leverages a probabilistic surrogate model and  an acquisition function (AF) to select promising designs to evaluate. 
However, maximizing the AF over mixed or high-cardinality discrete search spaces is challenging standard gradient-based methods cannot be used directly or evaluating the AF at every point in the search space would be computationally prohibitive.
To address this issue, we propose using probabilistic reparameterization (\PR{}). 
Instead of directly optimizing the AF over the search space containing discrete parameters, we instead maximize the expectation of the AF over a probability distribution defined by continuous parameters. 
We prove that under suitable reparameterizations, the BO policy that maximizes the probabilistic objective is the same as that which maximizes the AF, and therefore, \PR{} enjoys the same regret bounds as the original BO policy using the underlying AF. 
Moreover, our approach provably converges to a stationary point of the probabilistic objective under gradient ascent using scalable, unbiased estimators of both the probabilistic objective and its gradient. Therefore, as the number of starting points and gradient steps increase, our approach will recover of a maximizer of the AF (an often-neglected requisite for commonly used BO regret bounds).
We validate our approach empirically and demonstrate state-of-the-art optimization performance on a wide range of real-world applications. 
\PR{} is complementary to (and benefits) recent work and naturally generalizes to settings with multiple objectives and black-box constraints.
\end{abstract}

\vspace{-1ex}
\section{Introduction}
\vspace{-1ex}
\label{sec:intro}
Many scientific and engineering problems involve tuning discrete and/or continuous parameters to optimize an objective function. 
Often, the objective function is ``black-box'', meaning it has no known closed-form expression. 
For example, optimizing the design of an electrospun oil sorbent---a material that can be used to absorb oil in the case of a marine oil spill to mitigate ecological harm---to maximize properties such as the oil absorption capacity and mechanical strength \citep{oil_solbent} can involve tuning both discrete ordinal experimental conditions and continuous parameters controlling the composition of the material. 
For another example, optimizing the structural design of a welded beam can involve tuning the type of metal (categorical), the welding type (binary), and the dimensions of the different components of the beam (discrete ordinals)--resulting in a search space with over 370 million possible designs \citep{tran2019constrained}. 
We consider the scenario where querying the objective function is expensive and sample-efficiency is crucial. 
In the case of designing the oil sorbent, evaluating the objective function requires manufacturing the material and measuring its properties in a laboratory, requiring significant time and resources. 

Bayesian optimization (BO) is a popular technique for sample-efficient black-box optimization, due to its proven performance guarantees in many settings \citep{ucb, JMLR:v20:18-213} and its strong empirical performance \citep{frazier2018tutorial, turner2021bayesian}. 
BO leverages a probabilistic surrogate model of the unknown objective(s) and an acquisition function (AF) that provides utility values for evaluating a new design to balance exploration and exploitation. 
Typically, the maximizer of the AF is selected as the next design to evaluate. 
However, maximizing the AF over mixed search spaces (i.e., those consisting of discrete and continuous parameters) or large discrete search spaces is challenging\footnote{If the discrete search space has low enough cardinality that the AF can be evaluated at every discrete element, then acquisition optimization can be solved trivially.} and continuous (or gradient-based) optimization routines cannot be directly applied.
Theoretical performance guarantees of BO policies require that the maximizer of the AF is found and selected as the next design to evaluate on the black-box objective function \citep{ucb}. When the maximizer is not found, regret properties are not guaranteed, and the performance of the BO policy may degrade. 

To tackle these challenges, we propose a technique for improving AF optimization using a probabilistic reparameterization (\PR{}) of the discrete parameters. Our main contributions are:
\begin{enumerate_compact}
    \item We propose a technique, probabilistic reparameterization (\PR{}), for maximizing AFs over discrete and mixed spaces by instead optimizing a probabilistic objective (PO): the expectation of the AF over a probability distribution of discrete random variables corresponding to the discrete parameters. 
    \item We prove that there is an equivalence between the maximizers of the acquisition function and the the maximizers of the PO and hence, the policy that chooses designs that are best with respect to the PO enjoys the same performance guarantees as the standard BO policy.
    \item We derive scalable, unbiased Monte Carlo estimators of the PO and its gradient with respect to the parameters of the introduced probability distribution. 
    We show that stochastic gradient ascent using our gradient estimator is guaranteed to converge to a stationary point on the PO surface and will recover a global maximum of the underlying AF as the number of starting points and gradient steps increase. 
       This is important because many BO regret bounds require maximizing the AF~\citep{ucb}. Although the AF is often non-convex and maximization is hard, empirically, with a modest number of starting points, \PR{} leads to better AF optimization than alternative methods.
    \item We show that \PR{} yields state-of-the-art optimization performance on a wide variety of real-world design problems with discrete and mixed search spaces. 
    Importantly, \PR{} is \emph{complementary} to many existing approaches such as popular multi-objective, constrained, and trust region-based approaches; in particular, \PR{} is agnostic to the underlying probabilistic model over discrete parameters---which is not the case for many alternative methods.
\end{enumerate_compact}

\vspace{-1ex}
\section{Preliminaries}
\vspace{-1ex}
\paragraph{Bayesian Optimization}
We consider the problem of optimizing a black-box function $f:\mathcal X \times \mathcal Z \rightarrow \mathbb R$ over a compact search space $\mathcal X 
\times \mathcal Z$, where $\mathcal X  = \mathcal X^{(1)}\times\cdot\cdot\cdot\times \mathcal X^{(d)}$ is the domain of the $d \geq 0$ continuous parameters  ($x^{(i)} \in\mathcal X^{(i)}$ for $i=1, ..., d$) and $\mathcal Z = \mathcal Z^{(1)}\times\cdot\cdot\cdot\times \mathcal Z^{(d_z)}$ is the domain of the $d_z \geq 1$ discrete parameters ($z^{(i)} \in\mathcal Z^{(i)}$ for $i=1, ..., d_z$).\footnote{Throughout this paper, we use a mixed search space $\mathcal X \times \mathcal Z$ in our derivations, theorems, and proofs, without loss of generality with respect to the case of a purely discrete search space. If $d = 0$, then the objective function $f:\mathcal Z \rightarrow \mathbb R$ is defined over the discrete space $\mathcal Z$ and the continuous parameters in this exposition can simply be ignored.}

BO leverages (i) a probabilistic surrogate model---typically a Gaussian process (GP) \citep{Rasmussen2004}---fit to a data set $\mathcal D_n = \{\bm x_i, \bm z_i, y_i\}_{i=1}^n$ of designs and corresponding (potentially noisy) observations $y_i = f(\bm x_i, \bm z_i) + \epsilon_i, \epsilon_i \sim \mathcal N(0, \sigma^2)$, and (ii) an acquisition function $\alpha(\bm x, \bm z)$ that uses the surrogate model's posterior distribution to quantify the value of evaluating a new design.
Common AFs include expected improvement (EI) \citep{jones98} and upper confidence bound (UCB) \citep{ucb}---the latter of which enjoys no-regret guarantees in certain settings \citep{ucb}. 
The next design to evaluate is chosen by maximizing the AF $\alpha(\bm x, \bm z)$ over $\mathcal{X \times Z}$.
Although the black-box objective $f$ is expensive-to-evaluate, the AF is relatively cheap-to-query, and therefore, it can be optimized numerically. Gradient-based optimization routines are often used to maximize the AF over continuous domains \citep{garnett_bayesoptbook_2022}.

\vspace{-1ex}
\paragraph{Discrete Parameters}
In its basic form, BO assumes that the inputs are continuous. 
However, discrete parameters such as binary, discrete ordinal, and non-ordered categorical parameters are ubiquitous in many applications. 
In the presence of such parameters, optimizing the AF is more difficult, as standard gradient-based approaches cannot be directly applied. 
Recent works have proposed various approaches including multi-armed bandits \citep{nguyen2019, ru2020bayesian} and  local search \citep{oh2019combinatorial} for discrete domains and interleaved discrete/continuous optimization procedures for mixed domains \citep{deshwal2021bayesian, wan2021think}. A simple and widely-used approach across many popular BO packages \citep{balandat2020botorch, gpyopt2016} is to one-hot encode the categorical parameters, apply a continuous relaxation when solving the optimization, and discretize (round) the resulting continuous candidates.
Examples of continuous relaxations and discretization functions are listed in Table~\ref{table:discrete}.

\begin{table*}[!ht]
\vspace{-1ex}
    \centering
    \caption{\label{table:discrete} Different parameter types, their continuous relaxations, and discretization functions.}
    \begin{small}
    \begin{sc}
    \begin{tabular}{llll}
        \toprule
        Type & Domain & Cont. Relaxation  & $\discretize(\cdot)$ Function\\
        \midrule
         Binary & $z \in \{0,1\}$& $z' \in [0,1]$ & $\round(z')$\\
        Ordinal & 
$z \in \{0, ..., C-1\}$ & $z' \in [-0.5,C-0.5)$ & $\round(z')$\\
        Categorical &$ z \in \{0,..., C-1\}$& $ z' \in [0,1]^C$& $\argmax_c z'^{(c)}$ \\
        \bottomrule
    \end{tabular}
    \end{sc}
    \end{small}
    \vspace{-1ex}
\end{table*}

Although using a continuous relaxation allows for efficient optimization using standard optimization routines in an alternate continuous domain $\mathcal Z' \subset \mathbb{R}^m$, the AF value for an infeasible continuous value (i.e., $z' \notin \mathcal{Z}$) does not account for the discretization that must occur before the black-box function is evaluated. 
Moreover, the acquisition value for an infeasible continuous value can be larger than the AF value after discretization. 
For an illustration of this, see Fig.~\ref{fig:intro_figure} (middle/right).
In the worst case, BO will repeatedly select the same infeasible continuous design due to its high AF value, but discretization will result in a design that has already been evaluated and has zero AF value. 
To mitigate this degenerate behavior and avoid the over-estimation issue, \citet{GarridoMerchn2020DealingWC} propose discretizing $z'$ before evaluating the AF, but the AF is then non-differentiable with respect to the $z'$.
While this improves performance on small search spaces, the response surface has large flat regions after discretizing $z'$, which makes it difficult to optimize the AF. 
The authors of \citep{GarridoMerchn2020DealingWC} propose to approximate the gradients using finite differences, but, empirically, we find that this approach to be leads to sub-optimal AF optimization relative to \PR{}.
\begin{figure*}[!t]
    \centering
    \includegraphics[width=\linewidth]{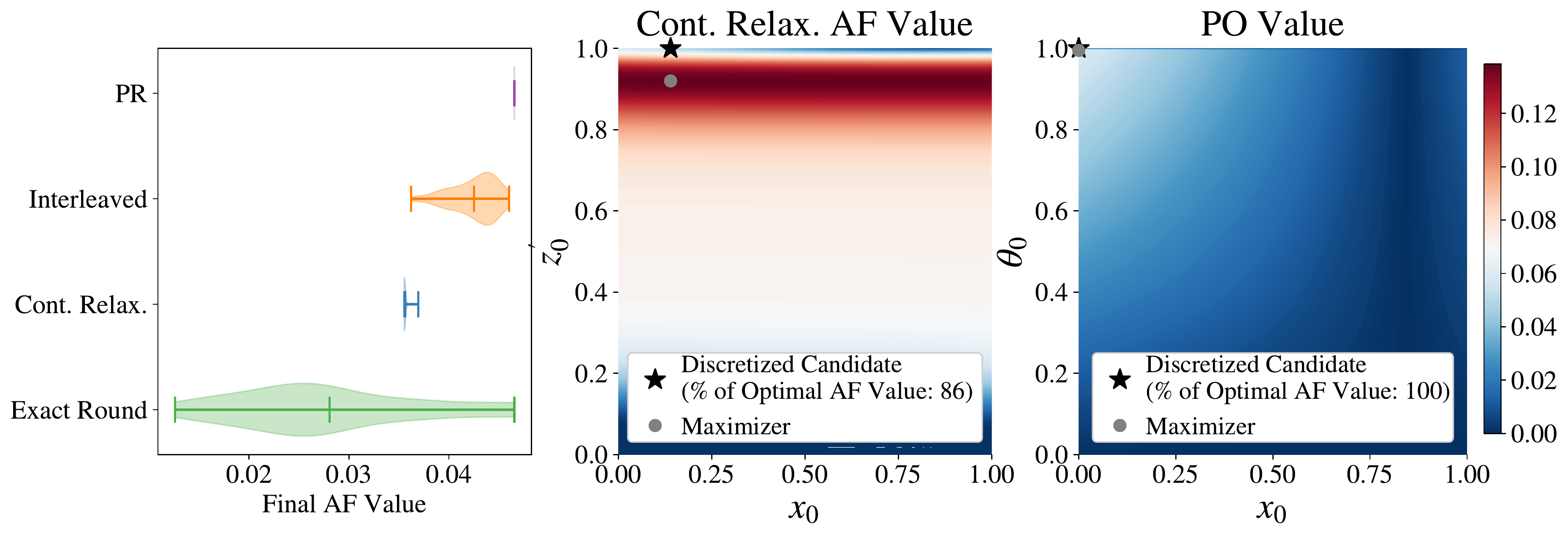}
    \caption{(\textbf{Left}) A comparison of AF optimization using different methods over a mixed search space shows that \emph{\PR{} outperforms alternative methods for AF optimization and has much lower variance across replications}. 
    The violin plots show the distribution of final AF values and the mean. ``Cont. Relax.'' denotes optimizing a continuous relaxation of the categoricals with exact gradients. ``Exact Round'' refers to optimizing a continuous relaxation with approximate gradients (via finite difference), but discretizes the relaxation before evaluating the surrogate \citep{GarridoMerchn2020DealingWC}. "Interleaved" alternates between one step of local search on the discrete parameters and one step of gradient ascent on the continuous parameters (used in \casmo{} \citep{wan2021think}).
    For each method, the best candidate across 20 restarts is selected (after discretization) and the acquisition value of the resulting feasible candidate is recorded. The AF is expected improvement \citep{jones98}. (\textbf{Middle/Right}) AF values with a continuous relaxation (middle) and the PO (right) for the Branin function over a mixed domain with one continuous parameter ($x_0$) and one binary parameter ($z_0)$ (see Appendix~\ref{appdx:exp_details} for details on Branin). 
    (\textbf{Middle}) Under a continuous relaxation, the maximizer of the AF is an infeasible point in the domain (grey circle), which results in a suboptimal AF value when rounded (black star); the resulting candidate only has 86\% of the AF value of the true maximizer. The maximum AF value across the feasible search space is shown in white and the red regions indicate that the continuous relaxation overestimates the AF value since it is greater than the maximum AF value of any feasible design.
    (\textbf{Right}) The PO is maximized at the AF unique maximizer within the valid search domain. These contours show that \PR{} avoids the overestimation issue that the naive continuous relaxation suffers from.
    }
    \label{fig:intro_figure}
    \vspace{-1ex}
\end{figure*}
\vspace{-1ex}
\section{Probabilistic Reparameterization}
\vspace{-1ex}
\begin{wrapfigure}{r}{0.62\textwidth}
\begin{minipage}{0.62\textwidth}
\vspace{-3.5ex}
\begin{algorithm}[H]
    \caption{BO with PR}
    \label{alg:bo_pr}
    \begin{algorithmic}[1]
        \STATE Input: black-box objective $f: \mathcal X \times \bm Z \rightarrow \mathbb R$
        \STATE Initialize $\mathcal D_0 \leftarrow \emptyset$, $\text{GP}_0 \leftarrow \text{GP}(\bm 0, k)$
        \FOR{$n=1$ \textbf{to} $N_\text{iterations}$}
        \STATE $(\bm x_n, \bm \theta_n) \leftarrow \argmax_{(\bm x, \bm \theta) \in \mathcal X \times \Theta} \mathbb E_{\bm Z \sim p(\bm Z | \bm \theta)}[\alpha(\bm x, \bm Z)]$\\
        \STATE Sample $\bm z_n \sim p(\bm Z | \bm \theta_n)$\\
        \STATE Evaluate $ f(\bm x_n, \bm z_n)$
        \STATE $\mathcal D_n \leftarrow \mathcal D_{n-1} \cup \{(\bm x_n, \bm z_n, \bm f (\bm x_n, \bm z_n))\}$\\
        \STATE Update posterior $\text{GP}_n$ given $\mathcal D_n$
        \ENDFOR
    \end{algorithmic}
\end{algorithm}
\vspace{-4ex}
\end{minipage}
\end{wrapfigure}
We propose an alternative approach based on probabilistic reparameterization, a relaxation of the original optimization problem involving discrete parameters. 
Rather than directly optimizing the AF via a continuous relaxation $\bm z'$ of the design $\bm z$, we instead reparameterize the optimization problem by introducing a discrete probability distribution $p(\bm Z |\bm \theta)$ over a random variable $\bm Z$ with support exclusively over $\mathcal Z$. 
This distribution is parameterized by a vector of continuous parameters $\bm \theta$. We use $\bm z$ to denote the vector $(z^{(1)}, ..., z^{(d_z)})$, where each element is a different (possibly vector-valued) discrete parameter.
Given this reparameterization, we define the probabilistic objective (PO):
\begin{equation}
    \label{eq:prob_reparam}
   \mathbb{E}_{\bm Z \sim p(\bm Z|\bm \theta)}[\alpha(\bm x, \bm Z)].
\end{equation}
Algorithm~\ref{alg:bo_pr} outlines BO with probabilistic reparameterization. 

\PR{} allows us to optimize $\bm \theta$ and $\bm x$ over a continuous space to maximize the PO instead of optimizing $\bm x$ and $\bm z$ to maximize $\alpha$ directly over the mixed search space $\mathcal X \times \mathcal Z$. 
As we will show later,
maximizing the PO allows us to recover a maximizer of~$\alpha$ over the space $\mathcal{X} \times \mathcal Z$. 
Choosing $p(\bm Z|\theta)$ to be a discrete distribution over $\mathcal Z$ means the realizations of $\bm Z$ are feasible values in $\mathcal Z$. 
Hence, the AF is only evaluated for feasible discrete designs. 
Since $p(\bm Z|\bm \theta)$ is a discrete probability distribution, we can express $\mathbb E_{\bm Z \sim p(\bm Z|\bm \theta)}[\alpha(\bm x, \bm Z)]$ as a linear combination where each discrete design is weighted by its probability mass:
\begin{equation}
    \label{eq:prob_mass}
    \mathbb E_{\bm Z \sim p(\bm Z|\bm \theta)}[\alpha(\bm x, \bm Z)] = \sum_{\bm z \in \mathcal Z} p(\bm z |\bm  \theta)\alpha(\bm x, \bm z).
\end{equation}

Example distributions for binary, ordinal, and categorical parameters are provided in Table~\ref{table:proposal}. 
\begin{table*}[h]
\vspace{-1ex}
    \centering
    \caption{\label{table:proposal} Examples of probabilistic reparameterizations for different parameter types. 
    We denote the $(C-1)$-simplex as $\Delta^{C-1}$.
    }
    \begin{small}
    \begin{sc}
    \begin{tabular}{lll}
        \toprule
        Parameter Type & Random Variable & Continuous Parameter\\
        \midrule
        Binary & $Z \sim \text{Bernoulli}(\theta)$ & $\theta \in [0,1]$\\
        Ordinal & $Z =  \lfloor \theta \rfloor + B, B \sim \text{Bernoulli}(\theta - \lfloor \theta \rfloor)$&  $\theta \in [0,C-1]$\\
        Categorical & $Z \sim  \text{Categorical}( \theta), \theta = (\theta^{(1)}, ..., \theta^{(C)})$ & 
$\theta \in \Delta^{C-1}$\\
        \bottomrule
    \end{tabular}
    \end{sc}
    \end{small}
    \vspace{-1ex}
\end{table*}

Although ordinal parameters could use the same categorical distributions as the non-ordered categorical parameters, we opt for the provided proposal distribution since it uses a scalar $\theta$ (rather than a $C$-element vector) and it naturally encodes the ordering of the values.
Using an independent random variable $Z^{(i)}\sim p(Z^{(i)} | \theta^{(i)})$ for each parameter $z^{(i)}$ for $i=1, ..., d_z$ means that the probabilistic objective can be expressed as
\begin{equation}
    \label{eq:prob_mass2}
    \mathbb E_{\bm Z\sim p(
    \bm Z | \bm \theta)}[\alpha(\bm x, \bm Z)]= \hspace{-1ex}\sum_{z^{(1)} \in \mathcal Z^{(1)}}\hspace{-1ex}\cdot\cdot\cdot \hspace{-1ex} \sum_{z^{(d_z)} \in \mathcal Z^{(d_z)}} \hspace{-1.5ex}\alpha\big(\bm x, z^{(1)}, ..., z^{(d_z)}\big)\prod_{i=1}^{d_z} p(z^{(i)} |\theta^{(i)}).
\end{equation}
\vspace{-1ex}

\vspace{-1ex}
\subsection{Analytic Gradients}
\vspace{-1ex}
One important benefit of \PR{} is that the PO in \eqref{eq:prob_reparam} is differentiable with respect to $\bm \theta$ (and $\bm x$, if the gradient of $\alpha$ with respect to $\bm x$ exists), whereas $\alpha(\bm x, \bm z)$ is not differentiable with respect to $\bm z$.  
The gradients of the PO with respect to $\bm \theta$ and $\bm x$ can be obtained by differentiating Equation~\ref{eq:prob_mass}:
\begin{equation}
    \label{eq:exact_gradient}
    \nabla_{\bm \theta} \mathbb E_{\bm Z \sim p(\bm Z | \bm \theta)}[\alpha(\bm x, \bm Z)] = \sum_{\bm z \in \mathcal Z} \alpha(\bm x, \bm z)\nabla_{\bm \theta}p(\bm z|\bm \theta)
\end{equation}
\begin{equation}
    \label{eq:exact_gradient_x}
    \nabla_{\bm x} \mathbb E_{\bm Z \sim p(\bm Z | \bm \theta)}[\alpha(\bm x, \bm Z)] = \sum_{\bm z \in \mathcal Z} p(\bm z|\bm \theta)\nabla_{\bm x}\alpha(\bm x, \bm z)
\end{equation}
This enables optimizing the PO (line 4 of Algorithm~\ref{alg:bo_pr}) efficiently and effectively using gradient-based methods.



\vspace{-1ex}
\subsection{Theoretical Properties}
\vspace{-1ex}
In this section, we derive theoretical properties of PR. Proofs are provided in Appendix~\ref{appdx:proofs}.  Our first result is that there is an equivalence between the maximizers of the PO and the maximizers of the AF over $\mathcal X \times \mathcal{Z}$.

\begin{restatable}[Consistent Maximizers]{theorem}{consistentmaximizers}
    \label{thm:consistent_maximizers}
    
    Suppose that $\alpha$ is continuous in $\bm x$ for every $\bm z \in \mathcal{Z}$. 
    Let $\mathcal H^*$ be the maximizers of $\alpha(\bm x, \bm z)$: $\mathcal H^* = \{(\bm x, \bm z) \in \argmax_{(\bm x, \bm z) \in \mathcal X \times \mathcal Z} \alpha(\bm x, \bm z)\}$. 
    Let $\mathcal J^* \subseteq \mathcal X \times \Theta$ be the maximizers of $\mathbb E_{\bm Z\sim p(\bm Z|\bm \theta)}[\alpha(\bm x, \bm Z)]$: $\mathcal J^* = \{(\bm x, \bm \theta) \in \argmax_{(\bm x, \bm \theta) \in \mathcal X \times \Theta} \mathbb E_{\bm Z \sim p(\bm Z | \bm \theta)}[\alpha(\bm x, \bm Z)]\}$, where $\Theta$ is the domain of $\bm \theta$. 
    Let $\hat{\mathcal H}^*  \subseteq \mathcal X \times \mathcal Z$ be defined as: $\hat{\mathcal H}^* = \{(\bm x, \tilde{\bm z}): (\bm x, \bm \theta) \in \mathcal J^*, \tilde{\bm z} \sim p(\bm Z|\bm \theta)\}$. 
    Then, $\hat{\mathcal H}^* = \mathcal H^*$.
\end{restatable}

Algorithm~\ref{alg:bo_pr} outlines BO with probabilistic reparameterization. Importantly, 
Theorem~\ref{thm:consistent_maximizers} states that sampling from the distribution parameterized by a maximizer of the PO yields a maximizer of $\alpha$, and therefore, Algorithm~\ref{alg:bo_pr} enjoys the performance guarantees of $\alpha(\cdot)$.

\begin{restatable}[Regret Bounds]{corollary}{regret}
    \label{corr:regret}
    Let $\alpha(\bm x, \bm z)$ be an acquisition function over a search space $\mathcal X \times \mathcal Z$ such that when $\alpha$ is applied as part of a BO strategy that strategy has bounded regret .
    If the conditions for the regret bounds of that BO strategy using $\alpha$ are satisfied, then Algorithm~\ref{alg:bo_pr} using $\alpha$ enjoys the same regret bound.
\end{restatable}


Examples of BO policies with bounded regret include those based on AFs such as upper confidence bound (UCB) \citep{ucb} or Thompson sampling (TS) \citep{russo2014learning} for single objective optimization, and UCB or TS with Chebyshev \citep{paria2018flexible} or hypervolume \citep{golovin2020random} scalarizations in the multi-objective setting.

Although the BO policy selects a maximizer of $\alpha$ is equivalent to the BO policy in Algorithm~\ref{alg:bo_pr}, maximizing the AF over mixed or high-dimensional discrete search spaces is challenging because commonly used gradient-based methods cannot directly be applied. 
The key advantage of our approach is that maximizers of the AF can be identified efficiently and effectively by optimizing the PO using gradient information instead of directly optimizing the AF.
We find that optimizing \PR{} yields better results than directly optimizing $\alpha$ or other common relaxations as shown in Figure~\ref{fig:intro_figure}(Left), where we compare AF optimization methods on the mixed Rosenbrock test problem (see Appendix~\ref{appdx:exp_details} for details).
\vspace{-1ex}
\section{Practical Monte Carlo Estimators}
\vspace{-1ex}
\subsection{Unbiased estimators of the Probabilistic Reparameterization and its Gradient}
\vspace{-1ex}
As the number of discrete configurations ($|\mathcal Z|$) increases,
the PO and its gradient may become computationally expensive to evaluate analytically because both require a summation of $|\mathcal Z|$ terms. 
Therefore, we propose to estimate the PO and its gradient using Monte Carlo (MC) sampling. 
The MC estimator of the PO is given by
\vspace{-1ex}
\begin{equation}
    \label{eq:mc_estimator}
    \mathbb{E}_{\bm Z \sim p(\bm Z | \bm \theta)}[\alpha(\bm x, \bm Z)] \approx \frac{1}{N}\sum_{i=1}^N \alpha(\bm x, \tilde{\bm z}_i),
\end{equation}
where $\tilde{\bm z}_1, ..., \tilde{\bm z}_N$ are samples from $p(\bm Z | \bm \theta)$. 
This estimator is unbiased and can be computed for a large number of samples by evaluating the AF independently (or in chunks) for each input $(\bm x, \tilde{\bm z}_n)$.

MC can also be used to estimate the gradient of the PO with respect to $\bm \theta$. We opt for using a score function gradient estimator \citep{kleijnen1996optimization} (also known as REINFORCE \citep{williams1992simple} and the likelihood ratio estimator \citep{glynn1990likelihood}) because it is simple, scalable, and can be computed using the acquisition values $\{\alpha(\bm x, \tilde{\bm z}_i)\}_{i=1}^N$ that are used in the MC estimator of the PO.
Many alternative lower variance estimators (e.g. \citet{yin2020probabilistic, ARSM_ICML2019}) would require many additional AF evaluations (see \citet{mohamed2020monte} for a review of MC gradient estimation). 
The score function is the gradient of the log probability with respect to the parameters of the distribution:
$\nabla_{\bm \theta}\log p(\bm Z | \bm \theta) = \frac{\nabla_{\bm \theta}p(\bm Z | \bm \theta)}{p(\bm Z | \bm \theta)}.$
Using this score function, we can express the analytic gradient as
\begin{equation*}
    \nabla_{\bm \theta} \mathbb E_{\bm Z \sim p(\bm Z | \bm \theta)}[\alpha(\bm x, \bm Z)] = \sum_{\bm z \in \mathcal Z} \alpha(\bm x, \bm z)p(\bm z| \bm \theta)\nabla_{\bm \theta}\log p(\bm z | \bm \theta)=\mathbb{E}_{\bm Z \sim p(\bm Z | \bm \theta)}[\alpha(\bm x,\bm Z)\nabla_{\bm \theta}\log p(\bm Z | \bm \theta)].
\end{equation*}
The unbiased MC estimator of the gradient of the PO with respect to $\bm \theta$ is given by
\begin{equation}
    \label{eq:score_func_estimator}
    \nabla_{\bm \theta} \mathbb E_{\bm Z \sim p(\bm Z | \bm \theta)}[\alpha(\bm x, \bm Z)] \approx \frac{1}{N}\sum_{i=1}^N \alpha(\bm x, \tilde{\bm z}_i)\nabla_{\bm \theta}\log p(\tilde{\bm z}_i | \bm \theta).
\end{equation}
Since the score function gradient is only defined when $p(\bm z | \bm \theta) > 0$, we reparameterize $\bm \theta$ to ensure $p(\bm z | \bm \theta) > 0$ for all $\bm z$ and $\bm \theta$ by using the softmax transformations provided in Table~\ref{table:transformation}, which are commonly used for computational convenience and stability in probablistic reparameterization \citep{yin2020probabilistic,ARSM_ICML2019}, and the solution converges as $\tau \rightarrow 0$.
\begin{wrapfigure}{r}{0.6\textwidth}
\begin{minipage}{0.6\textwidth}
    \vspace{-1ex}
    \centering
    \begin{small}
    \begin{sc}
    \begin{tabular}{ll}
        \toprule
        Parameter Type & Transformation ($\theta = g(\phi)$) \\
        \midrule
        Binary & $\theta = \sigma((\phi - \frac{1}{2})/\tau)$\\
        Ordinal & $\theta = \lfloor \phi \rfloor + \sigma((\phi- \lfloor \phi \rfloor - \frac{1}{2}) / \tau)$\\
        Categorical & $\theta^{(c)} = \textsc{softmax}((\bm \phi-0.5)/\tau)^{(c)}$\\
        \bottomrule
    \end{tabular}
    \end{sc}
    \end{small}
    \captionof{table}{\label{table:transformation} Transformations where $\tau \in \mathbb R_+$ and $\phi, \theta \in \Theta$. 
    }
    \vspace{-1ex}
\end{minipage}
\end{wrapfigure}
Moreover, even though $p(\bm z | \bm \theta) > 0$, when $p(\bm z | \bm \theta)$ is small, a small number $N$ of MC samples are unlikely to produce any samples where $\tilde{\bm z} = \bm z$. 
Instead of optimizing $\bm \theta$ directly, we instead optimize $\bm \phi$. 
Since the transformations $g(\cdot)$ are differentiable with respect to $\bm \phi$, the gradient (and MC gradient estimator) of the PO with respect to $\bm \phi$ are easily obtained using the gradient of the PO with respect to $\bm \theta$ and a simple application the chain rule (multiplying by $\nabla_{\bm \phi}\bm \theta$).
\subsection{Variance Reduction in Monte Carlo Gradient Estimation}
\label{sec:var_reduction}
Although the MC gradient estimator in \eqref{eq:score_func_estimator} is unbiased, score function gradient estimators can suffer from high variance \citep{mohamed2020monte}. Therefore, we adopt  a popular technique for variance reduction where the score function itself is used as a control variate, since its expectation is zero under $p(\bm Z |\bm \theta)$ \citep{mohamed2020monte}. Score function estimators with this control variate have been shown to be among the best performing gradient estimators  \citep{mohamed2020monte}. Moreover, this technique is simple and merely amounts to subtracting a value $\beta$ from the acquisition value in the score function estimator in Equation \eqref{eq:score_func_estimator}:
\begin{equation}
    \label{eq:score_func_estimator_baseline}
    \nabla_{\bm \theta} \mathbb E_{\bm Z \sim p(\bm Z | \bm \theta)}[\alpha(\bm x, \bm Z)] \approx \frac{1}{N}\sum_{i=1}^N [\alpha(\bm x, \tilde{\bm z}_i)-\beta]\nabla_{\bm \theta}\log p(\tilde{\bm z}_i | \bm \theta).
\end{equation}
The $\beta$ is commonly known as a baseline and is often taken to be a moving average of the (acquisition) values \citep{mohamed2020monte}. See Appendix \ref{appdx:exp_details} for details on $\beta$.
\vspace{-1ex}
\subsection{Convergence Guarantee using Stochastic Gradient Ascent}
\vspace{-1ex}
Since the score function gradient estimator is unbiased, we can leverage previous work on convergence in probability under stochastic gradient ascent \citep{robbins_monro} to arrive at our main convergence result for acquisition optimization.


\begin{restatable}[Convergence Guarantee]{theorem}{converge}
    \label{thm:converge}
Let $\alpha: \mathcal{X \times Z} \rightarrow \mathbb{R}$ be differentiable in $\bm x$ for every $\bm z \in \mathcal{Z}$.
    Let $(\hat{\bm x}_{t,m}, \hat{\bm \theta}_{t,m})$ be the best solution after running stochastic gradient ascent for $t$ time steps on the probabilistic objective $\mathbb{E}_{\bm Z \sim p(\bm Z | \bm \theta)}[\alpha(\bm x, \bm Z)]$ from $m$ starting points with its unbiased MC estimators proposed above. 
    Let $\{a_t\}_{t=1}^\infty$ be a sequence of positive step sizes such that $0 < \sum_{t=1}^\infty a_t^2 = A < \infty$ and $\sum_{t=1}^\infty a_t = \infty$, where $a_t$ is the step size used in stochastic gradient ascent at time step $t$. 
    Let $\hat{\bm z}_{t, m}\sim p(\bm Z|  \hat{\bm \theta}_{t, m})$.
    Then as $t \rightarrow \infty$, $m \rightarrow \infty$, and $\tau \rightarrow 0$, $(\hat{\bm x}_{t, m},\hat{\bm z}_{t, m}) \rightarrow (\bm x^*, \bm z^*) \in \argmax_{(\bm x, \bm z) \in \mathcal X \times \mathcal Z} \alpha(\bm x, \bm z)$ in probability.
\end{restatable}

The significance of Theorem~\ref{thm:converge} is that optimizing the PO is guaranteed to converge in probability to a global maximizer of the AF, meaning that optimizing the PO guarantees that resulting candidate design has maximal AF value. 
The implication is that the intended BO policy is followed and the underlying regret bounds of the AF are recovered (provided that the other conditions of the regret bound are met). Although global convergence is only guaranteed as $m \rightarrow \infty$, we observe in Figure~\ref{fig:intro_figure}(left) that \PR{} yields strong, stable acquisition optimization with only $m=20$ starting points, $200$ steps, and $\tau=\frac{1}{10}$ (see Appendix~\ref{appdx:tau} for further discussion) and outperforms alternative optimization approaches.
\section{Related Work}
\vspace{-1ex}
\label{sec:related_work}
Many methods for BO over discrete and mixed search spaces have been proposed. 
Previous work has largely focused on (i) improving the surrogate models or (ii) improving AF optimization.

\textbf{Improving models}: 
Historically, methods leveraging tree-based surrogate models, e.g., \smac{} \citep{smac} and \tpe{} \citep{tpe}, have been popular for optimizing discrete or mixed search spaces. 
Many recent works have considered alternative surrogate models. \bocs{} encodes categorical parameters as binary variables and uses Bayesian linear regression with pairwise interactions \citep{baptista2018bayesian}.
\combo{} uses a diffusion kernel on the graph defined by the Cartesian product of discrete parameters \citep{oh2019combinatorial}.
\mercbo{} similarly exploits the combinatorial graph, but with Mercer features and Thompson sampling \citep{deshwal2021mercer}.
\hybo{} extends the diffusion kernels to mixed continuous-discrete spaces \citep{deshwal2021bayesian}. 
However, these methods scale poorly with respect to the number of data points and parameters. 
Moreover, of the methods listed above, only \hybo{} supports continuous parameters without restricting them to a discrete set. 
\hybo{} enjoys a universal function approximation property, but relies on summing over all possible orders of interactions between base kernels for each parameter which results in exponential complexity with respect to the number of parameters and limits its applicability to low-dimensional problems. Moreover, the computational issues of such approaches make it difficult to apply them to multi-objective and constrained optimization.
\gryffin{} \citep{hase2021gryffin} uses kernel density estimation, but is limited to categorical search spaces.
\mivabo{} uses a linear combination of basis functions (e.g. pseudo Boolean features \citep{BOROS2002155} for discrete parameters) with interaction terms\citep{mixed_var_bo}. 
\mvrsm{} \citep{bliek2021} uses ReLU-based surrogates for computational efficiency, but is limited by the expressiveness of these models. 

\textbf{Optimizing acquisition functions}: 
As discussed previously, \citet{GarridoMerchn2020DealingWC} propose using continuous relaxation and discretize the inputs before evaluating the AF.
However, the resulting AF after discretization is piece-wise-flat along slices of the continuous relaxation of the discrete parameters and therefore is difficult to optimize. 
\cocabo{} \citep{ru2020bayesian} samples discrete parameters using a multi-armed bandit and optimizes the continuous parameters conditional upon the sampled discrete parameters. 
However, \cocabo{}'s performance degrades as number of discrete configurations increases. 
\casmo{} \citep{wan2021think} uses local trust regions combined with an interleaved AF optimization strategy that alternates between local search steps on the discrete parameters and gradient ascent for the continuous parameters.
Furthermore, both \cocabo{} and \casmo{} do not inherently exploit ordinal structure.

\textbf{Probabilistic reparameterization}:
\PR{} has been considered for optimizing discrete parameters in other domains such as reinforcement learning \citep{williams1992simple} and sparse regression \citep{yin2020probabilistic}.
However, \PR{} has not been leveraged for BO. Although the reparameterization trick used by \citet{wilson2018maxbo} is in a similar vein to \PR{}, \citet{wilson2018maxbo} reparameterize an existing multivariate normal random variable in terms of standard normal random variables and then use sample-path gradient estimators. In contrast, our approach introduces a new probabilistic formulation using discrete probability distributions and uses likelihood-ratio-based gradient estimators since sample-path gradients cannot be computed through discrete sampling. 

\textbf{Alternative methods for propagating gradients}: 
Alternative methods for propagating gradients through discrete structures have been considered in the deep learning community (among others). 
One approach is to use approximate discrete Concrete distributions \citep{maddison2016concrete, jang2016categorical}, which admit sample-path gradients. 
However, samples from Concrete distributions are not discrete and approximation error can result in pathologies similar to evaluating the AF using continuous relaxation. 
Moreover, approximately discrete samples prohibit using surrogate models that require discrete inputs (without discretizing the samples)---e.g., GPs with Hamming distance kernels \citep{ru2020bayesian}.
Another approach for gradient propagation in the deep learning community is to use straight-through gradient estimators (STE) \citep{bengio2013estimating}, where the gradient of the discretization function with respect to its input is estimated using, for example, an identity function.
This approach works well empirically in some cases, these estimators are not well-grounded theoretically. 
Nevertheless, we discuss and evaluate using STE for AF optimization in Appendix~\ref{appdx:alternative_methods}.

\vspace{-1ex}
\section{Experiments}
\vspace{-1ex}
In this section, we provide an empirical evaluation of \PR{} on a suite of synthetic problems and real world applications. For \PR{}, we use stochastic mini-batches of $N=128$ MC samples in our experiments and demonstrate that \PR{} is robust with respect to the number of MC samples (and compare against analytic PR, where computationally feasible) in Appendix~\ref{appdx:pr_mc_analysis}. We optimize \PR{} using Adam \citep{kingma2014adam} with an initial learning rate of $\frac{1}{40}$. We show that \PR{} Adam is generally robust to the choice of learning rate (more so than vanilla stochastic gradient ascent) in the sensitivity analysis in Figure~\ref{fig:saa_vs_adam} in Appendix~\ref{appdx:sgd_vs_saa}.
We compare \PR{} against two alternative acquisition optimization strategies: using a continuous relaxation (\contbo{}) and using exact discretization with approximate gradients (\exactround{}) \citep{GarridoMerchn2020DealingWC}. These approaches optimize the acquisition function with L-BFGS-B with exact and approximate gradients, respectively.
In addition, we compare against two state-of-the-art methods for discrete/mixed BO: a modified version of \casmo{} \citep{wan2021think} that additionally supports ordinal variables introduced in \citet{wan2022bayesian} and \hybo{} \citep{deshwal2021bayesian}, both of which are shown to outperform the other related works discussed in Section~\ref{sec:related_work}. 
In addition, we showcase how \PR{} is complementary to existing methods such as trust region methods \citep{turbo}.
We demonstrate this by using \PR{} with a trust region for the continuous and discrete ordinal parameters and optimize \PR{} within this trust region. In Appendix~\ref{appdx:alternative_methods}, we provide comparison of \turbo{} methods with alterative optimizers and find that \PR{} is the best optimizer when using \turbo{}s on 6 of the 7 benchmark problems. 
See Appendix~\ref{appdx:exp_details} for additional discussion of \PR{} + \turbo. 
For \PR{}, \exactround{}, and \PR{} + \turbo{} we use the sum of a product kernel and a sum kernel of a categorical kernel \citep{ru2020bayesian} for the categorical parameters and  Mat\'ern-5/2 kernel for all other parameters.\footnote{\contbo{} is incompatible with a categorical kernel, so we use a Mat\'ern-5/2 with one-hot encoded categorical parameters.} Alternative kernels over different representations of categorical parameters such as one-hot encoded vectors, latent embeddings \citep{zhang2019latent}, and known embeddings (e.g. using fingerprint-based reaction encodings for categorical parameters in chemical reaction optimization \citep{shields2021bayesian}) are evaluated in Appendix~\ref{appdx:alternative_kernels}. 

\contbo{}, \exactround{}, \PR{}, and \PR{} + \turbo{} use expected improvement \citep{jones98, gardner2014constrained} for single objective (constrained problems) and expected hypervolume improvement \citep{emmerich2006} for the multi-objective oil sorbent problem (where exact gradients with respect to continuous parameters are computed using auto-differentiation \citep{daulton2020ehvi}). 
We report the mean for each method $\pm$ 2 standard errors across 20 replications. 
Performance is evaluated in terms of regret (feasible regret for constrained problems and hypervolume regret for multi-objective problems). \casmo{} and \hybo{} are not run on Welded Beam and Oil Sorbent as they do not support constrained and multi-objective optimization. We also leave the multi-objective extension of \PR{}+\turbo{} to future work because it would add additional complexity \citep{morbo}. For \hybo{}, we only run 60 BO iterations on SVM due to the large wall time (see Figure \ref{fig:wall_times}) and only report partial results on Cellular Network due to a singular covariance matrix error. See Appendix~\ref{appdx:exp_details} for details on the experiment setup, regret metrics, benchmark problems, and methodological details. We leverage existing open source implementations of \casmo{}  and \hybo{} (see Appendix~\ref{appdx:exp_details} for links), and the implementations of all of other methods are available at \url{https://github.com/facebookresearch/bo_pr}.

\vspace{-1ex}
\subsection{Synthetic Problems}
\vspace{-1ex}
We evaluate all methods on 3 synthetic problems.
\textbf{Ackley} is a 13-dimensional function with 10 binary and 3 continuous parameters (a modified version of the problem in \citet{bliek2021}).
\textbf{Mixed Int F1} is a 16-dimensional variant of the F1 function from \citet{bbob_mixed_int} with 2 binary, 6 discrete ordinal parameters, and 8 continuous parameters. The discrete ordinal parameters have following cardinalities: 2 parameters with 3 values, 2 with 5 values, and 2 with 7 values.
\textbf{Rosenbrock} is a 10-dimensional Rosenbrock function with 6 discrete ordinal parameters with 4 values each and 4 continuous parameters.

\vspace{-1ex}
\subsection{Real World Problems}
\label{subsec:real_world_problems}
\vspace{-1ex}
We consider 5 real world applications including a problem with 5 black-box outcome constraints and a 3-objective problem (see Appendix~\ref{appdx:constrained_bo_mobo} for details on constrained and multi-objective BO). 

\textbf{Welded Beam} Optimizing the design of a welded steel beam is a classical engineering optimization. 
In this problem, the goal is to minimize manufacturing cost subject to 5 black-box constraints on structural properties of the beam (including shear stress, bending stress, and buckling load) by tuning 6 parameters: the welding configuration (binary), the metal material type (categorical with 4 options), and 4 ordinal parameters controlling the dimensions of the beam \citep{tran2019constrained}.

\textbf{SVM Feature Selection} This problem involves jointly performing feature selection and hyperparameter optimization for a Support Vector Machine (SVM) trained on the CTSlice UCI data set \citep{uci,sparse_bo}. 
The design space for this problem involves 50 binary parameters controlling whether a particular feature is included or not, and 3 continuous hyperparameters of the SVM.

\textbf{Cellular Network Optimization} In this 30-dimensional problem, the goal is to tune the tilt (ordinal with 6 values) and transmission power (continuous) for a set of 15 antennas \citep{samal2021dynamic} to maximize a coverage quality metric that is a function of signal power and interference \citep{maddox2021} over a geographic region of interest. 
We use the simulator from \citet{dreifuerst2021}.

\textbf{Direct Arylation Chemical Synthesis} Palladium-catalysed direct arylation has generated significant interest in the pharamceutical development sector\citep{davies2016recent}. 
In this problem, the goal is maximize yield for a direct arylation chemical reaction by tuning 3 categorical parameters corresponding to the choice of solvent, base, and ligand, as well 2 continuous parameters controlling the temperature and concentration. 
We fit a surrogate model to the direct arylation dataset from \citet{shields2021bayesian} in order to facilitate continuous optimization of temperature and concentration. In Appendix~\ref{appdx:alternative_kernels}, we demonstrate that \PR{} can leverage a kernel over fingerprint-based reaction encodings computed via density functional theory (DFT) for the categorical parameters \citep{shields2021bayesian}.

\textbf{Electrospun Oil Sorbent} Marine oil spills can cause ecological catastrophe. 
One avenue for mitigating environmental harm is to design and deploy absorbent materials to capture the spilled oil. 
In this problem, we tune 5 ordinal parameters (3 parameters with 5 values and 2 with 4 values) and 2 continuous parameters controlling the composition and manufacturing conditions for an electrospun oil sorbent material to maximize 3 competing objectives: the oil absorbing capacity, the mechanical strength, and the water contact angle \citep{oil_solbent}.

\begin{figure*}[!ht]
    \vspace{-1ex}
    \centering
    \includegraphics[width=\linewidth]{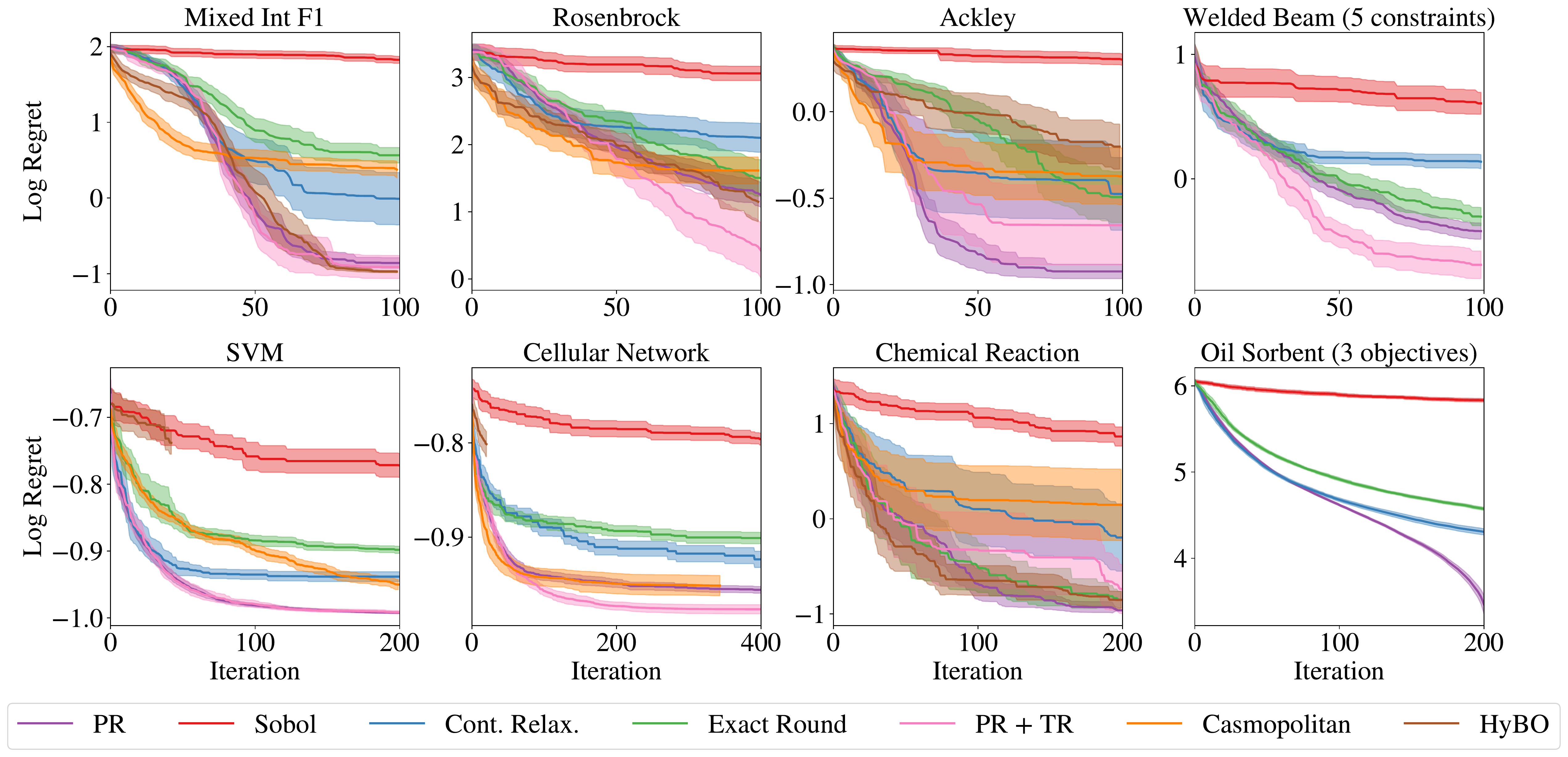}
    \caption{
        \PR{} (or \PR{} + \turbo{}) consistently outperforms alternatives with respect to log regret. 
    }
    \label{fig:exp1}
    \vspace{-1ex}
\end{figure*}
\FloatBarrier
\begin{figure*}[!ht]
    \centering
    \includegraphics[width=\linewidth]{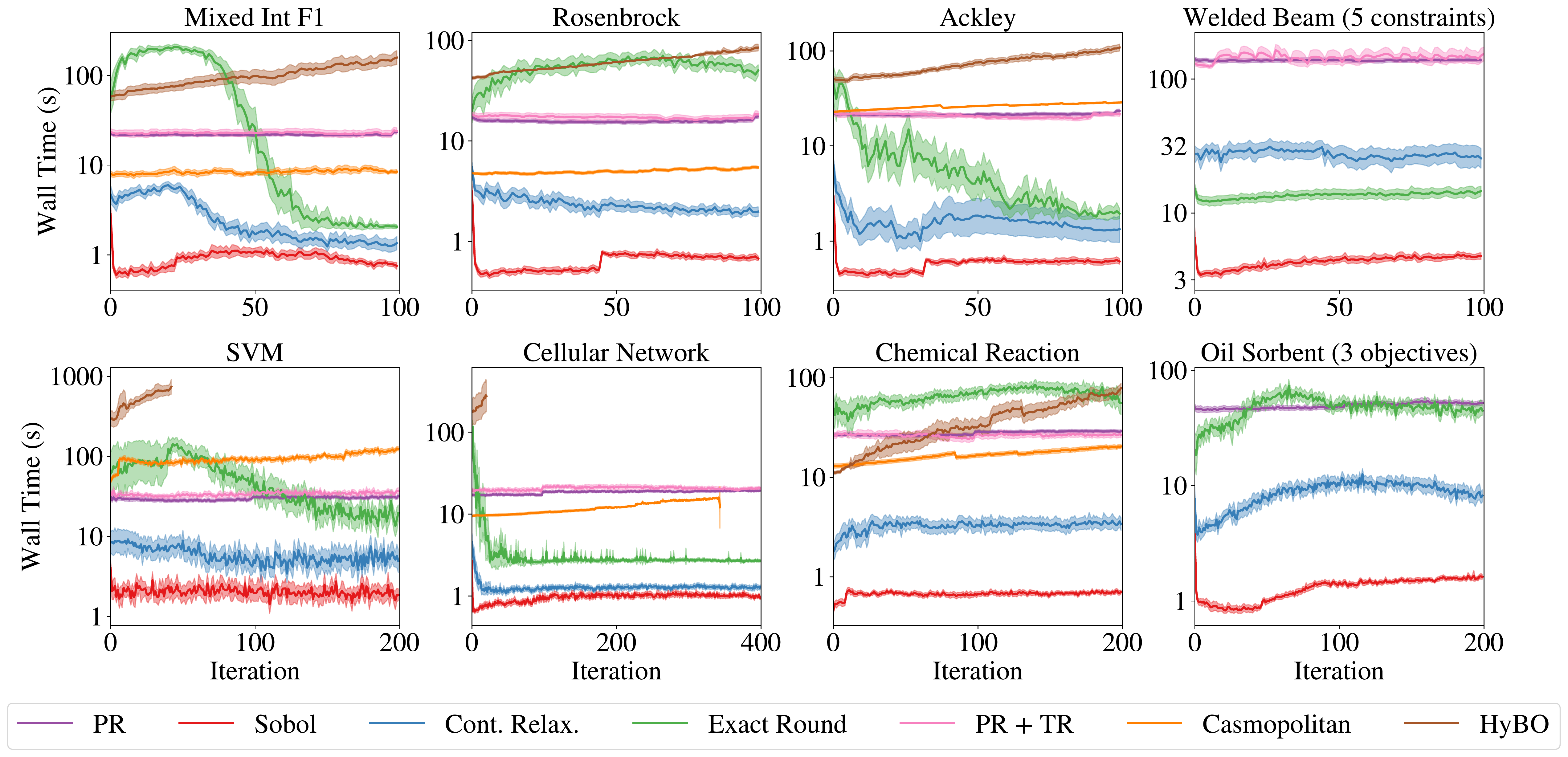}
    \caption{
        Wall time for candidate generation at each BO iteration in seconds. 
        \contbo{}, \exactround{}, \PR{}, and \PR{} + \turbo{} are run on a single Tesla V100-SXM2-16GB GPU and other methods are run on an Intel Xeon Gold 6252N CPU.}
    \label{fig:wall_times}
\end{figure*}

\subsection{Results}
\vspace{-1ex}
We find \PR{} consistently delivers strong empirical performance as shown in Figure~\ref{fig:exp1}. 
\emph{On all benchmark problems, \PR{} (or \PR{} + \turbo{}) outperforms all baseline methods (except for Mixed Int F1, where \hybo{} performs comparably)}.
Figure~\ref{fig:wall_times} shows the wall time for candidate generation over the number of BO iterations. 
Although \PR{} is computationally intensive, the computation is embarrassingly parallel and therefore exploiting GPU acceleration yields competitive wall times. 
Importantly, \PR{}'s wall time scales well with the number of observations and design parameters, unlike \hybo{} which scales poorly with both. 
However, the complexity of \PR{} scales additively in the number of GPs being used (e.g. outcomes being modeled), assuming they are evaluated sequentially. 
Hence, in multi-objective or constrained settings, \PR{} incurs a high cost in terms of wall time. However, empirically \PR{} achieves better optimization performance on constrained and multi-objective problems relative to \contbo{} and \exactround{}. 
We note that \casmo{} does not support multi-objective BO or constrained BO, and although \hybo{} could be used in those settings, it would be impractically slow because 1) its wall time would scale linearly with the number of modeled outcomes (using independent GPs) and 2) its diffusion kernel is non-differentiable, which would make optimizing hypervolume-based AFs slow \citep{daulton2020ehvi, NEURIPS2021_11704817}.

\vspace{-1ex}
\section{Discussion}
\vspace{-1ex}
The performance and regret properties of BO depend critically on properly maximizing the AF.  
For problems with discrete features, exhaustively trying all possible combinations of discrete values quickly becomes infeasible as the number of combinations grows.
Alternatives such as trying a subset of the possible combinations or resorting to continuous relaxations often leads to a failure to effectively optimize the AF which may result in sub-optimal BO performance.
As an alternative, we propose using \PR{} to better optimize the AF, and we demonstrate that \PR{} achieves strong performance on a large number of real-world problems.
Our approach is complementary to many other BO extensions, and combines seamlessly with, for example, trust region-based BO and specialized kernels for discrete parameters.
One limitation of \PR{} is that it requires computationally-demanding MC integration. However, given that the computation in \PR{} is embarrassingly parallel, it motivates for future research on optimizing AFs on distributed hardware.

\section*{Acknowledgements}
We thank Ben Letham, James Wilson, and Michael Cohen, as well as the members of the Oxford Machine Learning Research Group, for providing insightful feedback.


\FloatBarrier
\bibliographystyle{plainnat}
\bibliography{neurips_2022_updated}

\begin{thebibliography}{67}
\providecommand{\natexlab}[1]{#1}
\providecommand{\url}[1]{\texttt{#1}}
\expandafter\ifx\csname urlstyle\endcsname\relax
  \providecommand{\doi}[1]{doi: #1}\else
  \providecommand{\doi}{doi: \begingroup \urlstyle{rm}\Url}\fi

\bibitem[Balandat et~al.(2020)Balandat, Karrer, Jiang, Daulton, Letham, Wilson,
  and Bakshy]{balandat2020botorch}
Maximilian Balandat, Brian Karrer, Daniel~R. Jiang, Samuel Daulton, Benjamin
  Letham, Andrew~Gordon Wilson, and Eytan Bakshy.
\newblock Botorch: {A} framework for efficient monte-carlo {Bayesian}
  optimization.
\newblock In \emph{Advances in Neural Information Processing Systems 33: Annual
  Conference on Neural Information Processing Systems 2020, NeurIPS 2020,
  December 6-12, 2020, virtual}, 2020.

\bibitem[Baptista and Poloczek(2018)]{baptista2018bayesian}
Ricardo Baptista and Matthias Poloczek.
\newblock Bayesian optimization of combinatorial structures.
\newblock In \emph{Proc. of ICML}, volume~80 of \emph{Proceedings of Machine
  Learning Research}, pages 471--480. {PMLR}, 2018.

\bibitem[Bengio et~al.(2013)Bengio, L{\'e}onard, and
  Courville]{bengio2013estimating}
Yoshua Bengio, Nicholas L{\'e}onard, and Aaron Courville.
\newblock Estimating or propagating gradients through stochastic neurons for
  conditional computation.
\newblock \emph{ArXiv preprint}, abs/1308.3432, 2013.

\bibitem[Bergstra et~al.(2011)Bergstra, Bardenet, Bengio, and K{\'{e}}gl]{tpe}
James Bergstra, R{\'{e}}mi Bardenet, Yoshua Bengio, and Bal{\'{a}}zs
  K{\'{e}}gl.
\newblock Algorithms for hyper-parameter optimization.
\newblock In \emph{Advances in Neural Information Processing Systems 24: 25th
  Annual Conference on Neural Information Processing Systems 2011. Proceedings
  of a meeting held 12-14 December 2011, Granada, Spain}, pages 2546--2554,
  2011.

\bibitem[Berkenkamp et~al.(2019)Berkenkamp, Schoellig, and
  Krause]{JMLR:v20:18-213}
Felix Berkenkamp, Angela~P. Schoellig, and Andreas Krause.
\newblock No-regret {Bayesian} optimization with unknown hyperparameters.
\newblock \emph{Journal of Machine Learning Research}, 20\penalty0
  (50):\penalty0 1--24, 2019.
\newblock URL \url{http://jmlr.org/papers/v20/18-213.html}.

\bibitem[Bliek et~al.(2021)Bliek, Guijt, Verwer, and de~Weerdt]{bliek2021}
Laurens Bliek, Arthur Guijt, Sicco Verwer, and Mathijs de~Weerdt.
\newblock Black-box mixed-variable optimisation using a surrogate model that
  satisfies integer constraints.
\newblock GECCO '21, page 1851–1859. Association for Computing Machinery,
  2021.

\bibitem[Boros and Hammer(2002)]{BOROS2002155}
Endre Boros and Peter~L. Hammer.
\newblock Pseudo-boolean optimization.
\newblock \emph{Discrete Applied Mathematics}, 123\penalty0 (1):\penalty0
  155--225, 2002.
\newblock ISSN 0166-218X.
\newblock \doi{https://doi.org/10.1016/S0166-218X(01)00341-9}.

\bibitem[Daulton et~al.(2020)Daulton, Balandat, and Bakshy]{daulton2020ehvi}
Samuel Daulton, Maximilian Balandat, and Eytan Bakshy.
\newblock Differentiable expected hypervolume improvement for parallel
  multi-objective bayesian optimization.
\newblock In \emph{Advances in Neural Information Processing Systems 33: Annual
  Conference on Neural Information Processing Systems 2020, NeurIPS 2020,
  December 6-12, 2020, virtual}, 2020.

\bibitem[Daulton et~al.(2021)Daulton, Balandat, and
  Bakshy]{NEURIPS2021_11704817}
Samuel Daulton, Maximilian Balandat, and Eytan Bakshy.
\newblock Parallel {Bayesian} optimization of multiple noisy objectives with
  expected hypervolume improvement.
\newblock In M.~Ranzato, A.~Beygelzimer, Y.~Dauphin, P.S. Liang, and J.~Wortman
  Vaughan, editors, \emph{Advances in Neural Information Processing Systems},
  volume~34, pages 2187--2200. Curran Associates, Inc., 2021.

\bibitem[Daulton et~al.(2022{\natexlab{a}})Daulton, Cakmak, Balandat, Osborne,
  Zhou, and Bakshy]{daulton2022robust}
Samuel Daulton, Sait Cakmak, Maximilian Balandat, Michael~A Osborne, Enlu Zhou,
  and Eytan Bakshy.
\newblock Robust multi-objective {Bayesian} optimization under input noise.
\newblock \emph{arXiv preprint arXiv:2202.07549}, 2022{\natexlab{a}}.

\bibitem[Daulton et~al.(2022{\natexlab{b}})Daulton, Eriksson, Balandat, and
  Bakshy]{morbo}
Samuel Daulton, David Eriksson, Maximilian Balandat, and Eytan Bakshy.
\newblock Multi-objective {Bayesian} optimization over high-dimensional search
  spaces.
\newblock In \emph{Proceedings of the Thirty-Eighth Conference on Uncertainty
  in Artificial Intelligence, {UAI} 2022, Eindhoven, Netherlands, August 1-5,
  2022}, Proceedings of Machine Learning Research. {AUAI} Press,
  2022{\natexlab{b}}.

\bibitem[Davies and Morton(2016)]{davies2016recent}
Huw~ML Davies and Daniel Morton.
\newblock Recent advances in c--h functionalization.
\newblock \emph{The Journal of Organic Chemistry}, 81\penalty0 (2):\penalty0
  343--350, 2016.

\bibitem[Daxberger et~al.(2020)Daxberger, Makarova, Turchetta, and
  Krause]{mixed_var_bo}
Erik~A. Daxberger, Anastasia Makarova, Matteo Turchetta, and Andreas Krause.
\newblock Mixed-variable {Bayesian} optimization.
\newblock In \emph{Proceedings of the Twenty-Ninth International Joint
  Conference on Artificial Intelligence, {IJCAI} 2020}, pages 2633--2639.
  ijcai.org, 2020.
\newblock \doi{10.24963/ijcai.2020/365}.

\bibitem[Deshwal and Doppa(2021)]{deshwal2021}
Aryan Deshwal and Jana Doppa.
\newblock Combining latent space and structured kernels for {Bayesian}
  optimization over combinatorial spaces.
\newblock In M.~Ranzato, A.~Beygelzimer, Y.~Dauphin, P.S. Liang, and J.~Wortman
  Vaughan, editors, \emph{Advances in Neural Information Processing Systems},
  volume~34, pages 8185--8200. Curran Associates, Inc., 2021.
\newblock URL
  \url{https://proceedings.neurips.cc/paper/2021/file/44e76e99b5e194377e955b13fb12f630-Paper.pdf}.

\bibitem[Deshwal et~al.(2021{\natexlab{a}})Deshwal, Belakaria, and
  Doppa]{deshwal2021bayesian}
Aryan Deshwal, Syrine Belakaria, and Janardhan~Rao Doppa.
\newblock Bayesian optimization over hybrid spaces.
\newblock In \emph{Proc. of ICML}, volume 139 of \emph{Proceedings of Machine
  Learning Research}, pages 2632--2643. {PMLR}, 2021{\natexlab{a}}.

\bibitem[Deshwal et~al.(2021{\natexlab{b}})Deshwal, Belakaria, and
  Doppa]{deshwal2021mercer}
Aryan Deshwal, Syrine Belakaria, and Janardhan~Rao Doppa.
\newblock Mercer features for efficient combinatorial {Bayesian} optimization.
\newblock In \emph{Thirty-Fifth AAAI Conference on Artificial Intelligence
  (AAAI)}, pages 7210--7218, 2021{\natexlab{b}}.

\bibitem[Dreifuerst et~al.(2021)Dreifuerst, Daulton, Qian, Varkey, Balandat,
  Kasturia, Tomar, Yazdan, Ponnampalam, and Heath]{dreifuerst2021}
Ryan~M Dreifuerst, Samuel Daulton, Yuchen Qian, Paul Varkey, Maximilian
  Balandat, Sanjay Kasturia, Anoop Tomar, Ali Yazdan, Vish Ponnampalam, and
  Robert~W Heath.
\newblock Optimizing coverage and capacity in cellular networks using machine
  learning.
\newblock In \emph{ICASSP 2021-2021 IEEE International Conference on Acoustics,
  Speech and Signal Processing (ICASSP)}, pages 8138--8142. IEEE, 2021.

\bibitem[Dua and Graff(2017)]{uci}
Dheeru Dua and Casey Graff.
\newblock {UCI} machine learning repository, 2017.

\bibitem[Duchon et~al.(2004)Duchon, Flajolet, Louchard, and Schaeffer]{boltz}
Philippe Duchon, Philippe Flajolet, Guy Louchard, and Gilles Schaeffer.
\newblock Boltzmann samplers for the random generation of combinatorial
  structures.
\newblock \emph{Comb. Probab. Comput.}, 13\penalty0 (4–5):\penalty0
  577–625, jul 2004.
\newblock ISSN 0963-5483.
\newblock \doi{10.1017/S0963548304006315}.

\bibitem[Emmerich et~al.(2006)Emmerich, Giannakoglou, and
  Naujoks]{emmerich2006}
M.~T.~M. Emmerich, K.~C. Giannakoglou, and B.~Naujoks.
\newblock Single- and multiobjective evolutionary optimization assisted by
  gaussian random field metamodels.
\newblock \emph{IEEE Transactions on Evolutionary Computation}, 10\penalty0
  (4):\penalty0 421--439, 2006.

\bibitem[Eriksson and Poloczek(2021)]{scbo}
David Eriksson and Matthias Poloczek.
\newblock Scalable constrained {Bayesian} optimization.
\newblock In \emph{International Conference on Artificial Intelligence and
  Statistics}, pages 730--738. PMLR, 2021.

\bibitem[Eriksson et~al.(2019)Eriksson, Pearce, Gardner, Turner, and
  Poloczek]{turbo}
David Eriksson, Michael Pearce, Jacob~R. Gardner, Ryan Turner, and Matthias
  Poloczek.
\newblock Scalable global optimization via local {Bayesian} optimization.
\newblock In \emph{Advances in Neural Information Processing Systems 32: Annual
  Conference on Neural Information Processing Systems 2019, NeurIPS 2019,
  December 8-14, 2019, Vancouver, BC, Canada}, pages 5497--5508, 2019.

\bibitem[Frazier(2018)]{frazier2018tutorial}
Peter~I Frazier.
\newblock A tutorial on {Bayesian} optimization.
\newblock \emph{ArXiv preprint}, abs/1807.02811, 2018.

\bibitem[Gardner et~al.(2014)Gardner, Kusner, Xu, Weinberger, and
  Cunningham]{gardner2014constrained}
Jacob~R. Gardner, Matt~J. Kusner, Zhixiang~Eddie Xu, Kilian~Q. Weinberger, and
  John~P. Cunningham.
\newblock Bayesian optimization with inequality constraints.
\newblock In \emph{Proc. of ICML}, volume~32 of \emph{{JMLR} Workshop and
  Conference Proceedings}, pages 937--945. JMLR.org, 2014.

\bibitem[Garnett(2022)]{garnett_bayesoptbook_2022}
Roman Garnett.
\newblock \emph{{Bayesian Optimization}}.
\newblock Cambridge University Press, 2022.
\newblock in preparation.

\bibitem[Garrido-Merch{\'a}n and
  Hern{\'a}ndez-Lobato(2020)]{GarridoMerchn2020DealingWC}
Eduardo~C. Garrido-Merch{\'a}n and D.~Hern{\'a}ndez-Lobato.
\newblock Dealing with categorical and integer-valued variables in {Bayesian}
  optimization with gaussian processes.
\newblock \emph{Neurocomputing}, 380:\penalty0 20--35, 2020.

\bibitem[Glynn(1990)]{glynn1990likelihood}
Peter~W Glynn.
\newblock Likelihood ratio gradient estimation for stochastic systems.
\newblock \emph{Communications of the ACM}, 33\penalty0 (10):\penalty0 75--84,
  1990.

\bibitem[Hansen et~al.(2019)Hansen, Brockhoff, Mersmann, Tusar, Tusar, ElHara,
  Sampaio, Atamna, Varelas, Batu, et~al.]{hansencomparing}
Nikolaus Hansen, Dimo Brockhoff, Olaf Mersmann, Tea Tusar, Dejan Tusar,
  Ouassim~Ait ElHara, Phillipe~R Sampaio, Asma Atamna, Konstantinos Varelas,
  Umut Batu, et~al.
\newblock Comparing continuous optimizers: numbbo/coco on github.
\newblock 2019.

\bibitem[H{\"a}se et~al.(2021)H{\"a}se, Aldeghi, Hickman, Roch, and
  Aspuru-Guzik]{hase2021gryffin}
Florian H{\"a}se, Matteo Aldeghi, Riley~J Hickman, Lo{\"\i}c~M Roch, and
  Al{\'a}n Aspuru-Guzik.
\newblock Gryffin: An algorithm for {Bayesian} optimization of categorical
  variables informed by expert knowledge.
\newblock \emph{Applied Physics Reviews}, 8\penalty0 (3):\penalty0 031406,
  2021.

\bibitem[Hutter et~al.(2011)Hutter, Hoos, and Leyton-Brown]{smac}
Frank Hutter, Holger~H. Hoos, and Kevin Leyton-Brown.
\newblock Sequential model-based optimization for general algorithm
  configuration.
\newblock In \emph{Proceedings of the 5th International Conference on Learning
  and Intelligent Optimization}, page 507–523. Springer-Verlag, 2011.
\newblock ISBN 9783642255656.

\bibitem[Jang et~al.(2017)Jang, Gu, and Poole]{jang2016categorical}
Eric Jang, Shixiang Gu, and Ben Poole.
\newblock Categorical reparameterization with gumbel-softmax.
\newblock In \emph{Proc. of ICLR}. OpenReview.net, 2017.

\bibitem[Jones et~al.(1998)Jones, Schonlau, and Welch]{jones98}
Donald~R. Jones, Matthias Schonlau, and William~J. Welch.
\newblock Efficient global optimization of expensive black-box functions.
\newblock \emph{Journal of Global Optimization}, 13:\penalty0 455--492, 1998.

\bibitem[Kandasamy et~al.(2015)Kandasamy, Schneider, and
  Poczos]{pmlr-v37-kandasamy15}
Kirthevasan Kandasamy, Jeff Schneider, and Barnabas Poczos.
\newblock High dimensional {Bayesian} optimisation and bandits via additive
  models.
\newblock In Francis Bach and David Blei, editors, \emph{Proceedings of the
  32nd International Conference on Machine Learning}, volume~37 of
  \emph{Proceedings of Machine Learning Research}, pages 295--304, Lille,
  France, 07--09 Jul 2015. PMLR.
\newblock URL \url{https://proceedings.mlr.press/v37/kandasamy15.html}.

\bibitem[Kingma and Ba(2014)]{kingma2014adam}
Diederik~P Kingma and Jimmy Ba.
\newblock Adam: A method for stochastic optimization.
\newblock \emph{arXiv preprint arXiv:1412.6980}, 2014.

\bibitem[Kleijnen and Rubinstein(1996)]{kleijnen1996optimization}
Jack~PC Kleijnen and Reuven~Y Rubinstein.
\newblock Optimization and sensitivity analysis of computer simulation models
  by the score function method.
\newblock \emph{European Journal of Operational Research}, 88\penalty0
  (3):\penalty0 413--427, 1996.

\bibitem[Liu et~al.(2022)Liu, Feng, Eriksson, Letham, and Bakshy]{sparse_bo}
Sulin Liu, Qing Feng, David Eriksson, Benjamin Letham, and Eytan Bakshy.
\newblock Sparse {Bayesian} optimization.
\newblock \emph{arXiv preprint arXiv:2203.01900}, 2022.

\bibitem[Maddison et~al.(2017)Maddison, Mnih, and Teh]{maddison2016concrete}
Chris~J. Maddison, Andriy Mnih, and Yee~Whye Teh.
\newblock The concrete distribution: {A} continuous relaxation of discrete
  random variables.
\newblock In \emph{Proc. of ICLR}. OpenReview.net, 2017.

\bibitem[Maddox et~al.(2021)Maddox, Balandat, Wilson, and Bakshy]{maddox2021}
Wesley~J Maddox, Maximilian Balandat, Andrew~G Wilson, and Eytan Bakshy.
\newblock Bayesian optimization with high-dimensional outputs.
\newblock In \emph{Advances in Neural Information Processing Systems},
  volume~34, 2021.

\bibitem[Mohamed et~al.(2020)Mohamed, Rosca, Figurnov, and
  Mnih]{mohamed2020monte}
Shakir Mohamed, Mihaela Rosca, Michael Figurnov, and Andriy Mnih.
\newblock Monte carlo gradient estimation in machine learning.
\newblock \emph{J. Mach. Learn. Res.}, 21\penalty0 (132):\penalty0 1--62, 2020.

\bibitem[Nguyen et~al.(2020)Nguyen, Gupta, Rana, Shilton, and
  Venkatesh]{nguyen2019}
Dang Nguyen, Sunil Gupta, Santu Rana, Alistair Shilton, and Svetha Venkatesh.
\newblock Bayesian optimization for categorical and category-specific
  continuous inputs.
\newblock In \emph{The Thirty-Fourth {AAAI} Conference on Artificial
  Intelligence, {AAAI} 2020, The Thirty-Second Innovative Applications of
  Artificial Intelligence Conference, {IAAI} 2020, The Tenth {AAAI} Symposium
  on Educational Advances in Artificial Intelligence, {EAAI} 2020, New York,
  NY, USA, February 7-12, 2020}, pages 5256--5263. {AAAI} Press, 2020.

\bibitem[Oh et~al.(2019)Oh, Tomczak, Gavves, and Welling]{oh2019combinatorial}
ChangYong Oh, Jakub~M. Tomczak, Efstratios Gavves, and Max Welling.
\newblock Combinatorial {Bayesian} optimization using the graph cartesian
  product.
\newblock In \emph{Advances in Neural Information Processing Systems 32: Annual
  Conference on Neural Information Processing Systems 2019, NeurIPS 2019,
  December 8-14, 2019, Vancouver, BC, Canada}, pages 2910--2920, 2019.

\bibitem[Owen(2003)]{owen2003quasi}
Art~B Owen.
\newblock Quasi-monte carlo sampling.
\newblock \emph{Monte Carlo Ray Tracing: Siggraph}, 1:\penalty0 69--88, 2003.

\bibitem[Paria et~al.(2019)Paria, Kandasamy, and
  P{\'{o}}czos]{paria2018flexible}
Biswajit Paria, Kirthevasan Kandasamy, and Barnab{\'{a}}s P{\'{o}}czos.
\newblock A flexible framework for multi-objective {Bayesian} optimization
  using random scalarizations.
\newblock In \emph{Proceedings of the Thirty-Fifth Conference on Uncertainty in
  Artificial Intelligence, {UAI} 2019, Tel Aviv, Israel, July 22-25, 2019},
  volume 115 of \emph{Proceedings of Machine Learning Research}, pages
  766--776. {AUAI} Press, 2019.

\bibitem[Pelamatti et~al.(2021)Pelamatti, Brevault, Balesdent, Talbi, and
  Guerin]{pelamatti2021bayesian}
Julien Pelamatti, Lo{\"\i}c Brevault, Mathieu Balesdent, El-Ghazali Talbi, and
  Yannick Guerin.
\newblock Bayesian optimization of variable-size design space problems.
\newblock \emph{Optimization and Engineering}, 22\penalty0 (1):\penalty0
  387--447, 2021.

\bibitem[Rapin and Teytaud(2018)]{nevergrad}
J.~Rapin and O.~Teytaud.
\newblock {Nevergrad - A gradient-free optimization platform}.
\newblock \url{https://GitHub.com/FacebookResearch/Nevergrad}, 2018.

\bibitem[Rasmussen(2004)]{Rasmussen2004}
Carl~Edward Rasmussen.
\newblock Gaussian processes in machine learning.
\newblock In \emph{Advanced Lectures on Machine Learning: ML Summer Schools
  2003, Canberra, Australia, February 2 - 14, 2003, T{\"u}bingen, Germany,
  August 4 - 16, 2003, Revised Lectures}, 2004.

\bibitem[Robbins and Monro(1951)]{robbins_monro}
Herbert Robbins and Sutton Monro.
\newblock A stochastic approximation method.
\newblock \emph{The Annals of Mathematical Statistics}, 22\penalty0
  (3):\penalty0 400--407, 1951.
\newblock ISSN 00034851.

\bibitem[Ru et~al.(2020)Ru, Alvi, Nguyen, Osborne, and Roberts]{ru2020bayesian}
Bin~Xin Ru, Ahsan~S. Alvi, Vu~Nguyen, Michael~A. Osborne, and Stephen~J.
  Roberts.
\newblock Bayesian optimisation over multiple continuous and categorical
  inputs.
\newblock In \emph{Proc. of ICML}, volume 119 of \emph{Proceedings of Machine
  Learning Research}, pages 8276--8285. {PMLR}, 2020.

\bibitem[Russo and Roy(2014)]{russo2014learning}
Daniel Russo and Benjamin~Van Roy.
\newblock Learning to optimize via posterior sampling.
\newblock \emph{arXiv preprint: arXiv 1301.2609}, 2014.

\bibitem[Samal et~al.(2021)Samal, Swain, Bandopadhaya, Dandanov, Poulkov,
  Routray, and Palai]{samal2021dynamic}
Soumya~Ranjan Samal, Kaliprasanna Swain, Shuvabrata Bandopadhaya, Nikolay
  Dandanov, Vladimir Poulkov, Sidheswar Routray, and Gopinath Palai.
\newblock Dynamic coverage optimization for 5g ultra-dense cellular networks
  based on their user densities.
\newblock 2021.

\bibitem[Shields et~al.(2021)Shields, Stevens, Li, Parasram, Damani, Alvarado,
  Janey, Adams, and Doyle]{shields2021bayesian}
Benjamin~J Shields, Jason Stevens, Jun Li, Marvin Parasram, Farhan Damani,
  Jesus I~Martinez Alvarado, Jacob~M Janey, Ryan~P Adams, and Abigail~G Doyle.
\newblock Bayesian reaction optimization as a tool for chemical synthesis.
\newblock \emph{Nature}, 590\penalty0 (7844):\penalty0 89--96, 2021.

\bibitem[Srinivas et~al.(2010)Srinivas, Krause, Kakade, and Seeger]{ucb}
Niranjan Srinivas, Andreas Krause, Sham~M. Kakade, and Matthias~W. Seeger.
\newblock Gaussian process optimization in the bandit setting: No regret and
  experimental design.
\newblock In \emph{Proc. of ICML}, pages 1015--1022. Omnipress, 2010.

\bibitem[{The GPyOpt authors}(2016)]{gpyopt2016}
{The GPyOpt authors}.
\newblock {GPyOpt}: A {Bayesian} optimization framework in python.
\newblock \url{http://github.com/SheffieldML/GPyOpt}, 2016.

\bibitem[Tran et~al.(2019)Tran, Tran, and Wang]{tran2019constrained}
Anh Tran, Minh Tran, and Yan Wang.
\newblock Constrained mixed-integer gaussian mixture {Bayesian} optimization
  and its applications in designing fractal and auxetic metamaterials.
\newblock \emph{Structural and Multidisciplinary Optimization}, 59\penalty0
  (6):\penalty0 2131--2154, 2019.

\bibitem[Turner et~al.(2021)Turner, Eriksson, McCourt, Kiili, Laaksonen, Xu,
  and Guyon]{turner2021bayesian}
Ryan Turner, David Eriksson, Michael McCourt, Juha Kiili, Eero Laaksonen, Zhen
  Xu, and Isabelle Guyon.
\newblock {Bayesian} optimization is superior to random search for machine
  learning hyperparameter tuning: Analysis of the black-box optimization
  challenge 2020.
\newblock In \emph{NeurIPS 2020 Competition and Demonstration Track}, pages
  3--26, 2021.

\bibitem[Tu\v{s}ar et~al.(2019)Tu\v{s}ar, Brockhoff, and
  Hansen]{bbob_mixed_int}
Tea Tu\v{s}ar, Dimo Brockhoff, and Nikolaus Hansen.
\newblock Mixed-integer benchmark problems for single- and bi-objective
  optimization.
\newblock In \emph{Proceedings of the Genetic and Evolutionary Computation
  Conference}, GECCO '19, page 718–726. Association for Computing Machinery,
  2019.

\bibitem[Wan et~al.(2021)Wan, Nguyen, Ha, Ru, Lu, and Osborne]{wan2021think}
Xingchen Wan, Vu~Nguyen, Huong Ha, Bin~Xin Ru, Cong Lu, and Michael~A. Osborne.
\newblock Think global and act local: {Bayesian} optimisation over
  high-dimensional categorical and mixed search spaces.
\newblock In \emph{Proc. of ICML}, volume 139 of \emph{Proceedings of Machine
  Learning Research}, pages 10663--10674. {PMLR}, 2021.

\bibitem[Wan et~al.(2022)Wan, Lu, Parker-Holder, Ball, Nguyen, Ru, and
  Osborne]{wan2022bayesian}
Xingchen Wan, Cong Lu, Jack Parker-Holder, Philip~J Ball, Vu~Nguyen, Binxin Ru,
  and Michael Osborne.
\newblock Bayesian generational population-based training.
\newblock In \emph{ICLR Workshop on Agent Learning in Open-Endedness}, 2022.

\bibitem[Wang et~al.(2020{\natexlab{a}})Wang, Cai, Liu, Yang, and
  Ding]{oil_solbent}
Boqian Wang, Jiacheng Cai, Chuangui Liu, Jian Yang, and Xianting Ding.
\newblock Harnessing a novel machine-learning-assisted evolutionary algorithm
  to co-optimize three characteristics of an electrospun oil sorbent.
\newblock \emph{ACS Applied Materials \& Interfaces}, 12\penalty0
  (38):\penalty0 42842--42849, 2020{\natexlab{a}}.

\bibitem[Wang et~al.(2020{\natexlab{b}})Wang, Clark, Liu, and
  Frazier]{wang2016parallel}
Jialei Wang, Scott~C Clark, Eric Liu, and Peter~I Frazier.
\newblock Parallel {Bayesian} global optimization of expensive functions.
\newblock \emph{Operations Research}, 68\penalty0 (6):\penalty0 1850--1865,
  2020{\natexlab{b}}.

\bibitem[Wang et~al.(2017)Wang, Xiong, Ishibuchi, Wu, and Zhang]{WANG201725}
Rui Wang, Jian Xiong, Hisao Ishibuchi, Guohua Wu, and Tao Zhang.
\newblock On the effect of reference point in moea/d for multi-objective
  optimization.
\newblock \emph{Applied Soft Computing}, 58:\penalty0 25--34, 2017.
\newblock ISSN 1568-4946.
\newblock \doi{https://doi.org/10.1016/j.asoc.2017.04.002}.

\bibitem[Williams(1992)]{williams1992simple}
Ronald~J Williams.
\newblock Simple statistical gradient-following algorithms for connectionist
  reinforcement learning.
\newblock \emph{Machine learning}, 8\penalty0 (3):\penalty0 229--256, 1992.

\bibitem[Wilson et~al.(2018)Wilson, Hutter, and Deisenroth]{wilson2018maxbo}
James~T. Wilson, Frank Hutter, and Marc~Peter Deisenroth.
\newblock Maximizing acquisition functions for {Bayesian} optimization.
\newblock In \emph{Advances in Neural Information Processing Systems 31: Annual
  Conference on Neural Information Processing Systems 2018, NeurIPS 2018,
  December 3-8, 2018, Montr{\'{e}}al, Canada}, pages 9906--9917, 2018.

\bibitem[Yin et~al.(2019)Yin, Yue, and Zhou]{ARSM_ICML2019}
Mingzhang Yin, Yuguang Yue, and Mingyuan Zhou.
\newblock {ARSM:} augment-reinforce-swap-merge estimator for gradient
  backpropagation through categorical variables.
\newblock In \emph{Proc. of ICML}, volume~97 of \emph{Proceedings of Machine
  Learning Research}, pages 7095--7104. {PMLR}, 2019.

\bibitem[{Yin} et~al.(2020){Yin}, {Ho}, {Yan}, {Qian}, and
  {Zhou}]{yin2020probabilistic}
Mingzhang {Yin}, Nhat {Ho}, Bowei {Yan}, Xiaoning {Qian}, and Mingyuan {Zhou}.
\newblock {Probabilistic Best Subset Selection by Gradient-Based Optimization}.
\newblock \emph{arXiv e-prints}, 2020.

\bibitem[Zhang and Golovin(2020)]{golovin2020random}
Richard Zhang and Daniel Golovin.
\newblock Random hypervolume scalarizations for provable multi-objective black
  box optimization.
\newblock In \emph{Proc. of ICML}, volume 119 of \emph{Proceedings of Machine
  Learning Research}, pages 11096--11105. {PMLR}, 2020.

\bibitem[Zhang et~al.(2019)Zhang, Tao, Chen, and Apley]{zhang2019latent}
Yichi Zhang, Siyu Tao, Wei Chen, and Daniel Apley.
\newblock A latent variable approach to gaussian process modeling with
  qualitative and quantitative factors.
\newblock \emph{Technometrics}, 62:\penalty0 1--19, 07 2019.
\newblock \doi{10.1080/00401706.2019.1638834}.

\end{thebibliography}
\section*{Checklist}


\begin{enumerate}

\item For all authors...
\begin{enumerate}
  \item Do the main claims made in the abstract and introduction accurately reflect the paper's contributions and scope?
    \answerYes{}
  \item Did you describe the limitations of your work?
    \answerYes{}
  \item Did you discuss any potential negative societal impacts of your work?
    \answerYes{} See Appendix~\ref{appdx:society}.
  \item Have you read the ethics review guidelines and ensured that your paper conforms to them?
    \answerYes{}
\end{enumerate}

\item If you are including theoretical results...
\begin{enumerate}
  \item Did you state the full set of assumptions of all theoretical results?
    \answerYes{}
        \item Did you include complete proofs of all theoretical results?
    \answerYes{}
\end{enumerate}

\item If you ran experiments...
\begin{enumerate}
  \item Did you include the code, data, and instructions needed to reproduce the main experimental results (either in the supplemental material or as a URL)?
    \answerYes{}
  \item Did you specify all the training details (e.g., data splits, hyperparameters, how they were chosen)?
    \answerYes{}
        \item Did you report error bars (e.g., with respect to the random seed after running experiments multiple times)?
    \answerYes{}
        \item Did you include the total amount of compute and the type of resources used (e.g., type of GPUs, internal cluster, or cloud provider)?
    \answerYes{} See Appendix~\ref{appdx:exp_details}.
\end{enumerate}

\item If you are using existing assets (e.g., code, data, models) or curating/releasing new assets...
\begin{enumerate}
  \item If your work uses existing assets, did you cite the creators?
    \answerYes{}
  \item Did you mention the license of the assets?
    \answerYes{} See Appendix~\ref{appdx:exp_details}.
  \item Did you include any new assets either in the supplemental material or as a URL?
    \answerNo{}
  \item Did you discuss whether and how consent was obtained from people whose data you're using/curating?
    \answerNA
  \item Did you discuss whether the data you are using/curating contains personally identifiable information or offensive content?
    \answerNA
\end{enumerate}

\item If you used crowdsourcing or conducted research with human subjects...
\begin{enumerate}
  \item Did you include the full text of instructions given to participants and screenshots, if applicable?
    \answerNA
  \item Did you describe any potential participant risks, with links to Institutional Review Board (IRB) approvals, if applicable?
    \answerNA
  \item Did you include the estimated hourly wage paid to participants and the total amount spent on participant compensation?
    \answerNA
\end{enumerate}

\end{enumerate}
\clearpage
\appendix

\begin{center}
\hrule height 4pt
\vskip 0.25in
\vskip -\parskip
    {\LARGE\bf  Appendix to:\\[2ex] \papertitle}
\vskip 0.29in
\vskip -\parskip
\hrule height 1pt
\vskip 0.2in%
\end{center}

\section{Potential Societal Impacts}
\label{appdx:society}
Our work advances Bayesian optimization, a generic class of methods for optimization of expensive, difficult-to-optimize black-box problems. With this paper in particular, we improve the performance of Bayesian optimization on problems with mixed types of inputs. Given the ubiquity of such problems in many practical applications, we believe that our method could lead to positive broader impacts by solving these problems better and more efficiently while reducing the costs incurred for solving them. Concrete and high-stake examples where our method could be potentially applied (some of which have been already demonstrated by the benchmark problems considered in the paper) include but are not limited to applications in communications, chemical synthesis, drug discovery, engineering optimization, tuning of recommender systems, and automation of machine learning systems. On the flip side, while the method proposed is ethically neutral, there is potential of misuse given that the exact objective of optimization is ultimately decided by the end users; we believe that practitioners and researchers should be aware of such possibility and aim to mitigate any potential negative impacts to the furthest extent.

\section{Theoretical Results and Proofs}
\label{appdx:proofs}

\subsection{Results}

Let $\mathcal{P}_\mathcal{Z}^{(i)} := \mathcal{P}(\mathcal{Z}^{(i)})$ denote the set of probability measures on $\mathcal{Z}^{(i)}$ for each $i=1, ..., d_z$, and let $\mathcal{P}_\mathcal{Z}:= \prod_{i=1}^{d_z} \mathcal{P}_\mathcal{Z}^{(i)}$. For any $\alpha: \mathcal{X} \times \mathcal{Z} \rightarrow \mathbb{R}$, define $\tilde{\alpha}: \mathcal{X} \times \mathcal{P} \rightarrow \mathbb{R}$ as
\begin{align}
    \tilde{\alpha}(\bm x, p) = \int_\mathcal{Z} \alpha(\bm x, \bm z) dp(\bm z) = \sum_{\bm z\in \mathcal{Z}} \alpha(\bm x, \bm z) p(\{\bm z\}). 
\end{align}
Let $\Theta$ be a compact metric space, and consider the set of functionals $\Phi =\{\varphi ~s.t.~ \varphi: \Theta \rightarrow \mathcal{P}_\mathcal{Z}\}$. Let
\begin{align}
    \hat{\alpha}(\bm x, \bm \theta) := \tilde{\alpha}(\bm x, \bm \varphi(\bm \theta)) = \int_\mathcal{Z} \alpha(\bm x, \bm z) dp_{\varphi(\bm \theta)}(\bm z) = \sum_{\bm z\in \mathcal{Z}} \alpha(\bm x, \bm z) p_{\varphi(\bm \theta)}(\{\bm z\})
\end{align}
Since $\mathcal{Z}$ is finite, each element of $\varphi \in \Phi$ can be expressed as a mapping from $\Theta$ to $\mathbb{R}^{|\mathcal{Z}|}$. Namely, each $\varphi(\bm \theta)$ corresponds to a vector with $|\mathcal{Z}|$ elements containing the probability mass for each element of $\mathcal{Z}$ under $p_\varphi(\bm \theta)$.
Thus $(\mathcal{P}_\mathcal{Z}, \| \cdot \|)$ is a metric space under any norm $\|\cdot\|$ on $\mathbb{R}^{|Z|}$. Let $\alpha^* := \max_{(\bm x, \bm z) \in (\mathcal{X}\times \mathcal{Z})}\alpha(\bm x, \bm z)$ and let
$\mathcal{H}^* := \argmax_{(\bm x, \bm z) \in (\mathcal{X}\times \mathcal{Z})}\alpha(\bm x, \bm z)$ denote the set of maximizers of $\alpha$.
\begin{lemma}
    \label{lemma:general_consistent_maximizers}
    Suppose $\alpha$ is continuous in $\bm x$ for every $\bm z\in \mathcal{Z}$ and that $\varphi: \Theta \mapsto (\mathcal{P}_\mathcal{Z}, \| \cdot \|)$ is continuous with $\varphi(\Theta) = \mathcal{P}_\mathcal{Z}$. Let $\mathcal J^* := \argmax_{(\bm x,\bm \theta) \in \mathcal{X} \times \Theta} \hat{\alpha}(\bm x, \bm \theta)$. Then for any $(\bm x^*, \bm \theta^*) \in \mathcal J^*$, it holds that $(\bm x^*, \bm z) \in \mathcal{H}^*$ for all $z \in \mathrm{supp}\,  p_{\varphi(\bm \theta^*)}$.
\end{lemma}

\begin{proof}
    First, note that $\hat{\alpha}: \mathcal{X} \times \Theta \rightarrow \mathbb{R}$ is continuous (using that $\varphi$ is continuous and $\alpha$ is bounded). Since both $\mathcal{X}$ and $\Theta$ are compact $\hat{\alpha}$ attains its maximum, i.e.,  $\mathcal J^*$ exists.
    Let $(\bm x^*, \bm \theta^*) \in \mathcal J^*$. Clearly, there exists $\bm z^* \in \argmax_{\bm z\in \mathcal{Z}} \alpha(\bm x^*, \bm z)$ such that $\alpha(\bm x^*, \bm z^*) = \alpha^*$.    
    Suppose there exists $\bm z' \in \mathrm{supp}\,  p_{\varphi(\bm \theta^*)}$ such that $(\bm x^*, \bm z') \notin \mathcal{H}^*$. Then $\alpha(\bm x^*, \bm z') < \alpha^*$ and, since $\mathcal{Z}$ is finite, $ p_{\varphi(\bm \theta^*)}(\{\bm z'\}) > 0$. 
    Consider the probability measure $p'$ given by 
    \begin{align*}
        p'(\{\bm z\}) = \Biggl\{ \begin{array}{ll}
             0 & \text{ if } \bm z = \bm z' \\
             p_{\varphi(\bm \theta^*)}(\{\bm z^*\}) + p_{\varphi(\bm \theta^*)}(\{\bm z'\}) & \text{ if } \bm z = \bm z^* \\
             p_{\varphi(\bm \theta^*)}(\{\bm z\}) & \text{ otherwise}
        \end{array}
    \end{align*}
    Then 
    \begin{align*}
        \tilde{\alpha}(\bm x^*, p') - \hat{\alpha}(\bm x^*, \bm \theta^*) &= \sum_{\bm z\in \mathcal{Z}} \alpha(\bm x^*, z) p'(\{\bm z\}) - \hat{\alpha}(\bm x^*, \bm \theta^*) \\
        &= \sum_{\bm z\in \mathcal{Z}} \alpha(\bm x^*, \bm z) p_{\varphi(\bm \theta^*)}(\{\bm z\}) + p_{\varphi(\bm \theta^*)}(\{\bm z'\}) (\alpha(\bm x^*, \bm z^*) - \alpha(\bm x^*, \bm z'))\\ &~~~~~~~~~~- \hat{\alpha}(\bm x^*, \bm \theta^*) \\
        &=  p_{\varphi(\bm \theta^*)}(\{\bm z'\}) (\alpha(\bm x^*, \bm z^*) - \alpha(\bm x^*, \bm z')) \\
        &> 0
    \end{align*}
    Now $p' \in \mathcal{P}_{\mathcal{Z}}$, and so $p' = \varphi(\bm \theta')$ for some $\bm \theta' \in \Theta$. But then $\hat{\alpha}(\bm x^*, \bm \theta') > \hat{\alpha}(\bm x^*, \bm \theta^*)$. This is a contradiction.
\end{proof}

\begin{corollary}
    Suppose the optimizer of $g$ is unique, i.e., that $\mathcal H^* = \{(\bm x^*, \bm z^*)\}$ is a singleton. Then the optimizer of $\hat{\alpha}$ is also unique and $\mathcal J^* = \{(\bm x^*, \bm \theta^*)\}$, with $p_{\varphi(\bm \theta^*)}(\{\bm z^*\}) = 1$.
\end{corollary}

\begin{corollary}
\label{corr:concrete_mappings}
Consider the following mappings:
\begin{itemize}
    \item \textbf{Binary:} $\varphi: [0, 1] \rightarrow \mathcal{P}_{\{0, 1\}}$ with $p_{\varphi(\theta)}(\{1\}) = \theta$ and   $p_{\varphi(\theta)}(\{0\}) = 1- \theta$.
    \item \textbf{Ordinal:} $\varphi: [0, C-1] \rightarrow \mathcal{P}_{\{0, 1, \dotsc, C\}}$ with $p_{\varphi(\theta)}(\{i\}) = (1 - |i - \theta|)\, \mathbf{1}\{|i-\theta| \leq 1\}$ for $i=1, \dotsc, C$.
    \item \textbf{Categorical:} $\varphi: [0, 1]^C \rightarrow \mathcal{P}_{\{0, 1, \dotsc, C\}}$ with $p_{\varphi(\theta)}(\{i\}) = \frac{\theta_i}{\sum_{i=1}^C\theta_i}$. 
\end{itemize}
These mappings satisfy the conditions for Lemma~\ref{lemma:general_consistent_maximizers}. In the setting with multiple discrete parameters where the above mappings are applied in component-wise fashion for each discrete parameter, the component-wise mappings also satisfy the conditions for Lemma~\ref{lemma:general_consistent_maximizers}.
\end{corollary}
Clearly, the mappings given in Corollary~\ref{corr:concrete_mappings} are continuous functions of $\theta$. In the setting with multiple discrete parameters , the component-wise function is also continuous with respect to the distribution parameters for each discrete parameter. Hence, the mappings satisfy the conditions for Lemma~\ref{lemma:general_consistent_maximizers}.

\begin{lemma}
\label{lemma:equal_vals}
If $(\bm x^*, \bm z^*) \in \mathcal H^*=\argmax_{(\bm x, \bm z) \in \mathcal X \times \mathcal Z}\alpha(\bm x, \bm z)$, then
$$\alpha(\bm x^*, \bm z^*) = \max_{\bm\theta} \mathbb{E}_{\bm Z\sim p(\bm Z|\bm \theta)}[\alpha(\bm x^*, 
\bm Z)].$$
\end{lemma}
\begin{proof}

For any $\bm z^*$, let $\bm \theta^*$ be the parameters such that $p(\bm z^* | \bm \theta^*) = 1$ (i.e. a point mass on $\bm z^*$). From Equation \eqref{eq:prob_mass}, 
$$\mathbb{E}_{\bm Z\sim p(\bm Z|\bm \theta^*)}[\alpha(\bm x^*, 
\bm Z)] = \sum_{\bm z \in \mathcal Z} \alpha(\bm x^*, \bm z)p(\bm z|\bm \theta^*) = \alpha(\bm x^*, \bm z^*).$$
Claim: $\mathbb{E}_{\bm Z\sim p(\bm Z|\bm \theta^*)}[\alpha(\bm x^*, 
\bm Z)] = \max_{\bm\theta} \mathbb{E}_{\bm Z\sim p(\bm Z|\bm \theta)}[\alpha(\bm x^*, 
\bm Z)].$ 

Suppose there exists $\bm \theta'$ such that $\mathbb{E}_{\bm Z\sim p(\bm Z|\bm \theta')}[\alpha(\bm x^*, 
\bm Z)] > \mathbb{E}_{\bm Z\sim p(\bm Z|\bm \theta^*)}[\alpha(\bm x^*, 
\bm Z)]$. Since $(\bm x^*, \bm z^*) \in  \mathcal H^*$, $\alpha(\bm x^*, \bm z^*) = \max_{(\bm x, \bm z) \in \mathcal X \times \mathcal Z}\alpha(\bm x, \bm z)$. Hence, there is no convex combination of values of $\alpha$ that is greater than $\alpha(\bm x^*, \bm z^*)$. This is a contradiction.
\end{proof}

\consistentmaximizers*
\begin{proof}
From Lemma~\ref{lemma:general_consistent_maximizers}, we have that for any $(\bm x^*, \bm \theta^*) \in \mathcal J^*$, it holds that $(\bm x^*, \bm z) \in \mathcal{H}^*$ for all $z \in \mathrm{supp}\,  p_{\varphi(\bm \theta^*)}$.  Hence, $\hat{\mathcal{H}}^* \subseteq \mathcal H^*$. 

Now, let $(\bm x^*, \bm z^*) \in \mathcal{H}^*$. Let $\bm \theta^* \in \Theta$ such that $p(\bm z^*|\bm \theta^*) = 1$. From the proof of Lemma~\ref{lemma:equal_vals}, we have that $\mathbb{E}_{\bm Z\sim p(\bm Z|\bm \theta^*)}[\alpha(\bm x^*, 
\bm Z)] = \alpha(\bm x^*, 
\bm z^*).$ As in the proof of Lemma~\ref{lemma:equal_vals}, there is no convex combination of values of $\alpha$ greater than $\alpha(\bm x^*, \bm z^*)$. So $ \mathbb{E}_{\bm Z\sim p(\bm Z|\bm \theta^*)}[\alpha(\bm x^*, 
\bm Z)] = \max_{(\bm x, \bm \theta) \in \mathcal X \times \Theta}\mathbb{E}_{\bm Z\sim p(\bm Z|\bm \theta)}[\alpha(\bm x, 
\bm Z)]$, and therefore, $\bm x^*, \bm \theta^* \in \mathcal J^*$. Hence $(\bm x^*, \bm z^*) \in \hat{\mathcal H}^*$. So $\mathcal H^* \subseteq \hat{\mathcal H}^*$, and hence, $\hat{\mathcal H}^* = \mathcal H^*$.
\end{proof}

\begin{lemma}
Suppose that $\alpha: (\bm x, \bm z) \mapsto \mathbb{R}$ is differentiable with respect to~$\bm x$ for all $\bm z \in \mathcal{Z}$, and that the mapping $\varphi: \bm \theta \mapsto \mathcal{P}_{\mathcal{Z}}$ is such that $p_{\varphi(\bm \theta)}(\{\bm z\})$ is differentiable with respect to $\bm \theta$ for all $z \in \mathcal Z$. Then the probabilistic objective $\mathbb{E}_{\bm Z \sim p(\bm Z|\bm \theta)}[\alpha(\bm x, \bm Z)]$ is differentiable with respect to $(\bm x, \bm \theta)$.
\end{lemma}
\begin{proof}
For any $\bm z \in \mathcal Z$, the function $p(\bm z, \bm \theta) \alpha(\bm x, \bm z) = p_{\varphi(\bm \theta)}(\{\bm z\}) \alpha(\bm x, \bm z)$ is the product of two differentiable functions, hence differentiable. Therefore the probabilistic objective is a (finite) linear combination of differentiable functions, hence differentiable.
\end{proof}

\converge*

\begin{proof}
The binary and categorical mappings in Corollary~\ref{corr:concrete_mappings} are differentiable in $\theta$ (the ordinal mapping is differentiable almost everywhere\footnote{Technically, the arguments presented here do not prove convergence under the ordinal mapping, but we have found this to work well and reliably in practice. Alternatively, ordinal parameters could also just be treated as categorical ones in which case the convergence results hold. In practice, however, this introduces additional optimization variables that make the problem unnecessarily hard by removing the ordered structure from the problem.}). Since the acquisition function $\alpha: \mathcal{X \times Z} \rightarrow \mathbb{R}$ is differentiable in $\bm x$ for every $\bm z \in \mathcal{Z}$, this means that the PO is differentiable. Using the prescribed sequence of step sizes, optimizing the PO using stochastic gradient ascent will converge almost surely to a local maximum after a sufficient number of steps \citep{robbins_monro}. As we increase the number of randomly distributed starting points, the probability of not finding the global maximum of the PO will converge to zero \citep{wang2016parallel}. From Theorem~\ref{thm:consistent_maximizers}, the PO and the AF have the same set of maximizers. Hence, convergence in probability to a global maximizer of the PO means convergence in probability to a global maximizer of the AF.
\end{proof}




\section{Experiment Details}
\label{appdx:exp_details}
For each BO optimization replicate, we use $N_\text{init} = \min(20, 2*d_\text{eff})$ points from a scrambled Sobol sequence, where $d_\text{eff}$ is the ``effective dimensionality'' after one-hot encoding categorical parameters.  Unless otherwise noted, all experiments use 20 replications and confidence intervals represent 2 standard errors of the mean. The same initial points are used for all methods for that replicate and different initial points are used for each replicate. For each method we report the $\log_{10}$ regret. Since the optimal value is unknown for many problems, we set the optimal value to be  $f^*+0.1$ where $f^*$ is the best observed value across all methods and all replications. For constrained optimization $f^*$ is the best feasible observed value and for multi-objective optimization $f^*$ is the maximum hypervolume across all methods and replications. In total, the experiments in the main text (excluding \hybo{} and \casmo{})  ran for an equivalent of 2,009.82 hours on a single Tesla V100-SXM2-16GB GPU. The baseline experiments (\hybo{} and \casmo{}) ran for an equivalent of 745.10 hours on a single Intel Xeon Gold 6252N CPU.

\subsection{Additional Problem Details}

In this section, we describe the details of each synthetic problem considered in the experiments (the details of the remaining real-world problems are already described in Section \ref{subsec:real_world_problems}).

\paragraph{Ackley.} We use an adapted version of the 13-dimensional Ackley function modified from \citet{bliek2021}. The function is given by:
\begin{equation}
    f(\mathbf{x}) = -a \exp\Big(-b\sqrt{\frac{1}{d}\sum_{i=1}^d x_i^2)}\Big) -  \exp\Big(\frac{1}{d}\sum_{i=1}^d \cos(cx_i)\Big) + a + \exp({1}),
\end{equation}
where in this case $a=20, n=0.2, c=2\pi$ and $d=13$ and $\mathbf{x} \in [-1, 1]^{13}$. We discretize the first 10 dimensions to be binary with the choice $\{-1, 1\}$, and the final 3 dimensions are unmodified with the original bounds.

\paragraph{Mixed Int F1.} Mixed Int F1 is a partially discretized version of the 16-dimensional Sphere optimization problem \citep{hansencomparing}, given by:
\begin{equation}
    f(\mathbf{x}) = \sum_{i=1}^d (x_i - x_{\mathrm{opt}, i})^2 + f_{\mathrm{opt}},
\end{equation}
where $f_{\mathrm{opt}}$ is sampled from a Cauchy distribution with median = 0 and roughly 50\% of the values between $-100$ and $100$. The sampled $f_{\mathrm{opt}}$ is then clamped to be between $[-1000, 1000]$ and rounded to the nearest integer. $\mathbf{x}_{\mathrm{opt}}$ is sampled uniformly in $[-4, 4]^d$, and in this case $d = 16$. We discretize the first 8 dimensions as follows: the first 2 dimensions are binary with 2 choices $\{-5, 5\}$; the next 2 dimensions are ordinal with 3 choices $\{-5, 0, 5\}$; the next 2 dimensions are ordinal with 5 choices $\{-5, -2.5, 0, 2.5, 5\}$; the final 2 dimensions are ordinal with 7 choices $\{-5, -\frac{10}{3}, -\frac{5}{3}, 0, \frac{5}{3}, \frac{10}{3}, 5\}$. The remaining 8 dimensions are continuous with bounds $[-5, 5]^8$.

\paragraph{Rosenbrock.} We use an adapted version of the Rosenbrock function, given by:
\begin{equation}
    f(\mathbf{x}) = \Big( \sum_{i=1}^{d-1} \big( 100(x_{i+1} - x_i^2)^2 + (x_i - 1)^2 \big) \Big),
\end{equation}
where in this case $d=10$. The first 6 dimensions are discretized to be ordinal variables, with 4 possible values each $x_i \in \{-5, 0, 5, 10\} \forall i \in [1, 6]$. The final 4 dimensions are continuous with bounds $[-5, 10]^4$.

\paragraph{Chemical Reaction (Direct Arylation Chemical Synthesis).} For this problem, we fit a GP surrogate (with the same kernel used by the BO methods) to the dataset from \citet{shields2021bayesian}  (available at \url{https://github.com/b-shields/edbo/tree/master/experiments/data/direct_arylation} under the MIT license) in order to facilitate continuous optimization of temperature and concentration. The surrogate is included with our source code. 

\paragraph{Oil Sorbent.} We set the reference point for this problem to be $[-125.3865, -57.8292, 43.2665]$, which we choose using a commonly used heuristic to scale the nadir point (component-wise worst objective values across the Pareto frontier) \citep{WANG201725}.

\subsection{Method details}
\textbf{\PR{}, \contbo{}, \exactround{}, \PR{} + \turbo{}, and \exactroundste{}}.  We implemented all of these methods using BoTorch \citep{balandat2020botorch}, which is available under the MIT license at \url{https://github.com/pytorch/botorch}. \PR{} and \PR{} + \turbo{} use stochastic minibatches of 128 samples and the probabilistic objectives are optimized via Adam using a learning rate of $\frac{1}{40}$. The AFs of \contbo{}, \exactround{}, \exactroundste{} are deterministic and are optimized via L-BFGS-B---\exactround{} approximates gradients via finite differences \citep{GarridoMerchn2020DealingWC}. All methods use 20 random restarts and are run for a maximum of 200 iterations. We follow the default initialization heuristic in BoTorch \citep{balandat2020botorch}, which initializes the optimizer by evaluating the acquisition function at a large number of starting points (here, 1024, chosen from a scrambled Sobol sequence), and selecting (20) points using Boltzmann sampling~\citep{boltz} of the 1024 initial points, according to their acquisition function utilities.

\textbf{Combining \PR{} with trust regions}: When combining \PR{} with the trust regions used in \textsc{TuRBO} we only use a trust region over the continuous parameters and discrete ordinals with at least 3 values.
While methods like \casmo{} uses a Hamming distance for the trust regions over the categorical parameters, we choose not to do so as there is no natural way of efficiently optimizing \PR{} using gradient-based methods.
Finally, we do not use a trust region over the Boolean parameters as the trust region will quickly shrink to only include one possible value.
We use the same hyperparameters as \textsc{TuRBO}~\citep{turbo} for unconstrained problems and \textsc{SCBO}~\citep{scbo} in the presence of outcome constraints, including default trust region update settings. 

\textbf{Casmopolitan}: We use the implementation of \casmo{}---which is available at \url{https://github.com/xingchenwan/Casmopolitan} under the MIT licence---but modify it where appropriate to additionally handle the ordinal variables. Specifically, the ordinal variables are treated as continuous when computing the kernel. However, during interleaved search, ordinal variables are searched via local search similar to the categorical variables. We use a set of \casmo{} hyperparameters (i.e. success/failure sensitivity, initial trust region sizes and expansion factor) recommended by the authors. We use the same implementation of interleaved search for the acquisition optimization comparisons.

\textbf{HyBO}: We use the official implementation of \hybo{} at \url{https://github.com/aryandeshwal/HyBO}, which is licensed by the University of Amsterdam. We use the default hyperparameters recommended by the authors in all the experiments, and we use the full \hybo{} method with marginalization treatment of the hyperparameters as it has been shown to perform stronger empirically \citep{deshwal2021}.

\subsection{Gaussian process regression}
When there are no categorical variables we use $k_{\text{ordinal}}$ which is a product of an isotropic Matern-5/2 kernel for the binary parameters and a Matern-5/2 kernel with ARD for the remaining ordinal parameters.
In the presence of categorical parameters, this kernel is combined with a categorical kernel~\citep{ru2020bayesian} $k_{\text{cat}}$ as $k_{\text{cat}} \times k_{\text{ordinal}} + k_{\text{cat}} + k_{\text{ordinal}}$.
We use a constant mean function.
The GP hyperparameters are fitted using L-BFGS-B by optimizing the log-marginal likelihood. 
The ranges for the ordinal parameters are rescaled to $[0, 1]$ and the outcomes are standardized before fitting the GP.

\subsection{Variance Reduction via Control Variates}
As discussed in Section \ref{sec:var_reduction}, we use moving average baseline for variance reduction. Specifically, the baseline is an exponential moving average with a multiplier of 0.7, where each element is the mean acquisition value across the $N$ MC samples obtained while evaluating the probabilistic objective.
\subsection{Deterministic Optimization via Sample Average Approximation}
\label{appdx:saa}
Although multi-start stochastic ascent is provably convergent, an alternative optimization approach is to use common random numbers (i.e. a fixed set of base samples) to reduce variance when comparing a stochastic function at different inputs by using the same random numbers. The method of common random numbers leads to biased deterministic estimators that are lower-variance than their stochastic counterparts where random numbers are re-sampled at each step. Such techniques have been used in the context of BO in settings such as efficiently optimizing MC acquisition functions \citep{balandat2020botorch} and for optimizing risk measures of acquisition functions under random inputs \citep{daulton2022robust}.

Sampling a fixed set of points $\tilde{\bm z}_1, ..., \tilde{\bm z}_{N} \sim p(\bm Z | \bm \theta)$ would be a poor choice because $p(\bm Z | \bm \theta)$ can vary widely during AF optimization as $\bm \theta$ changes. Therefore, instead sample from $p(\bm Z | \bm \theta)$ using reparameterizations provided in Table~\ref{table:saa}. Specifically, we reparameterize $\bm Z$ as a deterministic function $h(\cdot, \cdot)$ that operates component-wise on $\bm \theta$ and the random variable $\bm U = (u^{(1)}, ..., u^{(d_z)}), u^{(i)} \sim \text{Uniform}(0,1)$: $\bm Z = h(\bm \theta, \bm U)$. That is, each random variable $Z^{(j)}$, where $j=1, ..., d_z$ has a corresponding independent base random variable $U^{(j)}$ such that $Z^{(j)}=h(\theta^{(j)}, U^{(j)})$.  Using a fixed a set of base samples $\{\tilde{\bm u}_i\}_{i=1}^N$, the samples of $\bm Z$ can be be computed as $\bm z_i = h(\bm \theta, \tilde{\bm u}_i)$. We note that even with fixed base samples, the samples $\{\bm z_i\}_{i=1}^N$ depends on $\bm \theta$, and hence, by using common \emph{base} uniform samples, we obtain a deterministic estimator where the values of the samples $\tilde{\bm z}_1, ..., \tilde{\bm z}_{N}$ can still vary with $\bm \theta$. Under this reparameterization, our probabilistic objective can be written as
\begin{equation}
\label{eqn:reparam_saa_po}
  \mathbb E_{\bm Z \sim p(\bm Z|\bm \theta)} [\alpha(\bm x,\bm Z)] = \mathbb E_{\bm U \sim p(\bm U)} [\alpha(\bm x, h(\bm \theta, \bm U))],
\end{equation}
where under the reparameterizations in Table~\ref{table:saa}, $U$ is a uniform random variable across the $d_z$-dimensional unit cube---$P(\bm U) = 
\text{Uniform}(0,1)^d_z$. Under this reparameterization we can define our sample average approximation estimator of the probabilistic objective as
\begin{equation}
\label{eqn:saa_estimator_po}
  \mathbb E_{\bm Z \sim p(\bm Z|\bm \theta)} [\alpha(\bm x,\bm Z)] \approx  \frac{1}{N}\sum_{i=1}^N \alpha(\bm x, h(\bm \theta, \tilde{\bm u}_i)).
\end{equation}
Our sample average approximation estimator of the gradient  of the probabilistic objective with respect to $\bm \theta$ is given by
\begin{equation}
    \label{eq:saa_score_func_estimator}
    \nabla_{\bm \theta} \mathbb E_{\bm Z \sim p(\bm Z | \bm \theta)}[\alpha(\bm x, \bm Z)] \approx \frac{1}{N}\sum_{i=1}^N \alpha(\bm x, h(\bm \theta, \tilde{\bm u}_i))\nabla_{\bm \theta}\log p(h(\bm \theta, \tilde{\bm u}_i) | \bm \theta).
\end{equation}

Sample average approximation estimators are deterministic and biased conditional on the selection of base samples. However, the reparameterizations in Table \ref{table:saa} create discontinuities in the PO, and the number of discontinuities increases with the number of MC samples. Nevertheless, we find that optimizing the PO using L-BFGS-B delivers strong performance on the benchmark problems and we compare against stochastic optimization in Figures~\ref{fig:saa_vs_adam} and \ref{fig:sgd}. As in the stochastic case, we reduce the variance further by leveraging quasi-MC sampling \citep{owen2003quasi} instead of i.i.d. sampling.

\begin{table*}[!ht]
    \centering
    \caption{\label{table:saa} Discrete random variables and their reparameterizations in terms of a Uniform random variable $U\sim \text{Uniform}(0,1)$ and $\theta$ via a deterministic function $h(\cdot, \cdot)$.}
    \begin{small}
    \begin{sc}
    \begin{tabular}{lll}
        \toprule
        Type & Random Variable & Reparameterization ($Z = h(\theta, U))$  \\
        \midrule
         Binary & $Z \sim \text{Bernoulli}(\theta)$& $h(\theta, U) =  \mathbbm{1}(U < \theta)$\\
        Ordinal & $Z = \lfloor \theta \rfloor + B$,& $h(\theta, U) =  \lfloor \theta \rfloor + \mathbbm{1}(U < \theta - \lfloor \theta \rfloor)$\\
        & $B \sim \text{Bernoulli}(\theta - \lfloor \theta \rfloor)$& \\ 
        Categorical &$Z \sim \text{Categorical}(\bm \theta)$& $h(\theta, U) =\min(  \argmax_{i=0 }^{C-1}\mathbbm{1}(U < \sum_c^{i}\theta^{(c)}))$\\
        \bottomrule
    \end{tabular}
    \end{sc}
    \end{small}
    \vspace{-1ex}
\end{table*}

\section{Constrained and Multi-Objective Bayesian Optimization}
\label{appdx:constrained_bo_mobo}
In many practical problems, the black-box objective must be maximized subject to $V > 0$ black-box outcome constraints $f_c^{(v)}(\bm x, \bm z) \geq 0$ for $v=1, ..., V$. See \citet{gardner2014constrained} for a more in depth review of black-box optimization with black-box constraints and BO techniques for this class of problems.

In the multi-objective setting, the goal is to maximize (without loss of generality) a set of $M$ objectives $f^{(1)}, ..., f^{(M)}$. Typically there is no single best solution, and hence the goal is to learn the Pareto frontier (i.e. the set of optimal trade-offs between objectives). In the multi-objective setting, the hypervolume indicator is a common metric for evaluating the quality of a Pareto frontier. See \citep{emmerich2006} for a review of multi-objective optimization.

\FloatBarrier
\section{Comparison with Enumeration}
When computationally feasible, the gold standard for acquisition optimization over discrete and mixed search spaces is to enumerate the discrete options and optimize any continuous parameters for each discrete configuration (or simply evaluated each discrete configuration for fully discrete spaces). In Figures~\ref{fig:enumerate} and \ref{fig:enumerate_time} we compare \PR{} (optimized with Adam using stochastic mini-batches of 128 MC samples) and analytic \PR{} (optimized with L-BFGS-B) against enumeration and show that \PR{} achieves log regret performance that is comparable to the gold standard of enumeration and does so in less wall time.

\FloatBarrier
\begin{figure*}[ht]
    \centering
    \includegraphics[width=0.5\linewidth]{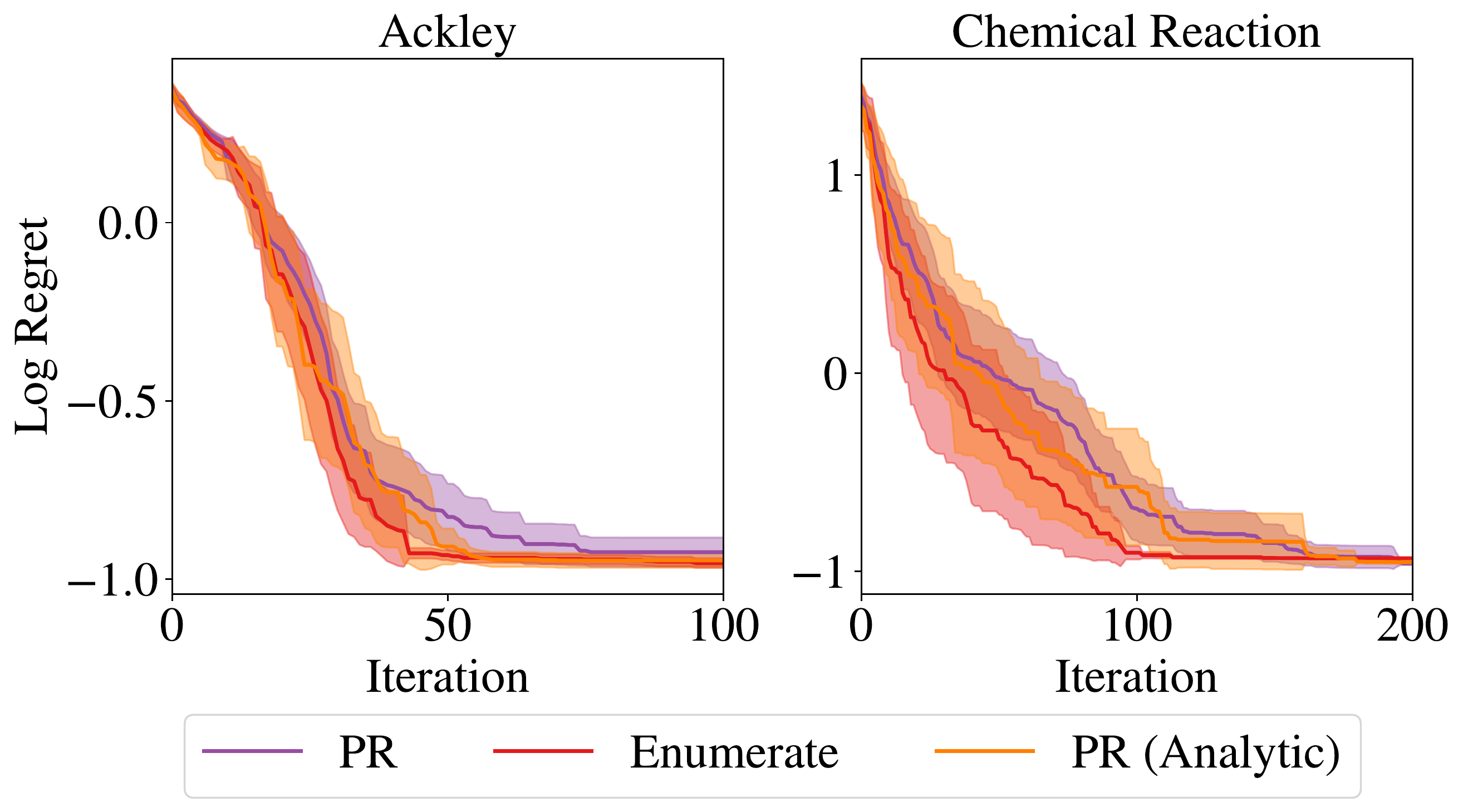}
    \caption{A comparison with an enumeration (gold standard) with respect to log regret.}
    \label{fig:enumerate}
\end{figure*}
\begin{figure*}[ht]
    \centering
    \includegraphics[width=0.5\linewidth]{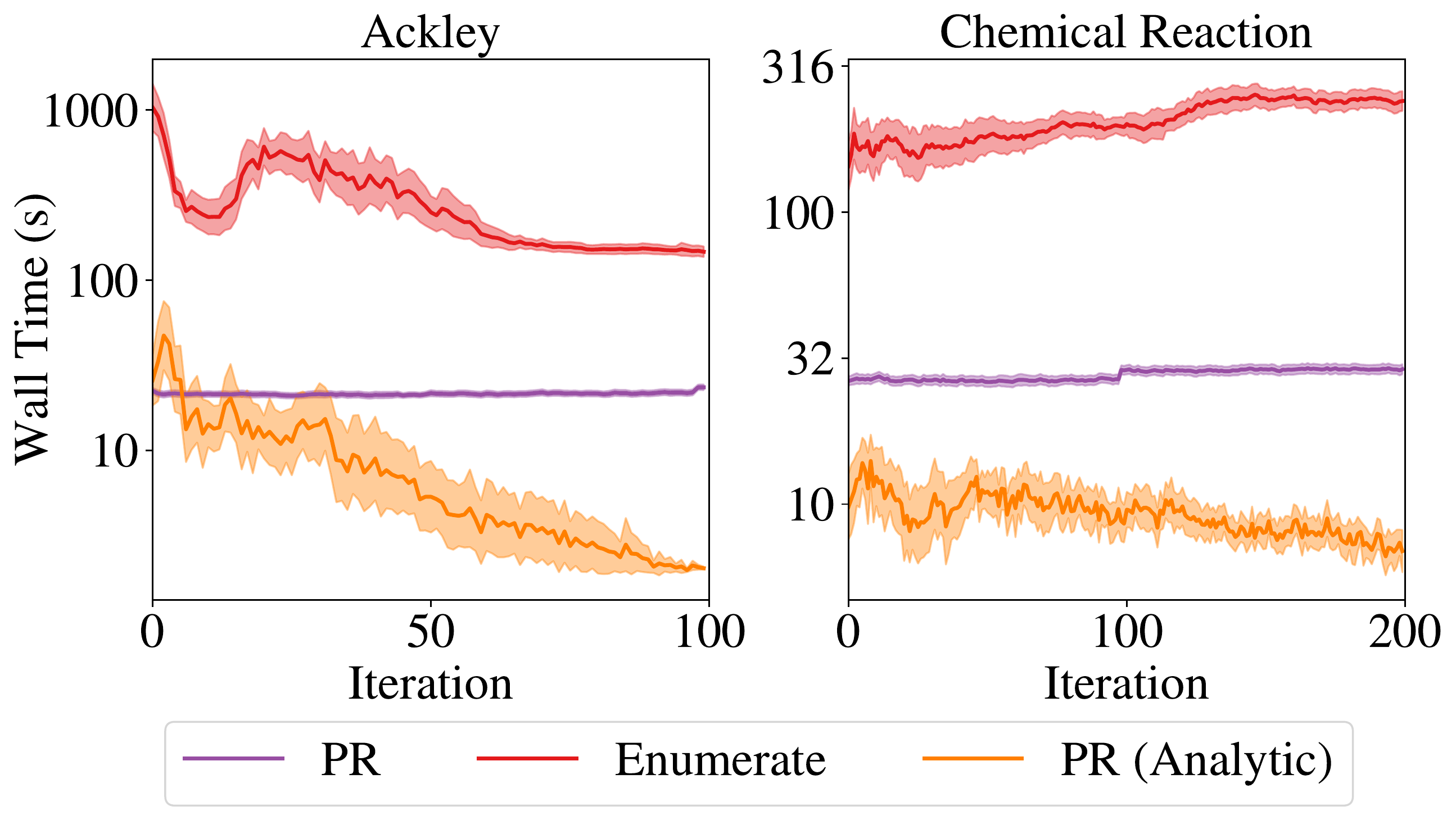}
    \caption{A comparison with enumeration with respect to wall time.}
    \label{fig:enumerate_time}
\end{figure*}
\FloatBarrier
\section{Analysis of MC sampling in Probabilistic Reparameterization }
\label{appdx:pr_mc_analysis}
The main text considers 1024 MC for \PR{}.  We consider 128, 256, and 512 samples, in addition to the default of 1024.  For problems with discrete spaces that are enumerable, we also consider analytic \PR{}. We do not find statistically significant differences between the final regret of any of these configurations (Figure~\ref{fig:exp1_mc_comparison}).  Run time is linear with respect to MC samples, and so substantial compute savings are possible when fewer MC samples are used (Figure~\ref{fig:wall_times_mc_comparison}). We observe comparable performance between \PR{} with 1,024 MC samples and as few as 128 MC samples. With 64 or fewer MC samples, we observe performance degradation with respect to log regret in Figure~\ref{fig:exp1_mc_comparison_fewer}, although wall time is considerably faster for fewer 64 or less MC samples as shown in Figure~\ref{fig:wall_times_mc_comparison_fewer}.
\FloatBarrier
\begin{figure*}[ht]
    \centering
    \includegraphics[width=\linewidth]{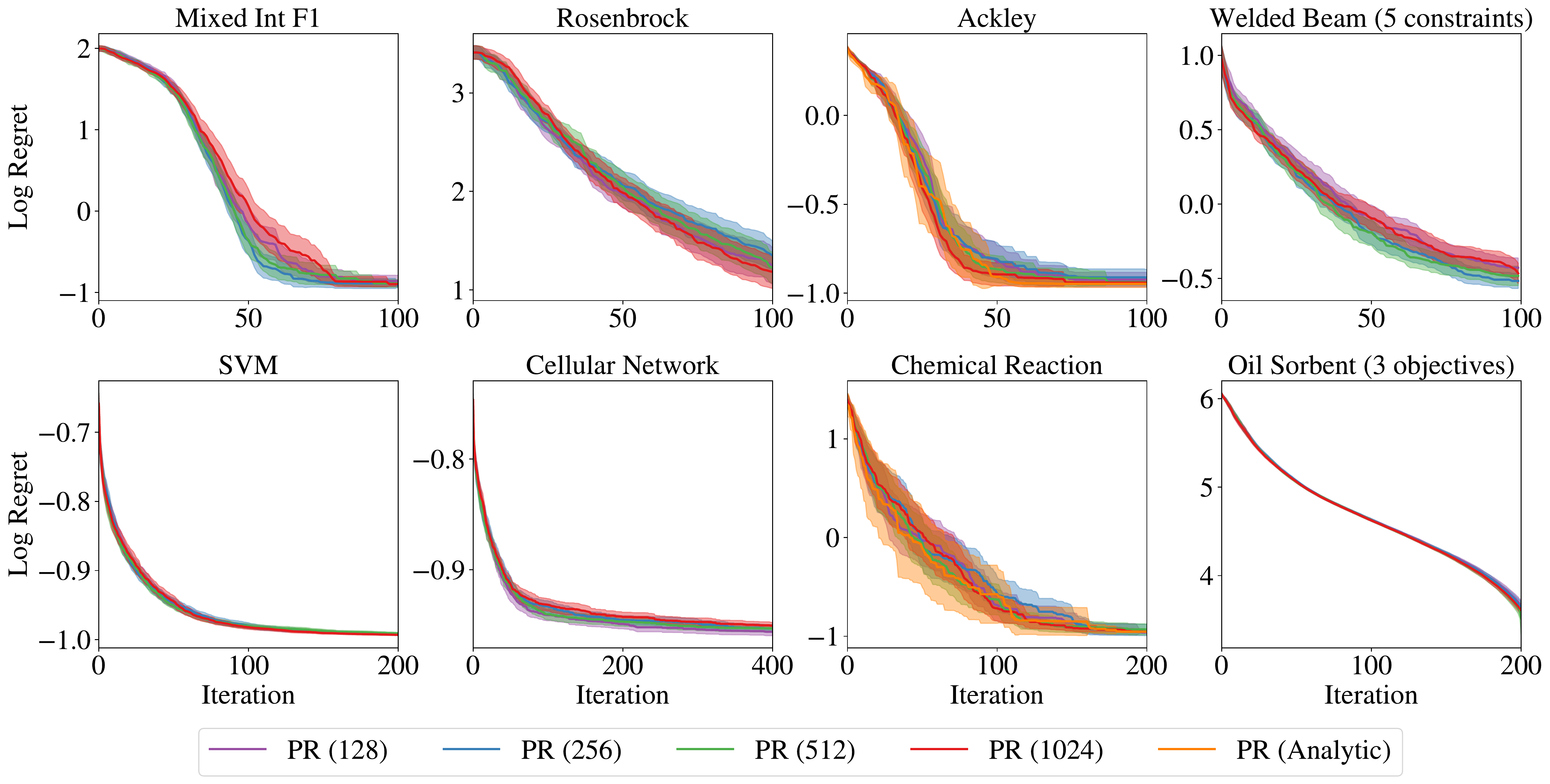}
    \caption{A sensitivity analysis of the optimization performance of \PR{} with respect to the number of MC samples. We find that \PR{} is robust to the number of MC samples, and that the performance of MC \PR{} matches that of analytic \PR{}.
    }
    \label{fig:exp1_mc_comparison}
\end{figure*}
\begin{figure*}[ht]
    \centering
    \includegraphics[width=\linewidth]{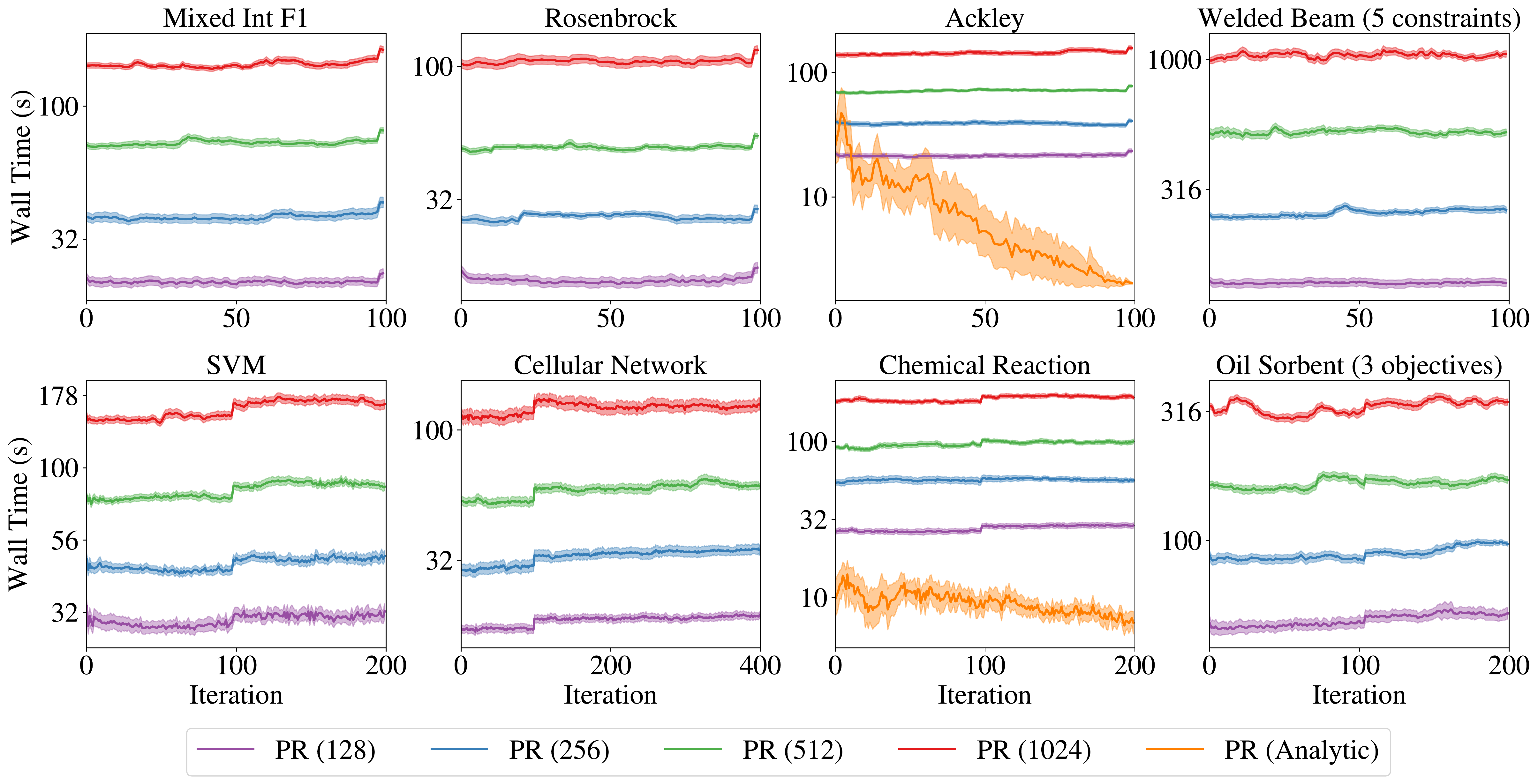}
    \caption{A sensitivity analysis of the wall time of \PR{} with respect to the number of MC samples. We observe that wall time scales linearly with the number of MC samples, which is expected since we compute \PR{} in $\frac{N}{32}$ chunks to avoid overflowing GPU memory.}
    \label{fig:wall_times_mc_comparison}
\end{figure*}
\begin{figure*}[ht]
    \centering
    \includegraphics[width=\linewidth]{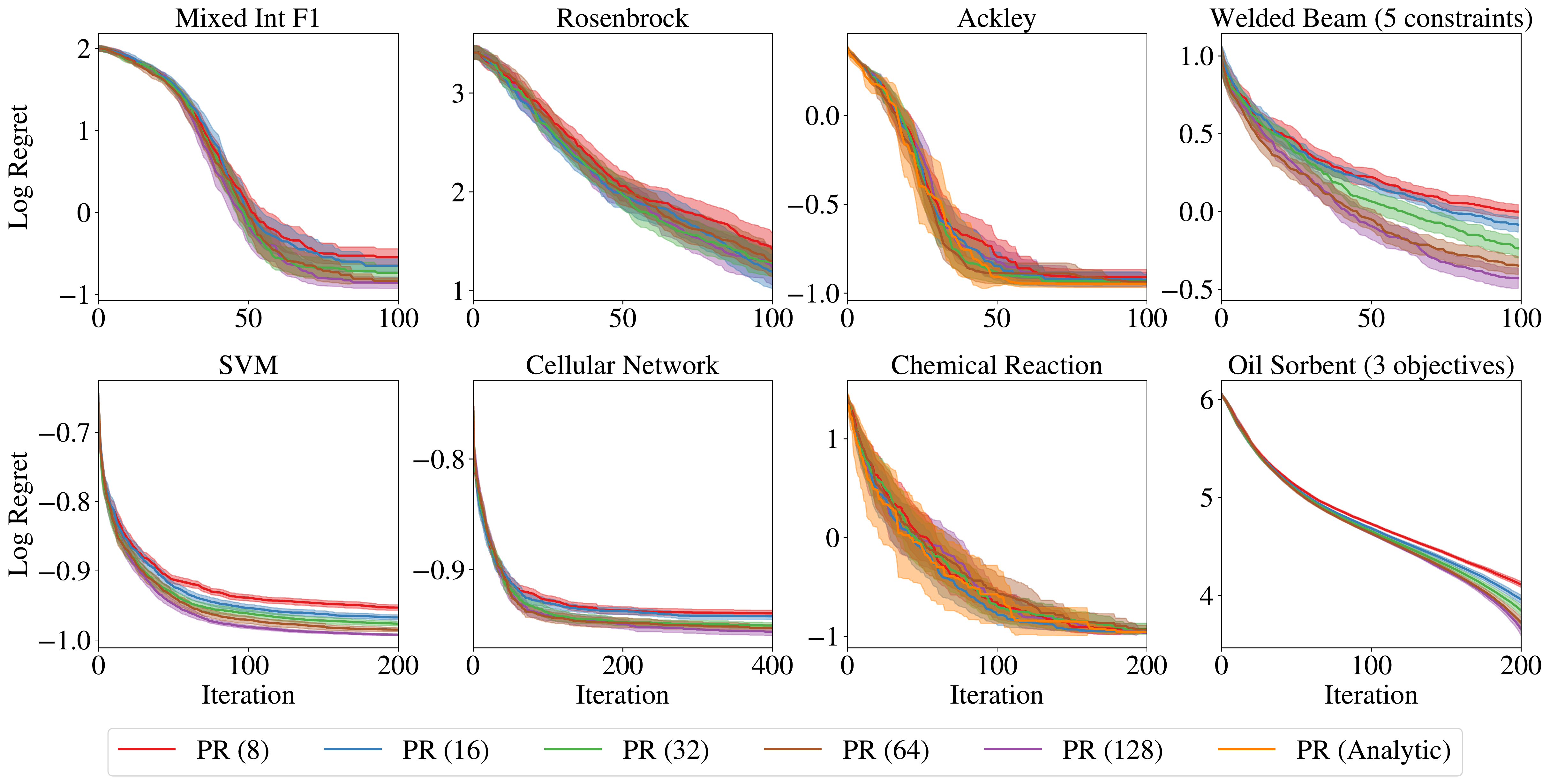}
    \caption{A sensitivity analysis of the optimization performance of \PR{} with respect to a small number of MC samples (with samples between 8 and 64). Performance degrades slightly when few samples are used.}
    \label{fig:exp1_mc_comparison_fewer}
\end{figure*}
\begin{figure*}[ht]
    \centering
    \includegraphics[width=\linewidth]{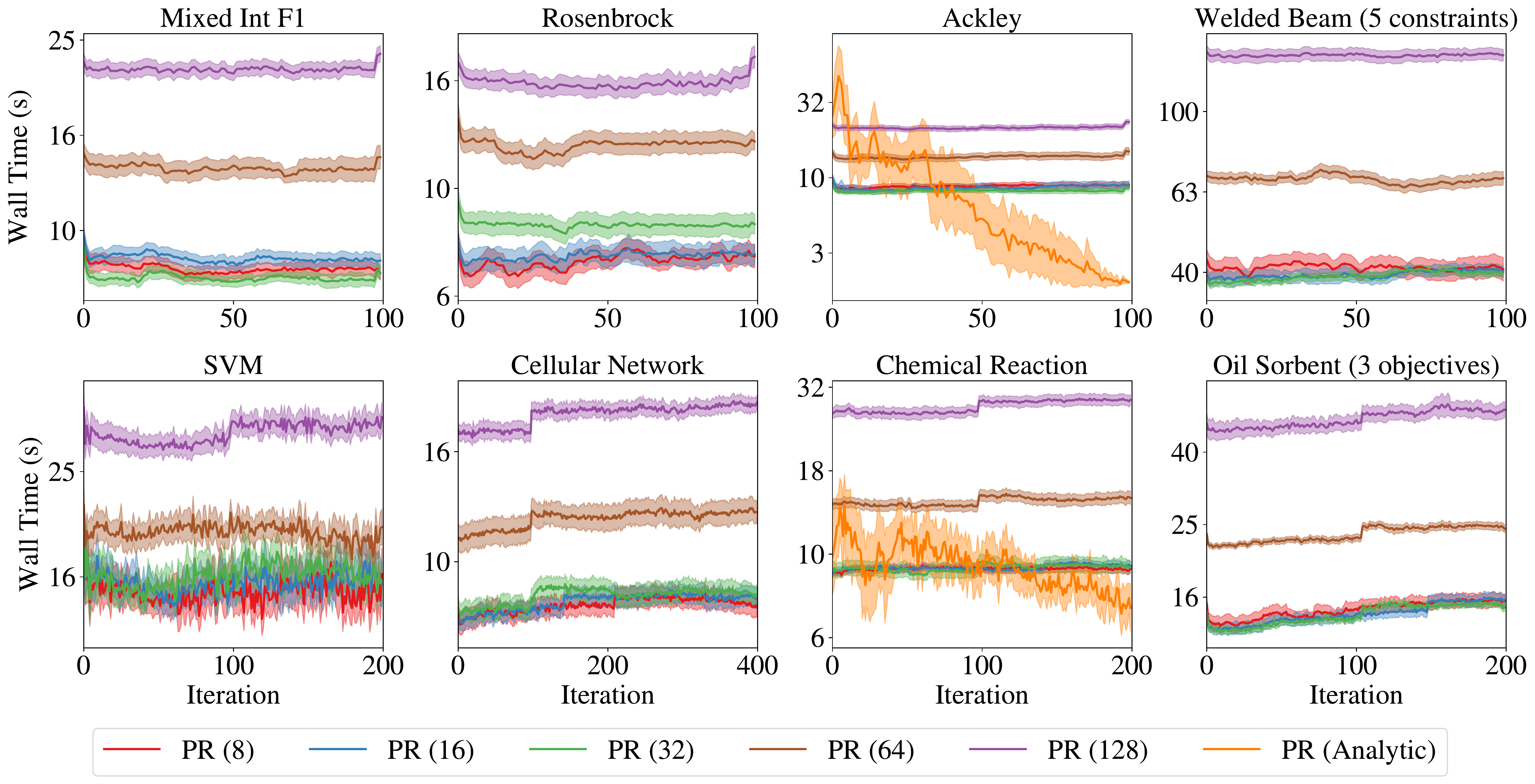}
    \caption{A sensitivity analysis of the wall time of \PR{} with respect to the number of MC samples (with samples between 8 and 64). We observe that wall time scales linearly with the number of MC samples, which is expected since we compute \PR{} in $\frac{N}{32}$ chunks to avoid overflowing GPU memory.}
    \label{fig:wall_times_mc_comparison_fewer}
\end{figure*}
\FloatBarrier
\subsubsection{Evaluation of Approximation Error in MC Sampling}
We examine the MC approximation error relative to analytic \PR{} on the chemical reaction and ackley problems. The results in Figure~\ref{fig:mc_approx_error} show the mean absolute percentage error (MAPE) $$\frac{100 }{|X_\text{discrete}|} \cdot \sum_{\bm x \in X_\text{discrete},\bm \theta \in \Theta_\text{discrete}}\frac{\mathbb E_{\bm Z\sim p(\bm Z |\bm \theta)} \alpha(\bm x,\bm Z) - \frac{1}{N}\sum_{i=1}^N \alpha(\bm x,\tilde{\bm z}_i)}{ \max_{\bm x \in X_\text{discrete},\bm \theta \in \Theta_\text{discrete}} \mathbb E_{\bm Z\sim p(\bm Z |\bm \theta)} \alpha(\bm x,\bm Z)}$$ evaluated over a random set of $|X_\text{discrete}| =|\Theta_\text{discrete}| 10,000$ points from $\mathcal X \times \mathcal \Theta$ (the sampled sets are denoted $X_\text{discrete}, \Theta_\text{discrete}$). We observe a rapid reduction in MAPE as we increase the number of samples. 
With 1024 samples, MAPE is 0.055\% (+/- 0.0002 \%) over 20 replications (different MC samples in \PR{}) on the chemical reaction problem and MAPE is 0.018\% (+/- 0.0003 \%) on the ackley problem.

With 128 samples, MAPE is 0.282\% (+/- 0.0029 \%) over 20 replications (different MC samples in \PR{}) on the chemical reaction problem and MAPE is 0.052\% (+/- 0.0021 \%) on the ackley problem.

\begin{figure*}[ht]
    \centering
    \includegraphics[width=0.8\linewidth]{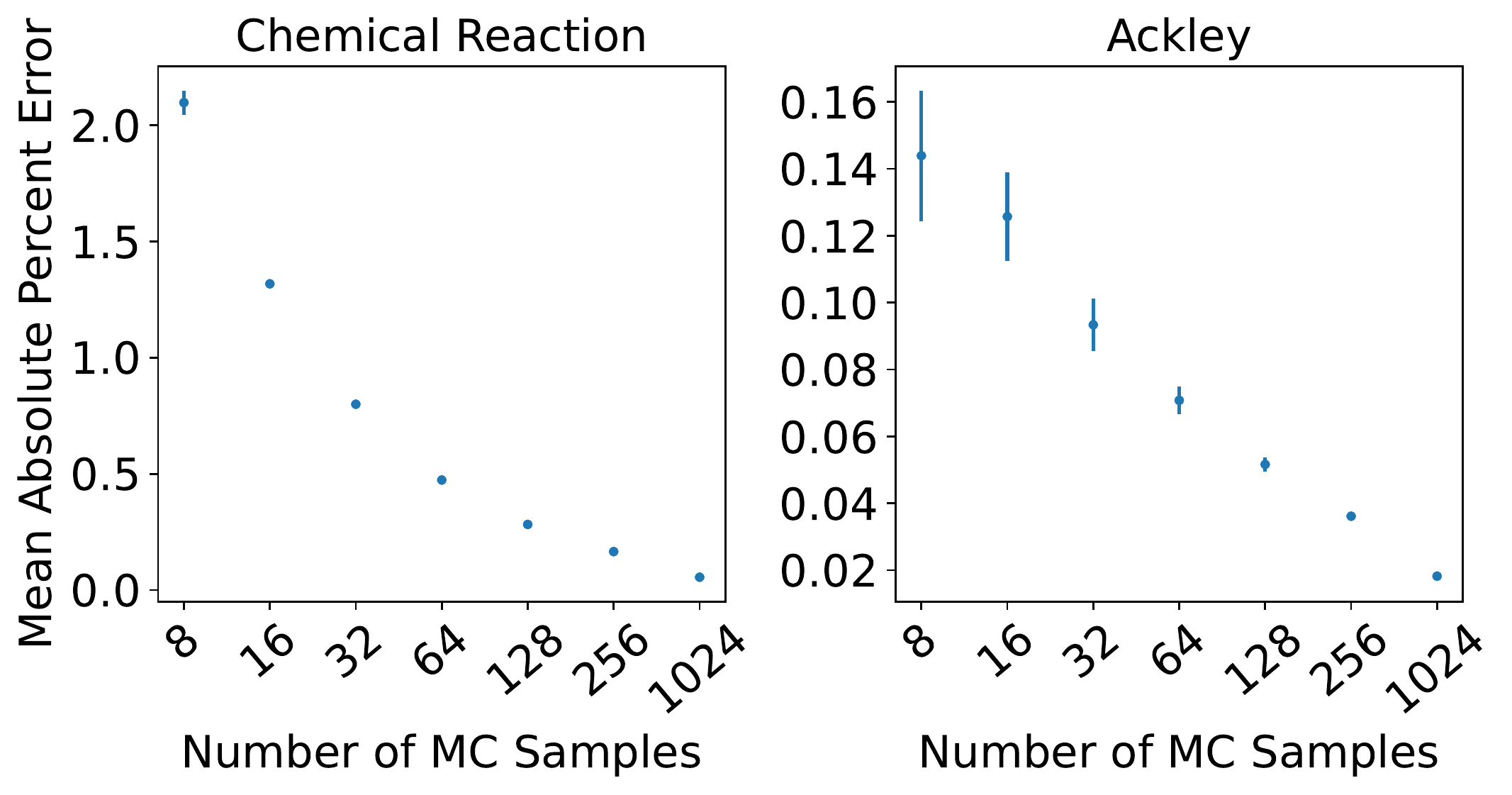}
    \caption{An evaluation of the mean absolute percentage error for the MC estimator of \PR{} (relative to analytic \PR{}).}
    \label{fig:mc_approx_error}
\end{figure*}
 

\section{Effect of $\tau$ in Transformation}
\label{appdx:tau}
Throughout the main text, we use $\tau = 0.1$, which we selected based on the observation that it provides a reasonable balance between retaining non-zero gradients of $g(\phi)$ with respect to $\phi$ and allowing $\theta$ to become close to 0 or 1 as shown in Figure~\ref{fig:transforms}.
\FloatBarrier
\begin{figure*}[ht]
    \centering
    \includegraphics[width=0.3\linewidth]{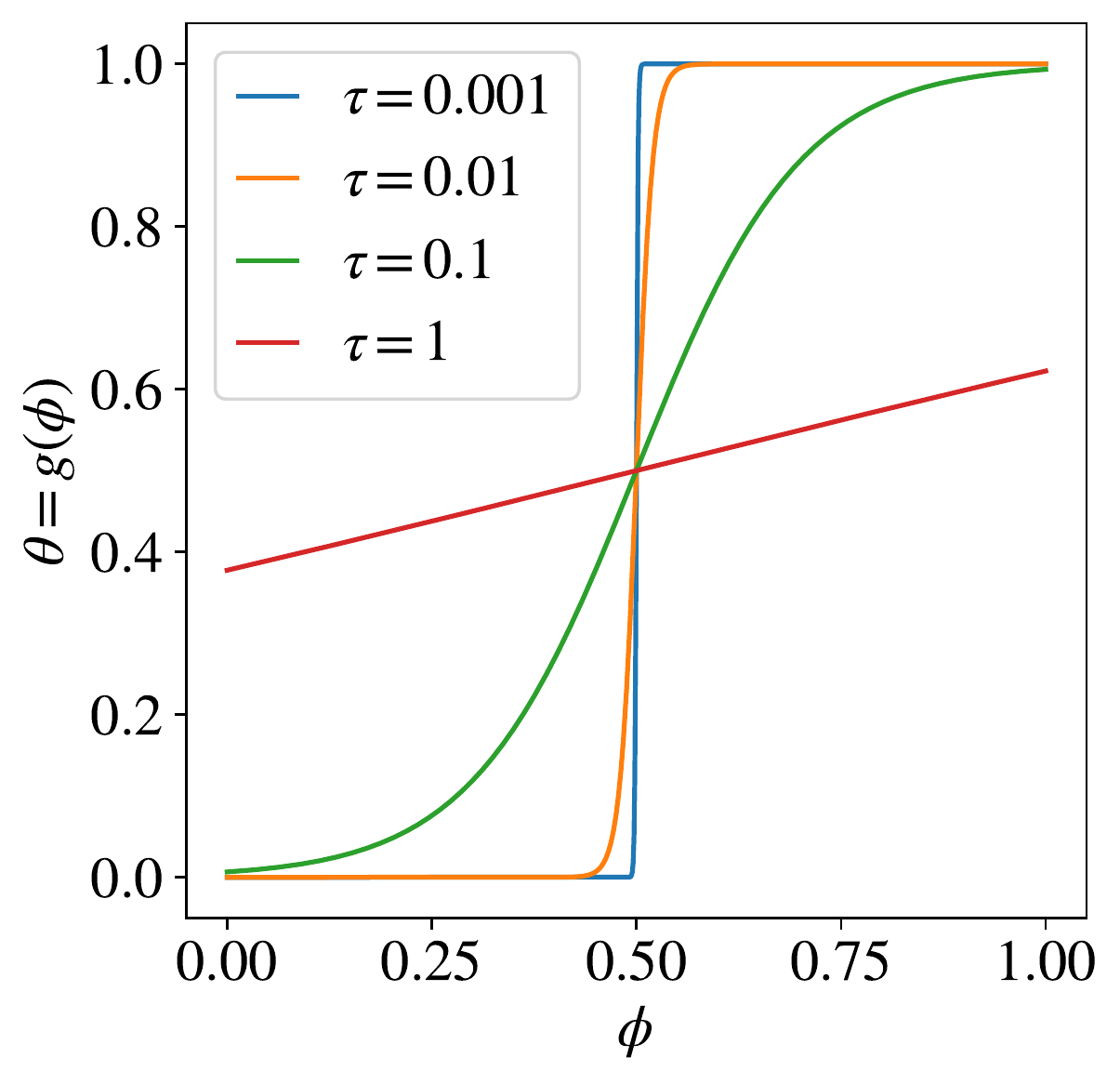}
    \caption{A comparison of the reparameterization of $\theta$ under various choices of $\tau$. We observe that $\tau = 0.1$ provides a reasonable balance between retaining non-zero gradients of $g(\phi)$ with respect to $\phi$ and allowing $\theta$ to become close to 0 or 1.}
    \label{fig:transforms}
\end{figure*}

As $\tau \rightarrow 0$, the $\theta$ can take more extreme values, but the gradient of the transformation with respect to $\phi$ also moves closer to zero. For larger values of $\tau$, the gradient of the transformation with respect to $\phi$ is larger, but $\theta$ has a more limited domain with less extreme values. We find that $\tau=0.1$ is a robust setting across all experiments.


\section{Alternative methods}
\label{appdx:alternative_methods}
\subsection{Straight-through gradient estimators}
An alternative approach to using approximating the gradients under exact rounding using finite differences is to approximate the gradients using straight-through gradient estimation (STE) \citep{bengio2013estimating}. The idea of STE is to approximate the gradient of a function with the identity function. In our setting, the gradient of the discretization function with respect to its input is estimated using an identity function. Using this estimator enables gradient-based AF optimization, even though the true gradient of the discretization function is zero everywhere that it is defined. Although STEs have been shown to work well empirically, these estimators are not well-grounded theoretically. Their robustness and potential pitfalls  in the context of AF optimization have not been well studied. Below, we evaluate the aforementioned \exactroundste{} approach and show that it offers competitive optimization performance (Figure~\ref{fig:exp1_ste}) with fast wall times (Figure~\ref{fig:wall_times_ste}), but does not quite match the optimization performance of \PR{} on several benchmark problems.
\FloatBarrier
\begin{figure*}[ht]
    \centering
    \includegraphics[width=\linewidth]{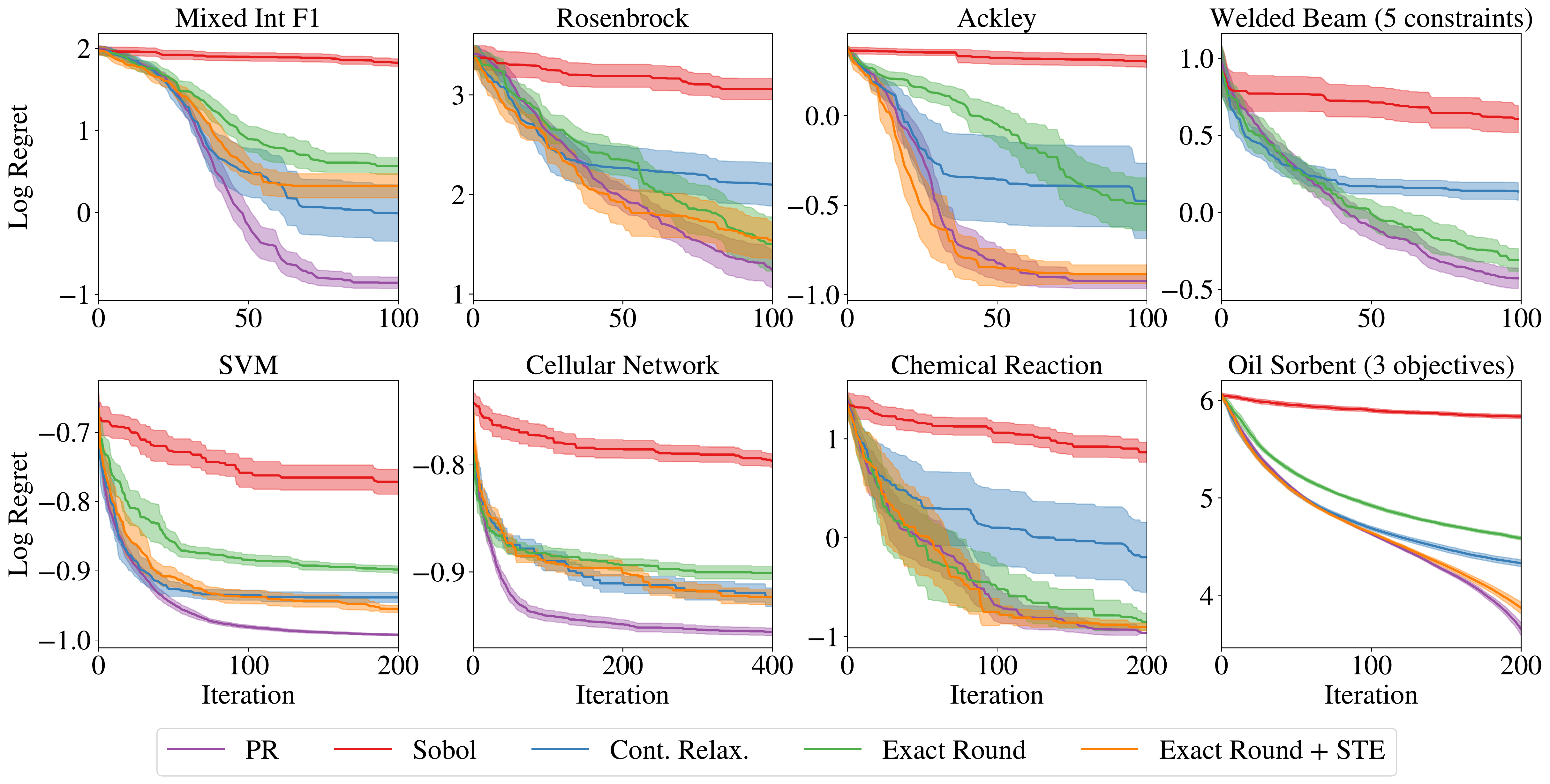}
    \caption{A comparison of exact rounding with straight-through gradient estimators versus  other acquisition optimization strategies.
        Log regret on each problem. 
        We report log hypervolume regret for Oil Sorbent and report the log regret of the best feasible objective for Welded beam.
    }
    \label{fig:exp1_ste}
\end{figure*}
\begin{figure*}[ht]
    \centering
    \includegraphics[width=\linewidth]{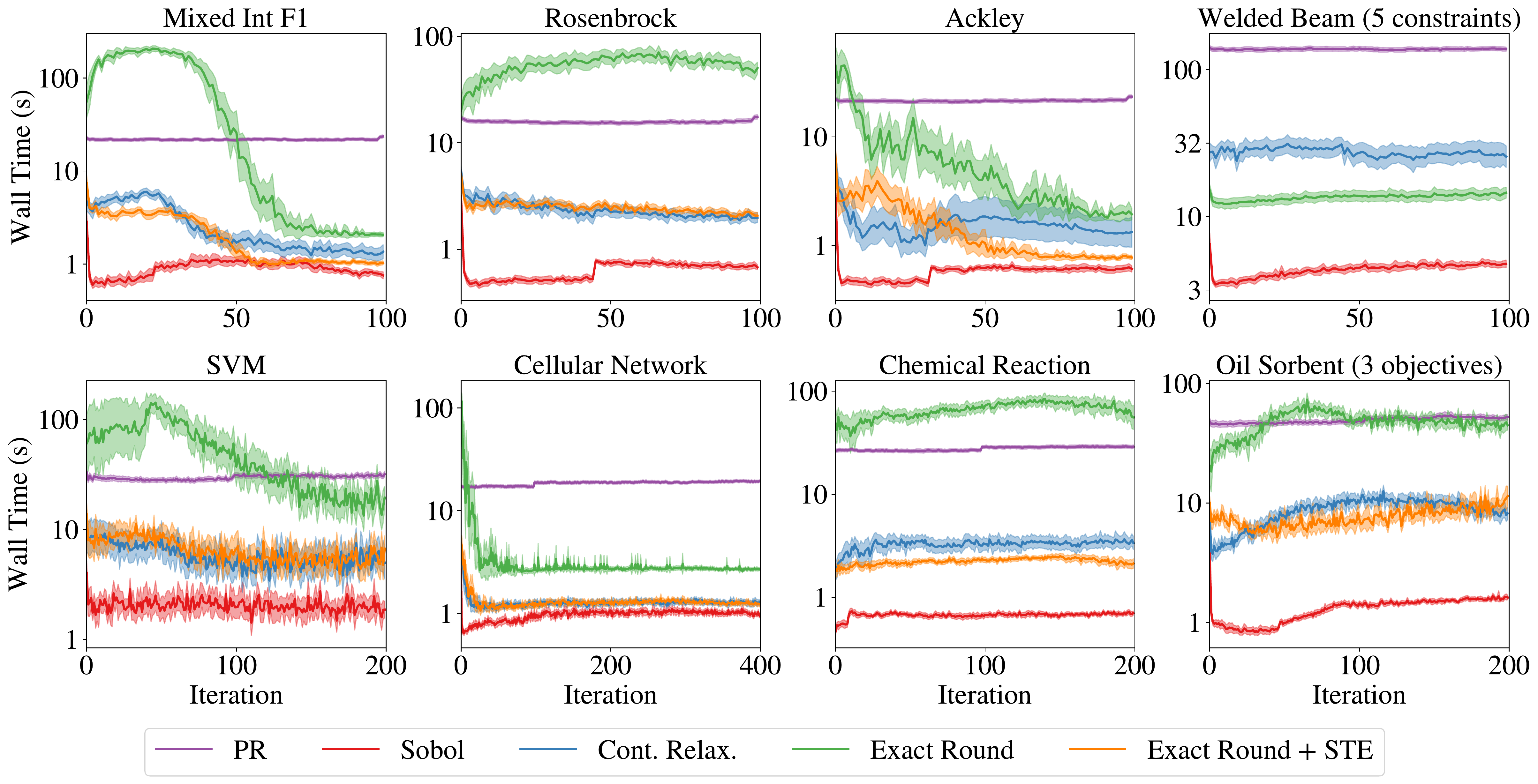}
    \caption{A comparison of wall times of exact rounding with straight-through gradient estimators versus  other acquisition optimization strategies.
    }
    \label{fig:wall_times_ste}
\end{figure*}
\FloatBarrier
\subsection{\turbo{} methods with alternative optimizers}
In this section, we consider alternative methods to \PR{} for optimizing AFs using within trust regions. The results in Figure~\ref{fig:exp1_turbo} show that \PR{} is a consistent best optimizer using TRs, but that STEs work quite well with TRs in many scenarios.
\begin{figure*}[ht]
    \centering
    \includegraphics[width=\linewidth]{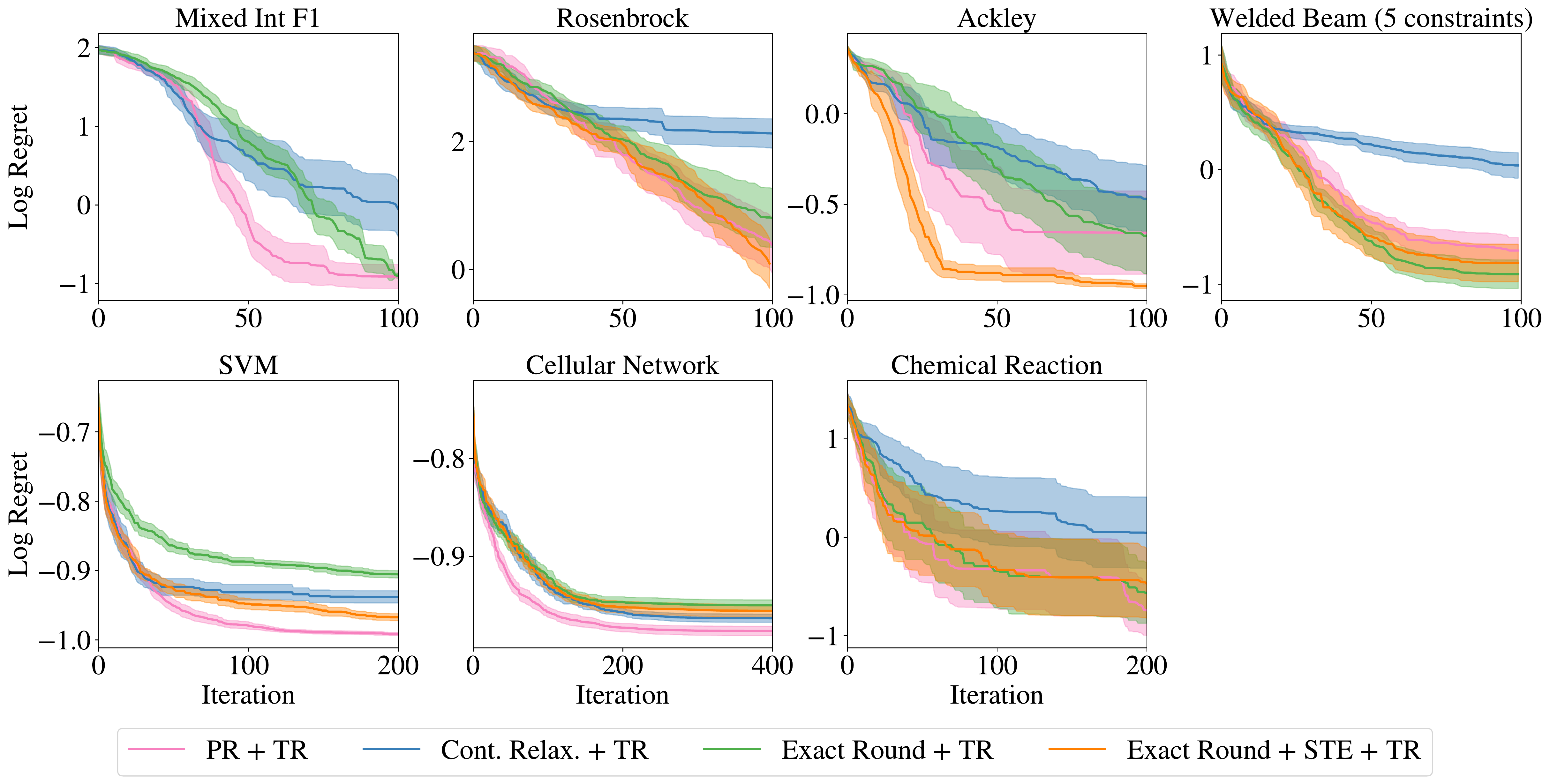}
    \caption{A comparison of \turbo{} methods with different acquisition optimization strategies. Log regret on each problem. 
        We report log hypervolume regret for Oil Sorbent and report the log regret of the best feasible objective for Welded beam. }
    \label{fig:exp1_turbo}
\end{figure*}
\begin{figure*}[ht]
    \centering
    \includegraphics[width=\linewidth]{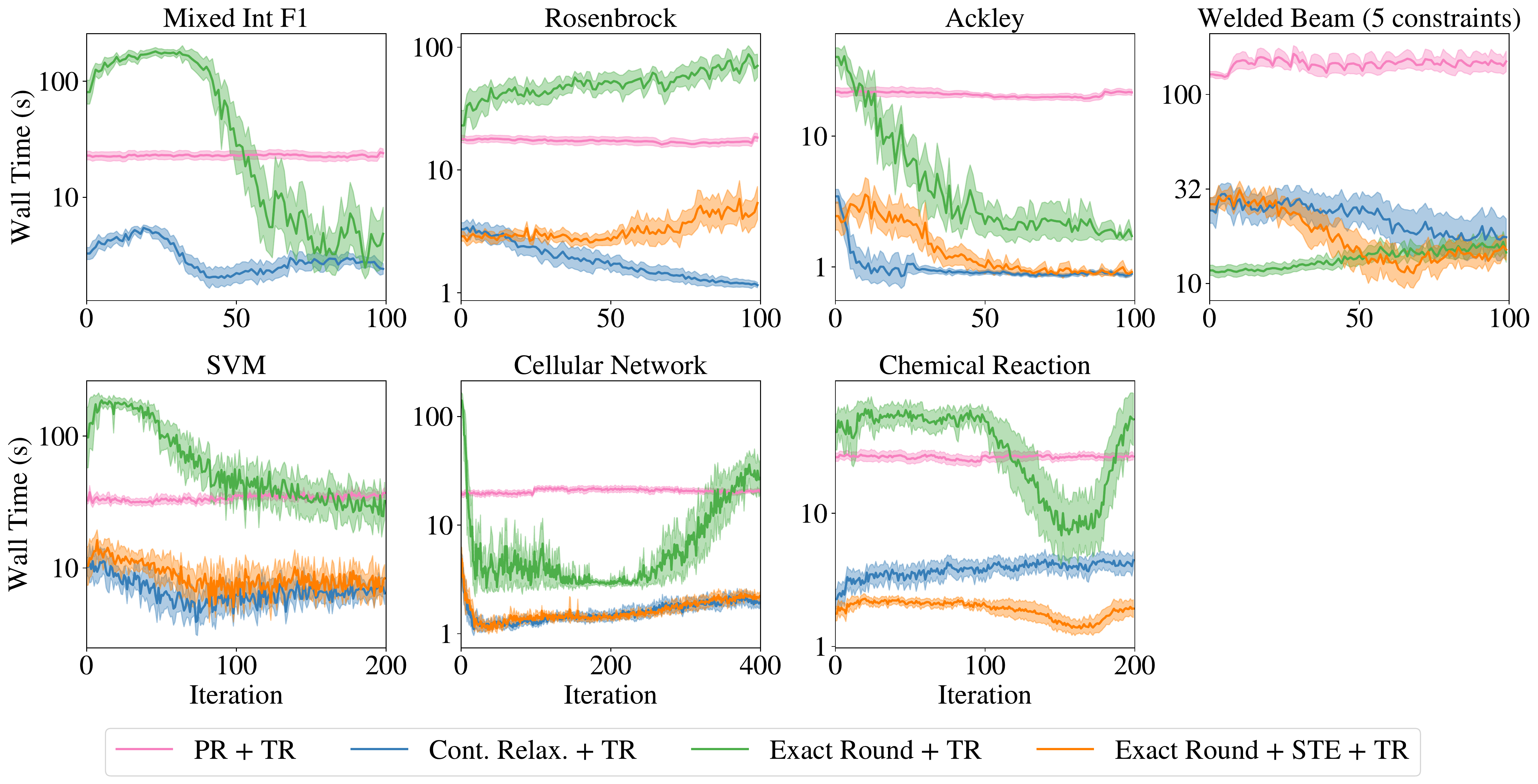}
    \caption{A comparison of wall times of \turbo{} methods with different acquisition optimization strategies.}
    \label{fig:wall_times_turbo}
\end{figure*}
\FloatBarrier
\section{Acquisition Function Optimization at a Given Wall Time Budget}
In Figure~\ref{fig:equal_wall_time_budget}, we provide additional starting points (64 points, rather than 20) to other non-\PR{} methods in order to provide them with additional wall-time budget. We find that using \PR{} with 64 MC samples, \PR{} provides rapid convergence compared to other baselines and therefore is a good optimization routine for any wall time budget.

\FloatBarrier
\begin{figure*}[ht]
    \centering
    \includegraphics[width=0.6\linewidth]{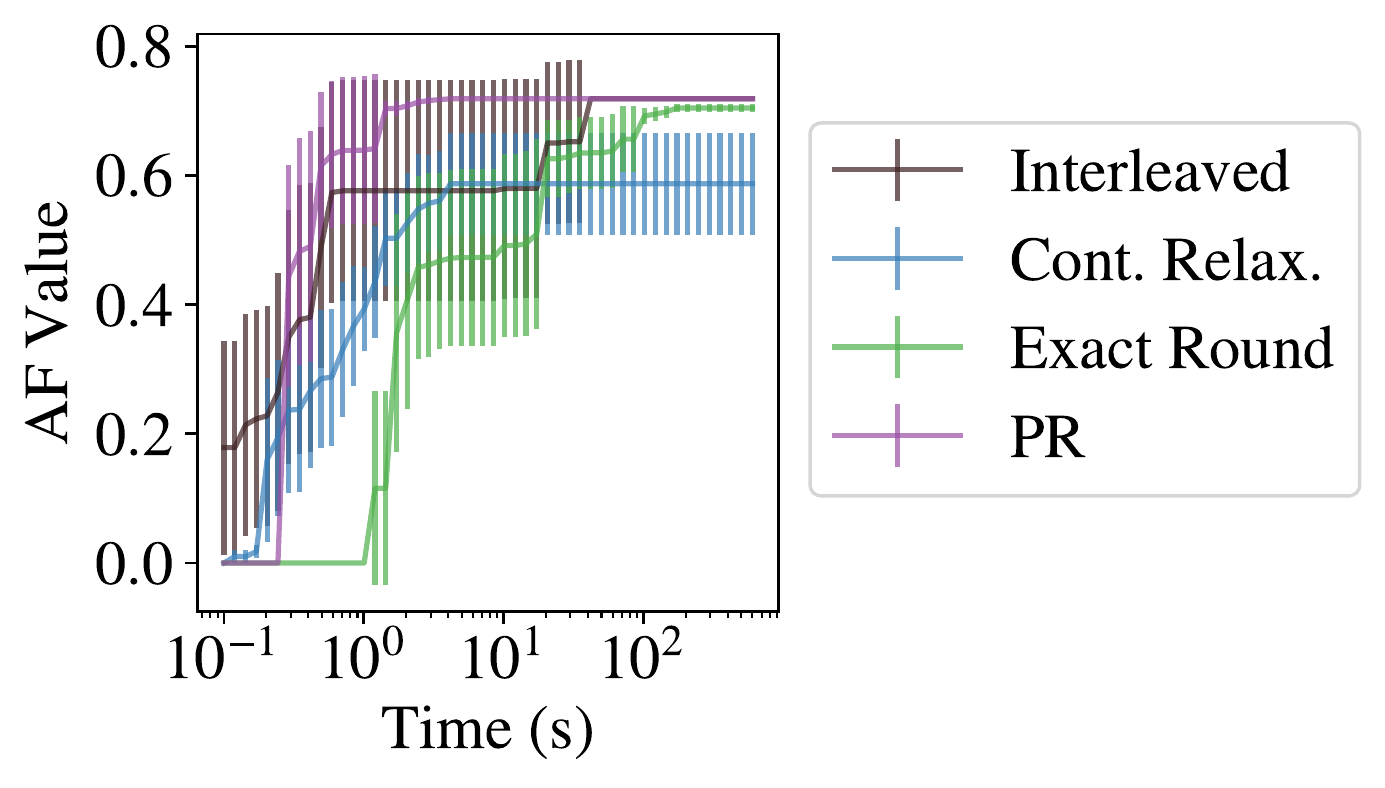}
    \caption{A comparison of methods for optimizing acquisition functions at a given wall time budget.}
    \label{fig:equal_wall_time_budget}
\end{figure*}
\FloatBarrier

\section{Alternative categorical kernels }
\label{appdx:alternative_kernels}
In this section, we demonstrate that \PR{} can be used with arbitrary kernels over the categorical parameters including those that require discrete inputs (which \contbo{} is incompatible with). Specifically for the categorical parameters, we compare using (a) a Categorical kernel (default) versus with a Mat\'ern-5/2 kernel with either (b) one-hot encoded categoricals, (c) a latent embedding kernel \citep{zhang2019latent}, or known embeddings based on density functional theory (DFT) \citep{shields2021bayesian}. For the latent embedding kernel, we follow \citet{pelamatti2021bayesian} and use a 1-d latent embedding for categorical parameters where the cardinality is less than or equal to 3 and a 2-d embedding for categorical parameters where the cardinality is greater than 3. For each latent embedding, we use an isotropic Mat\'ern-5/2 kernel and use product kernel across the kernels for the categorical, binary, ordinal, and continuous parameters. For the kernel over DFT embeddings, we use the DFT embeddings for the direct arylation dataset from \citet{shields2021bayesian}, which are available at \url{https://github.com/b-shields/edbo}. It is worth noting that in the Chemical Reaction problem, the black-box objective is a GP surrogate model with a Categorical kernel that is fit to the direct arylation dataset. The purpose of this section is demonstrate that \PR{} is agnostic to the choice of kernel over discrete parameters. Because the Chemical Reaction problem is based on a GP surrogate, we do not draw conclusions about which choice of kernel is best suited for modeling the true, unknown underlying Chemical Reaction yield function.
\FloatBarrier
\begin{figure*}[ht]
    \centering
    \includegraphics[width=\linewidth]{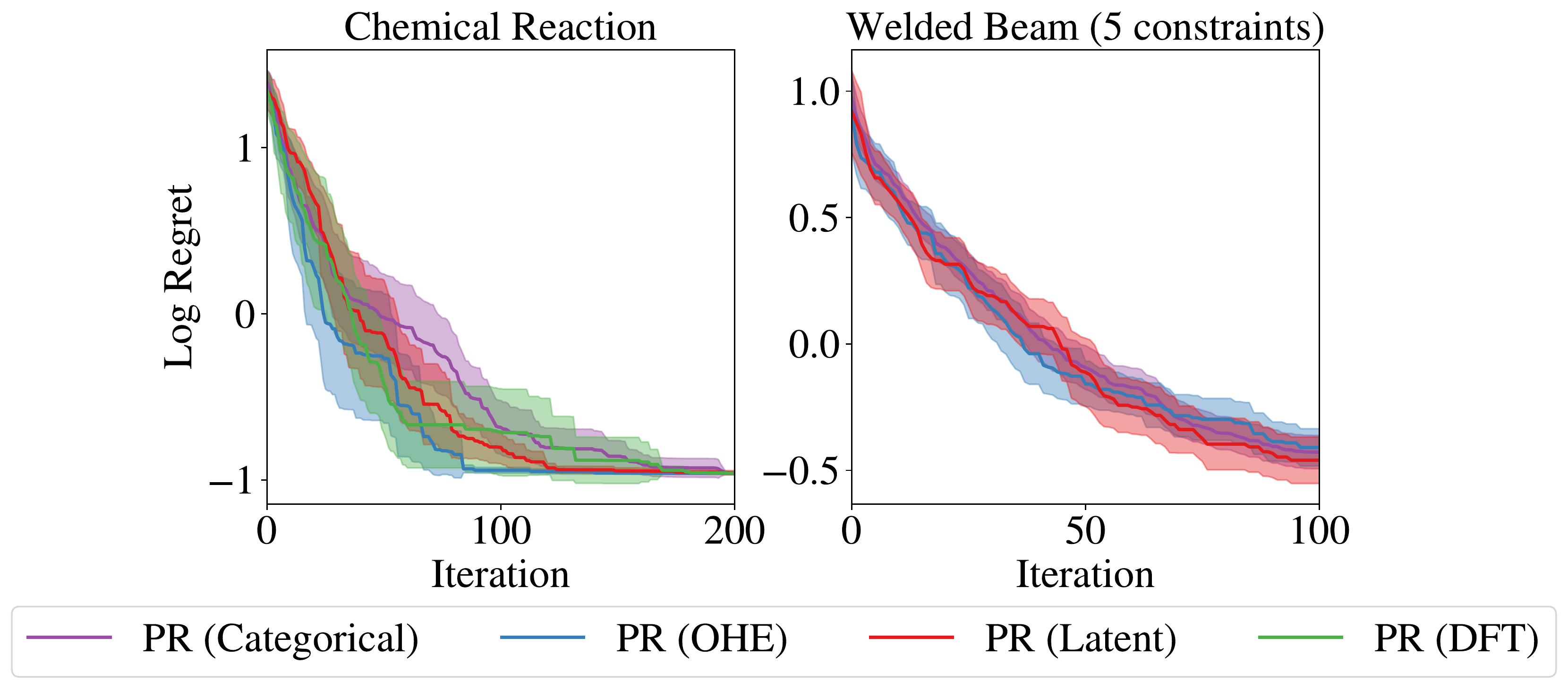}
    \caption{A comparison of different kernels over categorical parameters. Left: Welded beam has one categorical parameter, metal type (4 levels). Right: Chemical reaction has three categorical parameters, solvent, base, and ligand (with 4, 12, and 4 levels, respectively).}
    \label{fig:kernels}
\end{figure*}
\FloatBarrier
\section{Alternative Acquisition Functions }
\label{appdx:alternative_acqfs}
In this section, we compare \PR{} with expected improvement (EI) against \PR{} with upper confidence bound (UCB). For UCB, we set the hyperparameter $\beta$ in each iteration using the method in \citet{pmlr-v37-kandasamy15}. Although UCB comes enjoys bounded regret \citep{ucb}, we find empirically that EI works better on most problems. 
\FloatBarrier
\begin{figure*}[ht]
    \centering
    \includegraphics[width=\linewidth]{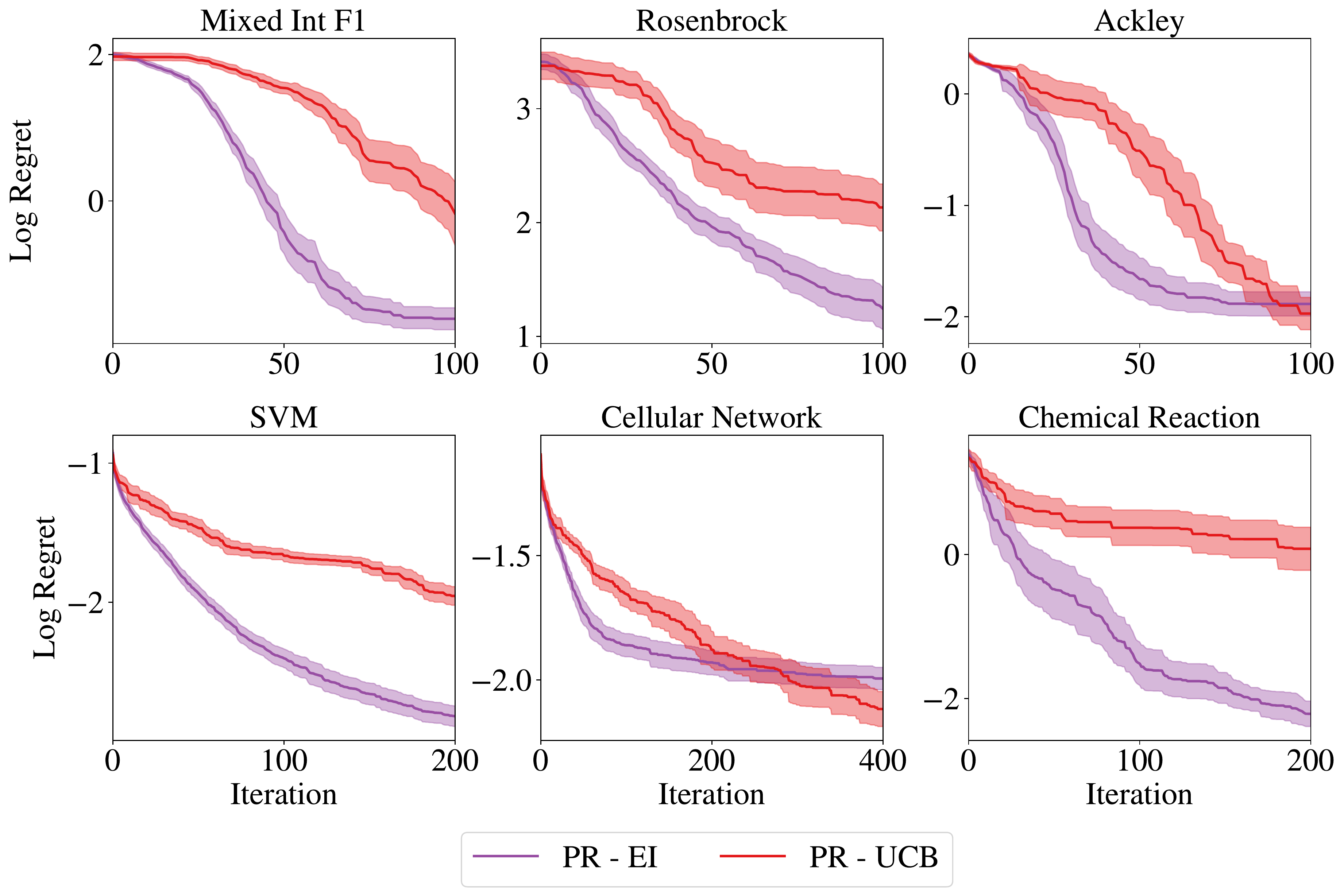}
    \caption{A comparison of expected improvement (EI) and upper confidence bound (UCB) acquisition functions with \PR{}.}
    \label{fig:ei_ucb}
\end{figure*}
\FloatBarrier
\section{Additional Results on Optimizing Acquisition Functions}
In this section, we provide additional results on various approaches for optimizing acquisition functions using the same evaluation procedure as in the main text. We use 50 replications.
\FloatBarrier
\begin{figure*}[ht]
    \centering
    \includegraphics[width=\linewidth]{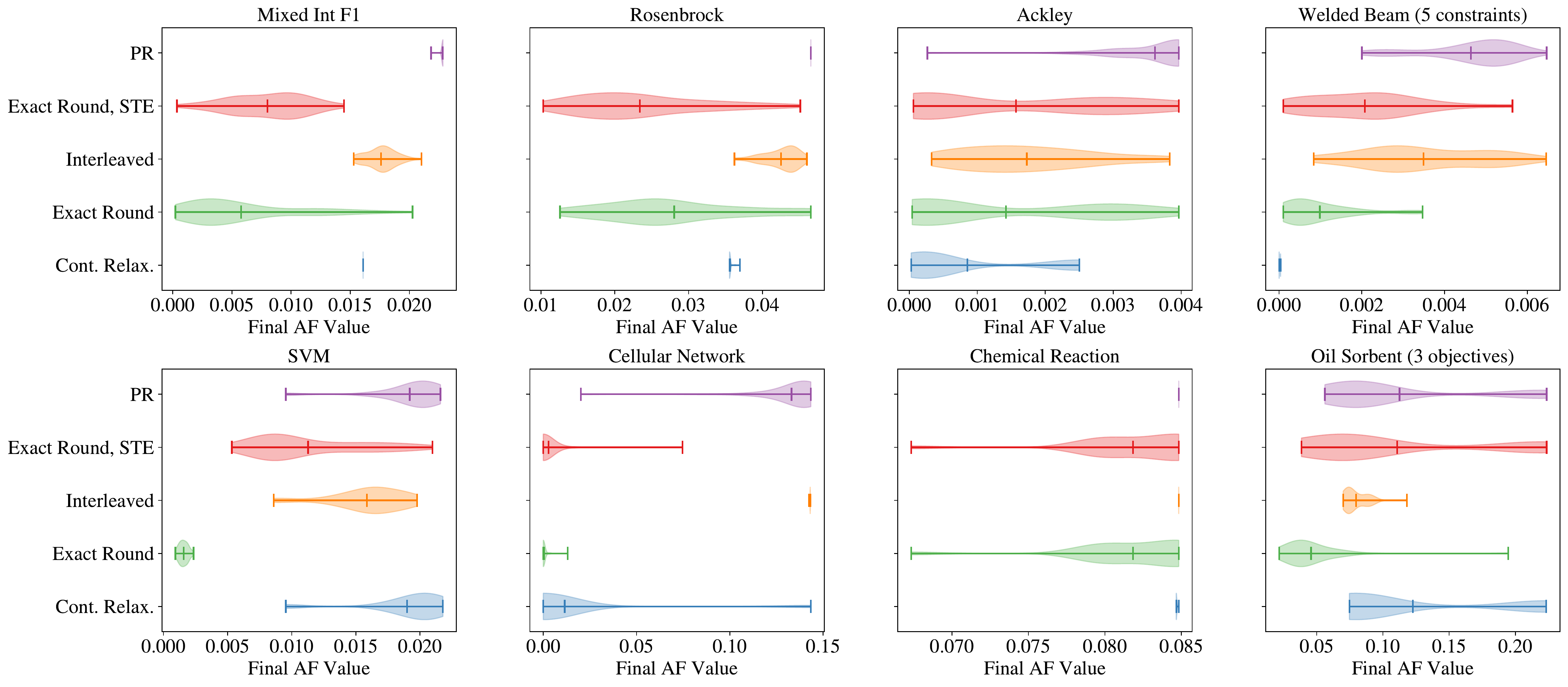}
    \caption{A comparison of methods for optimizing acquisition functions.}
    \label{fig:final_acqf}
\end{figure*}
\FloatBarrier
\section{Stochastic vs Deterministic Optimization}
\label{appdx:sgd_vs_saa}

We compare optimizing \PR{} with stochastic and deterministic optimization methods. For stochastic optimizers, we compare stochastic gradient ascent (SGA) and Adam with various initial learning rates. For SGA, the learning rate is decayed each time step $t$ by multiplying the initial learning rate by $t^{-0.7}$ and for Adam a fixed learning rate is used. For stochastic optimizers, the MC estimators of \PR{} and its gradient stochastic mini-batches of $N=128$ MC samples are used. For deterministic optimization, base samples are kept fixed. All routines are run for a maximum of 200 iterations. In Figure~\ref{fig:sgd}, we observe that Adam is more robust to the choice of learning rate than SGA and generally is the best performing method. Furthermore, Adam consistently performs better than deterministic optimization. We compare Adam with a learning rate of $\frac{1}{40}$ against L-BFGS-B in Figure~\ref{fig:saa_vs_adam}.
\FloatBarrier
\begin{figure*}[ht]
    \centering
    \includegraphics[width=\linewidth]{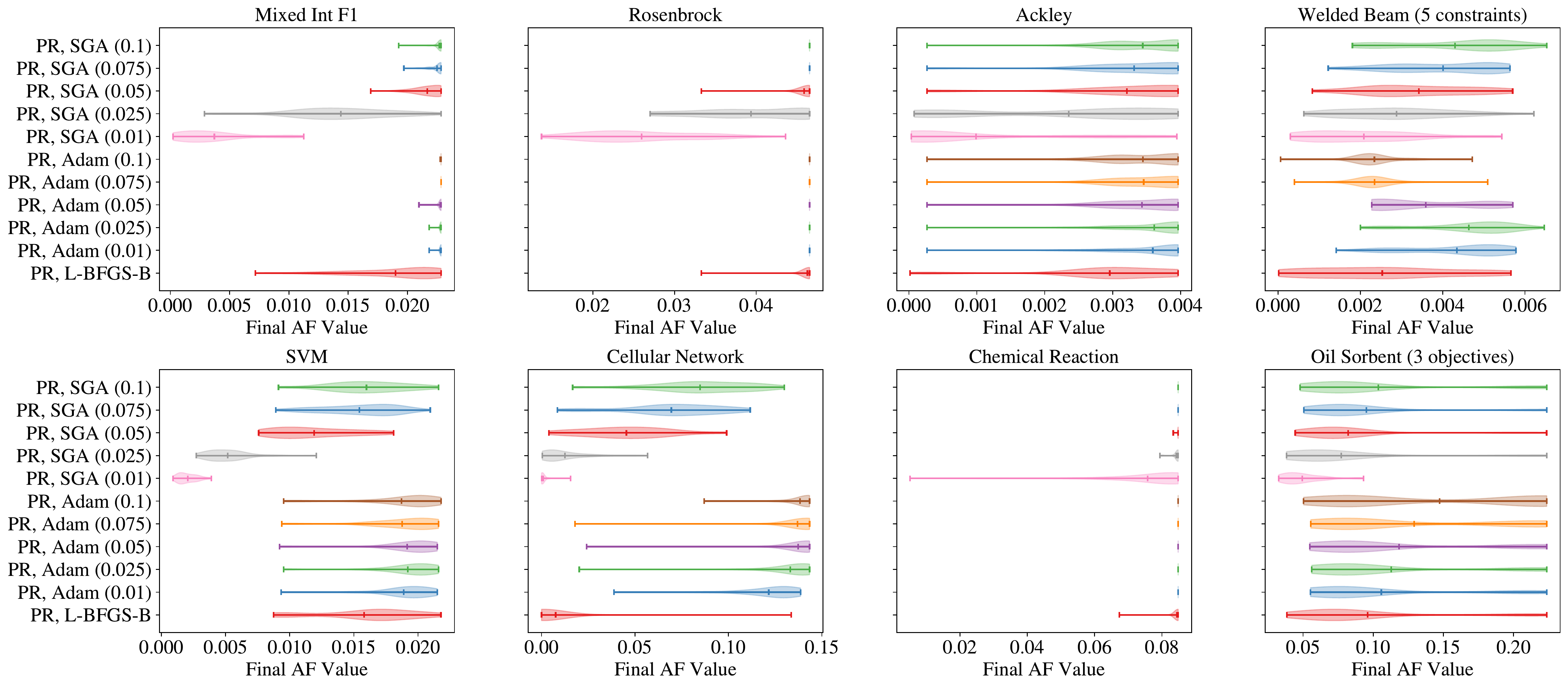}
    \caption{A comparison of \PR{} using stochastic and deterministic optimization methods. The initial learning rate for stochastic gradient ascent is given in parentheses.}
    \label{fig:sgd}
\end{figure*}
\begin{figure*}[!ht]
    \centering
    \includegraphics[width=\linewidth]{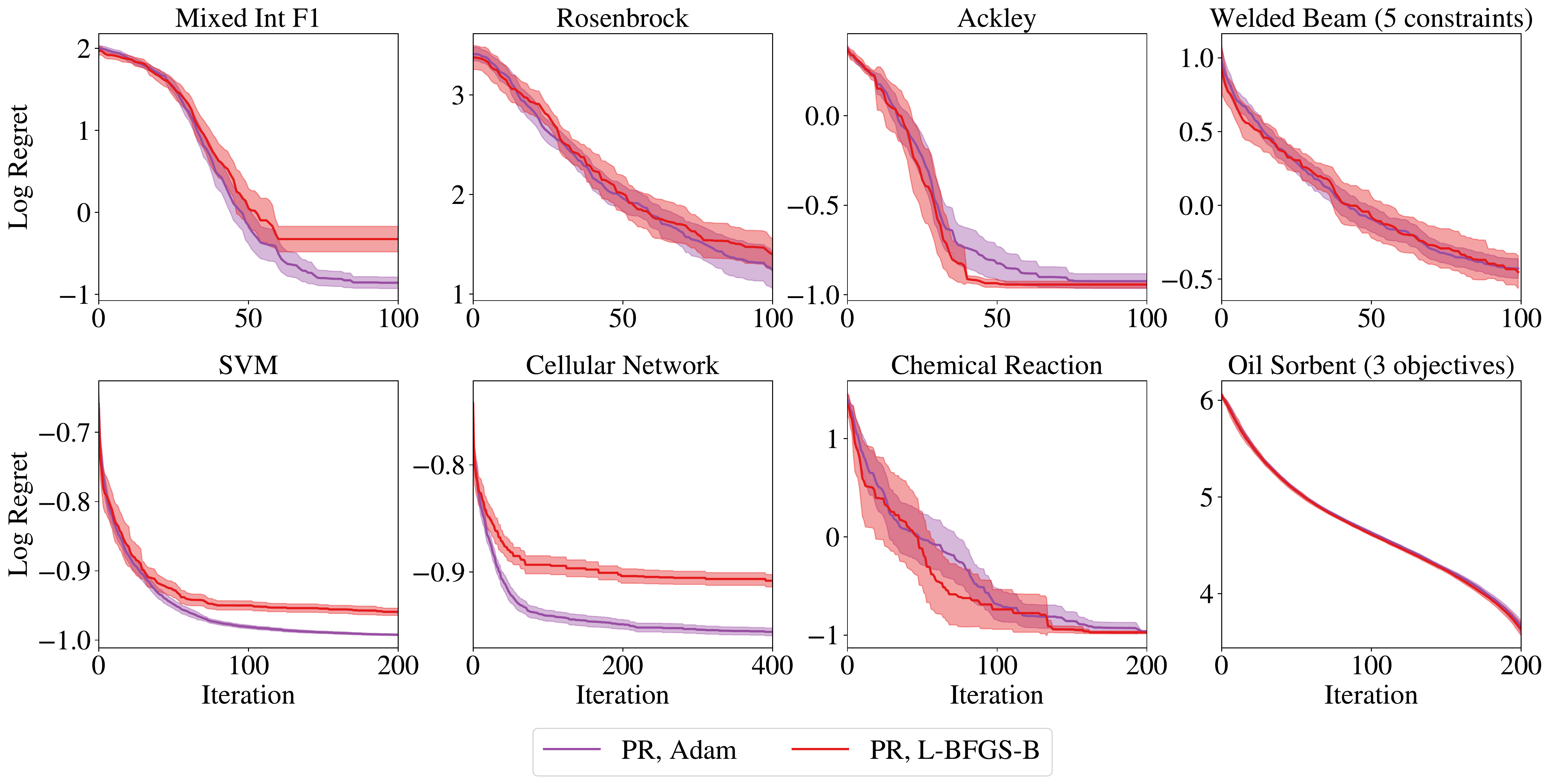}
    \caption{
        A comparison of optimizing the PO using deterministic estimation (via SAA) and  optimization versus stochastic estimation and optimization.}
    \label{fig:saa_vs_adam}
\end{figure*}
\FloatBarrier

\section{Comparison with an Evolutionary Algorithm}
In Figures~\ref{fig:ea} and ~\ref{fig:ea}, we compare against the evolutionary algorithm PortfolioDiscreteOnePlusOne, which is the recommended algorithm for discrete and mixed search spaces in the Nevergrad package \citep{nevergrad}. We find that \PR{} significantly outperforms this baseline by a large margin with respect to log regret, but is slower than the evolutionary algorithm with respect to wall time.

\FloatBarrier
\begin{figure*}[ht]
    \centering
    \includegraphics[width=0.75\linewidth]{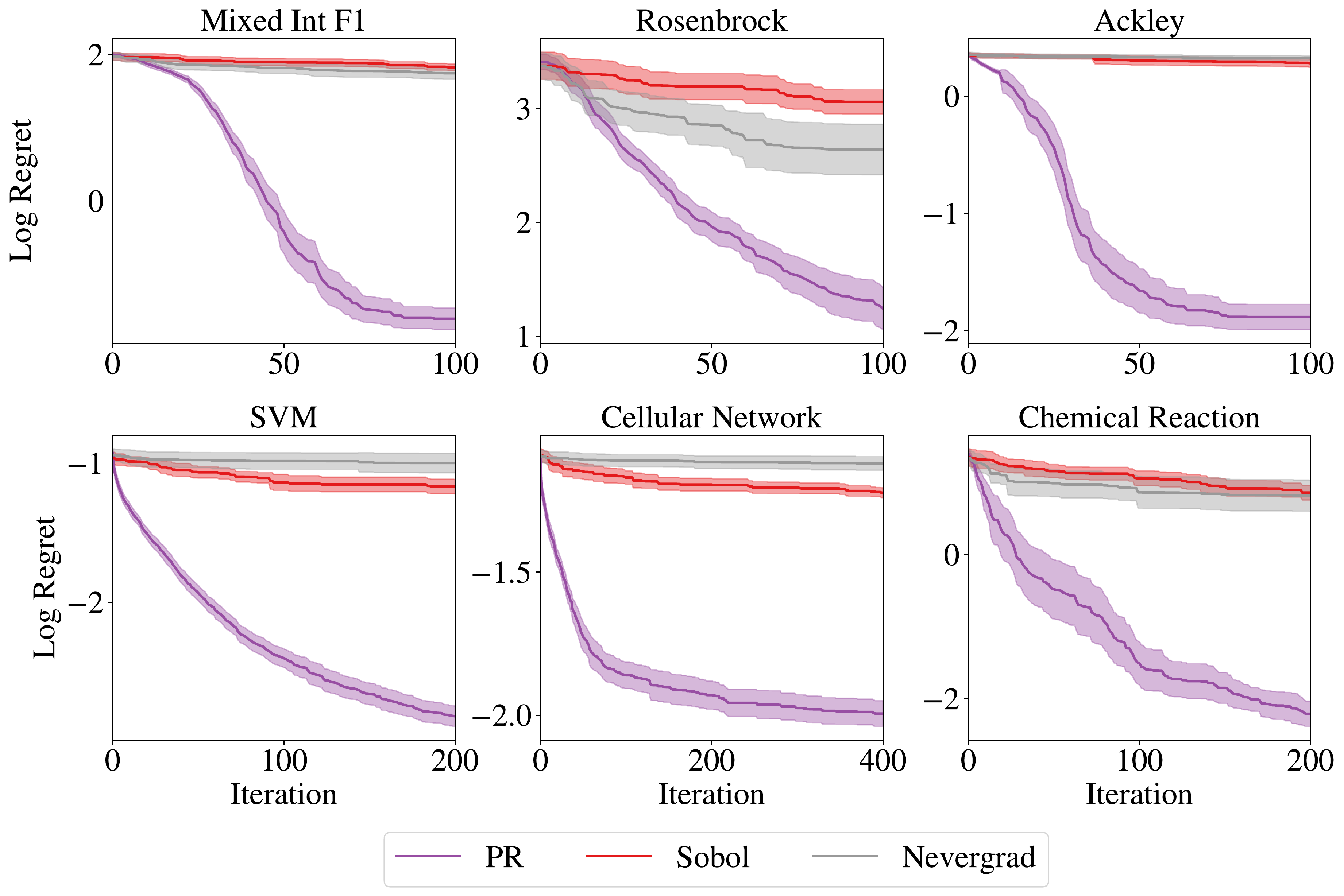}
    \caption{A comparison with an evolutionary algorithm with respect to log regret.}
    \label{fig:ea}
\end{figure*}
\begin{figure*}[ht]
    \centering
    \includegraphics[width=0.75\linewidth]{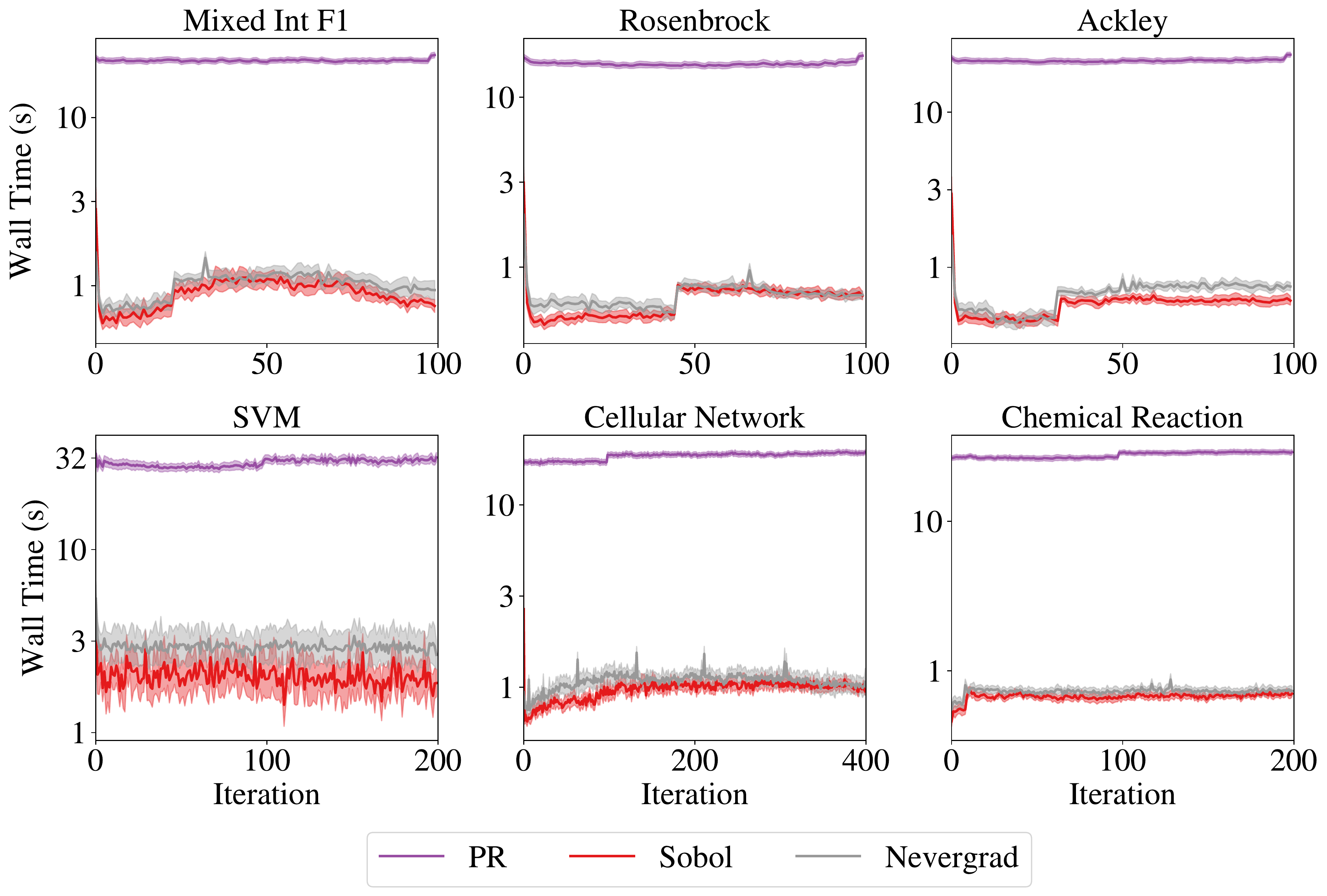}
    \caption{A comparison with an evolutionary algorithm with respect to wall time.}
    \label{fig:ea_times}
\end{figure*}
\FloatBarrier


\end{document}